%% file: iclr2027_conference.tex
\documentclass{article} %
\usepackage{iclr2027_conference,times}

\input{math_commands.tex}
\usepackage{color}
\usepackage{xcolor}
\usepackage{hyperref}
\usepackage{fontawesome5}
\usepackage{url}
\usepackage{amssymb}
\usepackage{amsthm}
\usepackage{graphicx}
\graphicspath{{./}}
\usepackage{booktabs}
\usepackage{array}
\usepackage{multirow}
\usepackage[normalem]{ulem}
\usepackage{xcolor}
\usepackage{amsmath}
\usepackage{tikz}
\usepackage{enumitem}
\usetikzlibrary{positioning, calc, arrows.meta}
\newcommand{\linkbutton}[3]{{\hypersetup{pdfborder={0 0 0}}\href{#1}{\tikz[baseline=(b.base)]\node[draw=black!70,fill=black!8,rounded corners=3pt,inner xsep=7pt,inner ysep=3pt,text=black] (b) {#2\,\textsf{#3}};}}}

\newcommand{\relabel}[3]{%
  \pgfmathsetmacro{\relabelu}{\linewidth/#2}%
  \begin{tikzpicture}
    \node[inner sep=0, anchor=north west] (img) at (0,0) {\includegraphics[width=\linewidth]{#1}};
    \begin{scope}[shift={(img.north west)}, x=\relabelu pt, y=-\relabelu pt]
      #3
    \end{scope}
  \end{tikzpicture}}

\newtheorem{theorem}{Theorem}

\newtheorem{proposition}[theorem]{Proposition}

\theoremstyle{definition}

\newtheorem{remark}[theorem]{Remark}

\renewcommand{\R}{\mathbb{R}}
\newcommand{\Normal}{\mathcal{N}}
\newcommand{\cM}{\mathcal{M}}

\newcommand{\bq}{\mathbf{q}}
\newcommand{\bp}{\mathbf{p}}
\newcommand{\bz}{\mathbf{z}}
\newcommand{\bu}{\mathbf{u}}

\newcommand{\bI}{\mathbf{I}}
\newcommand{\bzero}{\mathbf{0}}
\newcommand{\SIGReg}{\mathrm{SIGReg}}

\newcommand{\bx}{\mathbf{x}}

\providecommand{\bx}{\mathbf{x}}
\providecommand{\ba}{\mathbf{a}}

\providecommand{\algo}{LeHami}
\providecommand{\algos}{\algo}   %

\title{LeHami: Symplectic Latent Dynamics Toward Long-Horizon Rollouts in World Models}

\author{%
Romina Yalovetzky$^{1}$\thanks{Corresponding author: \texttt{romina.yalovetzky@nyu.edu}}
\And
Randall Balestriero$^{2,3}$
\And
Yann LeCun$^{1,3}$
\AND
\normalfont\normalsize $^{1}$Courant Institute, New York University \\
$^{2}$Brown University \\
$^{3}$AMI Labs \\[0.7em]
\normalfont\normalsize
\makebox[\textwidth][c]{%
\linkbutton{https://github.com/rominay/LeHami}{\faGithub}{Code}\hspace{0.8em}%
\linkbutton{https://rominay.github.io/lehami-website/}{\faGlobe}{Website}}
}

\iclrfinalcopy
\begin{document}

\maketitle
\fancyhead{}
\renewcommand{\headrulewidth}{0pt}

\begin{abstract}

Joint-embedding predictive architectures (JEPAs) learn world models directly in latent space, but unconstrained autoregressive predictors may contract the latent distribution over long rollouts. We introduce Symplectic Hamiltonian JEPA (\textbf{LeHami}), which replaces the standard predictor with an action-conditioned Hamiltonian flow over a structured latent state $(\mathbf{q}, \mathbf{p})$ integrated using an exactly symplectic leapfrog step. This construction preserves phase-space volume at every rollout step and rules out uniform contraction, asymptotically stable fixed points, and merging of distinct latent states, independently of the learned parameters or action sequence. We obtain the momentum-like component from temporal derivatives of the learned position-like representation, providing an explicit second-order state for prediction and control. We measure rollout collapse on the data, with four ratios of the rollout to the truth, on five environments learned from pixels: conservative and actuated pendulums, a mass-spring, a two-body system and a cart pole. With the encoder and latent state held fixed, an MLP and a transformer predictor contract the rollout on every environment, while LeHami's rollouts do not merge. LeHami is the most accurate of the three at long horizons, in position on four of the five environments and in velocity on three, while at short horizons the free predictors are more accurate on two of the five. The learned flow also integrates at unseen sub-frame step sizes and, run backward with the recorded actions, recovers earlier frames as well as it predicts later ones. LeWM, with a richer three-frame encoder of its own, is more accurate on four of the five, which suggests incorporating a richer encoder within LeHami in future work. These results show how geometric structure can provide inference-time guarantees that are complementary to representation regularization in latent world models.
Code for \algo~is available at \url{https://github.com/rominay/LeHami}.

\end{abstract}

\section{Introduction} \label{section:intro}
\input{intro}

\section{Related work} \label{sec:related}

\textbf{Hamiltonian learning.} Hamiltonian neural networks learn $H$ from positions and momenta, measured or encoded from pixels \citep{greydanus2019hamiltonian}. Symplectic recurrent networks integrate a learned $H$ with leapfrog over many steps \citep{chen2019symplectic}. Hamiltonian generative networks infer a latent phase space from pixels and decode every step back to a frame \citep{toth2019hamiltonian}. Lagrangian networks learn the Lagrangian instead \citep{cranmer2020lagrangian}, and SympNets form the general class of symplectic networks that includes the leapfrog step \citep{jin2020sympnets}. On the DeepMind Hamiltonian Dynamics Suite, physics priors do not reliably beat generic models, while a continuous, time-reversible step helps \citep{botev2021priors}; SyMetric measures whether a model learned from pixels recovers the true dynamics \citep{higgins2021symetric}. The properties of the leapfrog step are classical \citep{arnold1989mathematical, hairer2003geometric}, and we do not claim them as new.

\textbf{Hamiltonian world models.} \citet{cui2026physically} propose Hamiltonian
World Models, a perspective that decodes predicted latent trajectories into future
observations and plans over the resulting rollouts; as a position paper, it reports no
experiments. \citet{tang2026ham} decompose the latent into a canonical and a context
subspace and train it with a JEPA-style objective, and add the action as an external control term, which makes the update
non-symplectic.
\citet{luan2026ph} use port-Hamiltonian dynamics as a structural regularizer alongside the recurrent transition of a decoder-based
Dreamer-style agent \citep{hafner2019dream, hafner2020mastering}. \algo~differs in that the action enters the Hamiltonian rather than an
external control term, and every rollout step is exactly symplectic, so phase-space volume
is preserved by construction, learned end to end from pixels without a decoder.

\textbf{Structured predictors in JEPAs.} HamJEPA \citep{alvarez2026beyond} maps one augmented view to another with a leapfrog step of a Hamiltonian, to learn image representations evaluated by classification accuracy, with no time and no actions. LIFS-JEPA \citep{belloulata2026contractive} refines one prediction by iterating a contractive map, where converging is the goal. In a world model each iteration is a step in time, and a contracting predictor makes distinct states merge.

Overall, \algo{} brings Hamiltonian learning into JEPA world models learned from video with actions for a problem that, to our knowledge, none of these works targets: the collapse of long rollouts (Appendix~\ref{app:compare-prior}).

\section{LeHami: Symplectic Hamiltonian JEPA}
\label{sec:ham_jepa}
\input{method}

\section{Empirical validation of LeHami}
\label{sec:experiments_new}
\input{experiments}

\section{Conclusion} \label{sec:conclusion}
We introduced \algo{}, a predictor for latent world models whose exact symplectic step preserves volume, which provably eliminates some rollout collapse modes, and integrates in continuous time, invertibly given the action history. On five environments and four measures taken on the data, with the same encoder and state, an MLP and a transformer predictor contract long rollouts while \algo{}'s do not merge, and \algo{} is the most accurate of the three in position at long horizons on four of five, but not at short horizons on two of five. Run backward with the recorded actions, it recovers the true past.

\paragraph{Limitations and future work.}
First, strongly dissipative systems (e.g., heavy friction) contract phase-space volume, which a volume-preserving step cannot follow; port-Hamiltonian extensions (Appendix~\ref{app:design}) are a natural direction for future work. Our experiments use simulated video of physical systems; extending them to contact, real video, observation noise, encoder error and complex actuation (e.g., LeWM's Push-T) is the next step.  Second, our $H_\phi$ is separable with a mass that is constant within an episode; on the cart pole, whose mass depends on the configuration, \algo{} is less exact, and on two-body at the suite's $0.05$\,s per frame it fails (Appendix~\ref{app:add-physics}). Third, the volume guarantee does not bound the error, and planning does not beat LeWM yet.
Finally, a learned momentum could combine a richer encoder's accuracy with rollouts that do not collapse.

\subsection*{AI use statement}

AI was used for the writing of some parts of the manuscript and for some parts of the code development used to run the experiments. Moreover, AI was used to check the proof of Proposition 2. The authors take full responsibility for the scientific
contributions and all written content.

\subsection*{Ethics statement}
This work aims to advance fundamental machine learning research, specifically in the domain of world models. As foundational research, we do not foresee any immediate societal impacts associated with this work.

\subsection*{Acknowledgements}
R.Y. thanks Alfredo Canziani for his guidance and feedback throughout this work, and
Helmut Katzgraber for his insights during the planning of the project.

\subsection*{Reproducibility statement}
The claims of this paper are supported by the appendices and by the code release
at \url{https://github.com/rominay/LeHami}.
\emph{Theory.} Proposition~\ref{prop:volume} is stated with its assumptions in
Section~\ref{sec:ham_jepa} and proved in full in Appendix~\ref{app:volume};
Appendix~\ref{app:remarks} sets out what the guarantee does and does not imply.
\emph{Models and training.} What is held fixed between \algo{} and the baselines, and how
checkpoints are selected, is in Appendix~\ref{appendix:methodology}. The architecture as
implemented and the parameter counts measured from the checkpoints are in
Appendix~\ref{appendix:architecture} and Table~\ref{tab:multi-params}, and the remaining
design choices (the SG stencil, the scale gauge and the SIGReg placements) are in
Appendix~\ref{app:design}. Every trained model corresponds to one configuration file in
the release, which fixes all its hyperparameters, and the resolved configuration is saved with each
checkpoint. The third-party code we build on is pinned to the exact commits we used.
\emph{Data.} All five environments are specified by their equations of motion, rendering
and generation parameters in Appendix~\ref{appendix:datasets} and Appendix~\ref{app:add-envs}. The episode-level
train/validation/test protocol, including the generator seed and the check that no episode
is shared across splits, is in Table~\ref{tab:datasets}. The release contains the seeded
generators and the split manifests that training and evaluation both read, and each manifest records the hash of the
dataset file it was cut from, so a regenerated dataset can be verified.
\emph{Evaluation.} The probes are described in Appendix~\ref{app:probes}, and the planning protocol, with
its goal construction, CEM optimiser, leakage checks and success criterion, is in
Appendix~\ref{app:planning}. Results are reported as the mean and standard deviation over three training seeds unless
stated otherwise.
\emph{Code.} One runner script per experimental block trains and evaluates every model
and seed reported in the paper, and the planning and compute analyses reuse its checkpoints. The
notebooks regenerate the figures and tables from the frozen numbers shipped with
them, without a GPU, the datasets or retraining.

\bibliography{iclr2027_conference}
\bibliographystyle{iclr2027_conference}

\appendix
\section{Appendix organization}
Appendix~\ref{appendix:background_hamiltonian} collects the background on the Hamiltonian mechanics we use,
including phase-space density and the symplectic form, and compares \algo{} with prior Hamiltonian models (\S\ref{app:compare-prior}).
Appendix~\ref{appendix:theory} contains the theory: the formal statement and proof that
the leapfrog step is exactly symplectic (\S\ref{app:volume}); the scope of that guarantee,
in three remarks on what volume preservation does and does not imply (\S\ref{app:remarks});
why we build on the Hamiltonian rather than the Lagrangian formulation
(\S\ref{hamiltonian_no_lagrangian}); and the design choices in the flow and the encoder
(\S\ref{app:design}): the shape of the Hamiltonian, the SG stencil and its causal and centred
windows, why no mass is learned, the learned scale gauge, and the SIGReg placements.
Appendix~\ref{appendix:experimenta_setup} gives the experimental set-up in full: what is held
fixed across models and how checkpoints are selected (\S\ref{appendix:methodology}); the
implementation of \algo{} (\S\ref{appendix:architecture}), covering the Hamiltonian flow, the
per-frame and temporal encoders, the stencil at training and inference, and an extension to
training across recording rates that the architecture supports but that we do not exercise here;
the five environments and their split protocol (\S\ref{appendix:datasets}); and the evaluation,
namely the probes (\S\ref{app:probes}), the four measures of collapse (\S\ref{app:add-measures}) and the planning protocol (\S\ref{app:planning}), with its goal
construction, the frames each latent state is read from, the optimiser, the leakage checks and
the success criterion.
Appendix~\ref{appendix:supplementary_results} reports supplementary results. It first shows
that the Hamiltonian structure replaces manual design choices at lower training compute and
without extra parameters (\S\ref{appendix:design-choices}). It then gives the full results on
the conservative pendulum (\S\ref{appendix:results_cons}): the geometry of the latent space
against the learned and physical energies, rollout prediction across probe tiers, independence
from the encoder with a frozen DINOv2, and the phase-space volume of the one-step map. Finally, it
gives the full results on the actuated pendulum (\S\ref{appendix:results_act_pendulum}): the
learned energy under actuation, where to apply the regularizer, rollout prediction, how a
temporal window in the encoder captures velocity, the goal distribution, the planning compute,
a decomposition of the velocity clause of the planning criterion into speed and direction, and
the prediction of the rod's length and mass.
It ends with the physical errors on all five environments (\S\ref{app:add-physics}), \algo{} against LeWM with its own state (\S\ref{app:add-lewm}), the backward rollout (\S\ref{app:add-backward}), the mass carried by the rollout (\S\ref{app:add-mass}) and further controls (\S\ref{app:controls}).
\input{background}

\input{appendix_theory_design}
\input{appendix_experiments}

\end{document}

%% file: math_commands.tex
\usepackage{amsmath,amsfonts,bm}

\def\eqref#1{equation~\ref{#1}}

\def\1{\bm{1}}

\DeclareMathAlphabet{\mathsfit}{\encodingdefault}{\sfdefault}{m}{sl}
\SetMathAlphabet{\mathsfit}{bold}{\encodingdefault}{\sfdefault}{bx}{n}

\newcommand{\E}{\mathbb{E}}

\newcommand{\R}{\mathbb{R}}

\newcommand{\N}{\mathbb{N}}

%% file: intro.tex
Joint Embedding Predictive Architecture (JEPA) \citep{lecun2022path} has emerged as a promising framework for action-conditioned latent world modeling. While existing JEPA formulations successfully prevent \textit{encoder collapse:} mapping distinct inputs to constant representations, via training-time regularization \citep{bardes2021vicreg, balestriero2025lejepa} or frozen backbones \citep{assran2025v, zhou2024dino}, maintaining consistent dynamics over extended evaluation rollouts remains a persistent challenge \citep{maes2026leworldmodel, zhang2026hierarchical}. Long-horizon stability is further hindered by \textit{predictor contraction}: in an autoregressive rollout each prediction is fed back as the next input, so even a slight per-step contraction of the transition map compounds over the horizon and progressively shrinks the latent distribution. %

\begin{figure}[t]
    \centering
    \relabel{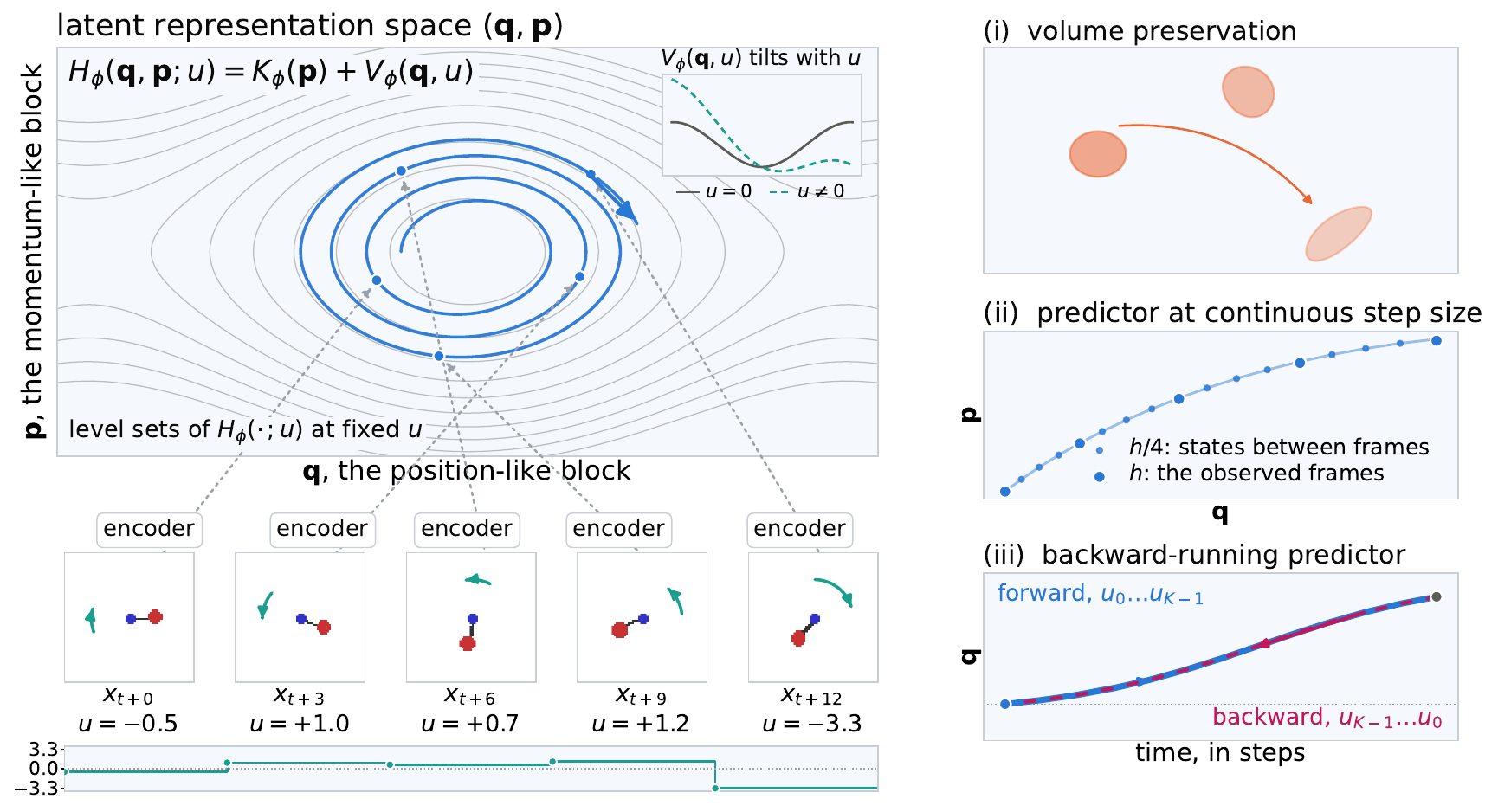}{830.247}{%
      \fill[fill={rgb,255:red,245;green,248;blue,252}] (140.0,31.0) rectangle (150.45,46.0);
      \node[inner sep=0, anchor=base west] at (140.4,44.5) {\fontfamily{phv}\fontsize{7.8}{8}\selectfont\itshape T};}
    \caption{\textbf{\algo: Learning symplectic world models end-to-end.}
    \textbf{Bottom left:} Actuated pendulum frames and applied torque sequence $u_t$ (green).
    \textbf{Upper left:} Latent phase space $(\bq_t, \bp_t)$, where transitions are generated by a learned scalar $H_\phi(\bq_t,\bp_t;u_t) = T_\phi(\bp_t) + V_\phi(\bq_t,u_t)$ via a symplectic step. Grey lines indicate level sets; the blue curve shows the rollout.
    \textbf{Right:} Three properties of the Hamiltonian flow:
    \textbf{(i) Volume preservation:} Latent volume is exactly preserved under any action sequence.
    \textbf{(ii) Continuous step size:} $H_\phi$ integrates at arbitrary time intervals, including between observed frames.
    \textbf{(iii) Invertibility:} Integrates backward by flipping the step ($\Delta t \rightarrow -\Delta t$) and replaying the actions in reverse order.}
    \label{fig:concept}
\end{figure}

Physics avoids this contraction through geometric \textit{structure}. Hamiltonian mechanics governs dynamics via a single scalar $H$, constraining conservative trajectories to energy level sets and preserving phase-space volume at every discrete step \citep{arnold1989mathematical}. Crucially, volume preservation persists under actuation, provided the action enters the Hamiltonian. Port-Hamiltonian extensions further cover dissipative systems \citep{marsden1999introduction, zhong2020dissipative, duong2021hamiltonian}, making it a physically grounded prior for latent spaces rather than an idealization. While $H$ is traditionally derived analytically \citep{goldstein2011classical}, recent methods parameterize a learnable $H_\phi$ directly from trajectories or pixel inputs \citep{greydanus2019hamiltonian, chen2019symplectic, toth2019hamiltonian}. \algo{} brings this line of work into JEPA world models learned from video under actions, to address one of their failures: the collapse of long rollouts. Section~\ref{sec:related} sets out how it differs from prior Hamiltonian models.

We propose \textbf{\algo} (Symplectic Hamiltonian JEPA), a framework that embeds the latent representation space onto a symplectic manifold where evolution is driven by a learnable action-conditioned Hamiltonian flow $H_\phi(\bq, \bp, \bu)$ over a canonical latent pair $(\bq, \bp)$, where $\mathbf{u}$ denotes the control action vector, integrated by an exact symplectic leapfrog step \citep{hairer2003geometric} and trained end-to-end without decoders or supervised coordinates. \algo{} inherits three structural properties from the Hamiltonian flow, illustrated in Fig.~\ref{fig:concept} (right):

\begin{enumerate}[label=(\roman*), leftmargin=*, noitemsep]
    \item \textbf{Exact volume preservation:} Symplectic integration preserves phase-space volume by construction, provably eliminating some representation collapse modes during evaluation rollouts for arbitrary action sequences (Proposition~\ref{prop:volume}).

    \item \textbf{Continuous-time step size:} Because $H_\phi$ parameterizes a continuous vector field, a single trained model can be integrated at arbitrary step sizes $\Delta t$ without retraining, facilitating training on multi-frequency data \citep{doshi2024scaling} and non-uniform planning grids \citep{tan2016model}.
    \item \textbf{Invertibility given the action history:} Flipping the step ($\Delta t \to -\Delta t$) and replaying the actions in reverse order allows backward-in-time rollouts, serving as an essential primitive for computing backward reachable tubes and online safety guarantees in optimal control \citep{mitchell2005time, herbert2017fastrack}.
\end{enumerate}

Moreover, a central objective in physical world modeling is learning representations that capture the \emph{complete physical state} of an environment, especially for control tasks like object interception \citep{kim2014catching} and precision throwing \citep{zeng2020tossingbot}. In physics, uniquely determining future evolution requires canonical momentum alongside position within a structured phase space. Without this geometric prior, standard JEPA architectures with unstructured predictors resort to architectural design, such as expanding the encoder’s temporal window or context-conditioned predictors to capture velocity, lacking guarantees. In contrast, \algo{} leverages this prior to model the latent space as a canonical phase space, enabling complete physical representations.

\paragraph{Contributions.} We summarize our key contributions as follows:
\begin{itemize}
    \item \textbf{Symplectic latent world model (\algo):} We introduce an end-to-end trainable JEPA framework equipped with a learned separable Hamiltonian predictor over a canonical latent split $(\bq, \bp)$. We prove that its volume-preserving step eliminates three specific collapse modes during evaluation rollouts across arbitrary horizons and actions (Proposition~\ref{prop:volume}).
    \item \textbf{Rollout collapse, measured on the data:} To our knowledge, this is one of the first works to show rollout collapse in JEPA world models with measures taken on the data. With the encoder and latent state fixed, an MLP and a transformer predictor contract the rollout on all five environments and four measures, while \algo{}'s rollouts do not merge (Section~\ref{sec:collapse}).

    \item \textbf{Continuous-time \& invertibility:} We show that a single trained predictor $H_\phi$ acts as a continuous vector field, enabling sub-frame time integration without performance degradation, and that run backward with the recorded actions it recovers earlier frames as well as it predicts later ones.

    \item \textbf{Complete physics representation \& principled planning:} By construction rather than by architectural tuning, \algo{} represents the full state through exact geometric structure. Across five environments, \algo{} is the most accurate of the same-state predictors at long horizons in position on four of five and in velocity on three of five (the pendulums and the mass-spring), and beats chance in both on all five. It can also be used for planning, with the limits we report.
\end{itemize}

%% file: method.tex
 \algo{} makes the latent dynamics of a JEPA world model Hamiltonian. Its latent state is a canonical pair $(\bq_t,\bp_t)$: the encoder outputs the position-like $\bq_t$, and the momentum-like $\bp_t$ is its time derivative, obtained by zero-parameter numerical differentiation of the encoder trajectory. The predictor advances this pair by one leapfrog step of a learned, action-conditioned Hamiltonian, so every step of the rollout is a symplectic map. We first state what this structure guarantees (Proposition~\ref{prop:volume}), then give the flow, the multi-step objective and the design choices.

A symplectic map preserves the canonical two-form $\Omega = \sum_i dq_i \wedge dp_i$, the oriented area spanned in each position–momentum plane (Appendix~\ref{appendix:background_hamiltonian}), and with it phase-space volume: a cloud of latent states can be stretched, rotated and sheared by the rollout, and squeezed in some directions only if it is stretched in others: its total volume never shrinks. Constraining the predictor to be such a map therefore rules out, at every step, the contraction available to an unconstrained rollout (Proposition~\ref{prop:volume}).

\begin{proposition}[\textbf{Symplectic volume preservation and the rollout failure modes it removes}]
\label{prop:volume}
Let $\Psi_h$ be the leapfrog step of Eq.~\ref{eq:leapfrog} for step size $h$ generated by a separable
$H_\phi = T_{\phi_1}(\bp) + V_{\phi_2}(\bq,\bu)$ with $\bu$ held fixed within the step. Then
$\Psi_h$ is an exactly symplectic diffeomorphism of $\R^{2d}$ for every step size $h$ and every parameter $\phi$; it then preserves phase-space volume
exactly at every horizon. Here $h$ is the integrator's step, which equals the observation
interval $\Delta t$ during training but need not at inference (Section~\ref{section:experiments_time}). Consequently, three collapse modes are ruled out: uniform contraction, convergence to an asymptotically stable fixed point, and the merging of distinct states. They are stated and proved in Appendix~\ref{app:volume}.
\end{proposition}

Crucially, symplecticity in the latent representation space is a \emph{structural} guarantee as it holds at every inference step by construction, not only in
expectation over training batches. By endowing the rollout with symplectic geometry, it may also improve long-horizon consistency and planning performance. Moreover, it has certain properties, as discussed in Section \ref{section:intro}, that are of interest for many applications and that no unstructured predictor can have. We chose the Hamiltonian framework rather than the Lagrangian one; the reasons are given in Appendix \ref{hamiltonian_no_lagrangian}.

\textbf{Why a free predictor can contract.} A world model is trained on short clips ($K=8$ steps here) and then rolled out far beyond them, each prediction fed back as the next input. Let $J_i$ be the Jacobian of the one-step predictor at step $i$. Each step multiplies the volume of a small cloud of latent states by $|\det J_i|$, and two rollouts from nearby starts differ, to first order, by $J_{k-1}\cdots J_0\,(\bz^a_0-\bz^b_0)$. An MLP or a transformer puts no constraint on $\det J_i$, and beyond the training horizon nothing in the loss does either: a contraction of $0.99$ per step is barely visible in training ($0.99^8\approx0.92$) and collapses a long rollout ($0.99^{50}\approx0.6$). Encoder regularizers such as SIGReg \citep{balestriero2025lejepa} and VICReg \citep{bardes2021vicreg} act on the encoder output, not on the predictor. A symplectic step has $|\det J_i|=1$ by construction.

\textbf{Scope and limitations of the guarantee.} Three immediate remarks follow Proposition~\ref{prop:volume}. We state them here in short form; Appendix~\ref{app:remarks} gives the arguments in full.
\emph{(i) Shape, not size.} Volume preservation forbids the degenerate failures but
it does not bound the growth of rollout error. It holds over the whole latent space, for any $H_\phi$, trained or random, and says nothing about the thin surface the data occupy, so we measure collapse on the data (Section~\ref{sec:collapse}). Nor does it exclude a latent that does not move: a constant encoder with $\nabla V = 0$ is a neutrally stable fixed point.
\emph{(ii) Regularization is not made redundant.} Symplecticity preserves volume, not
distributional shape. The two act at different places: The regularizer acts on
the encoder distribution in expectation over batches, while symplecticity on every inference step.
\emph{(iii) No tie to the physical Hamiltonian.} Nothing in training ties the latent or $H_\phi$ to the physical ones: the encoder may choose any representation in which the flow predicts well, and physical quantities are read by frozen probes. The guarantee holds for any learned $H_\phi$.

\begin{figure}
    \centering
    \relabel{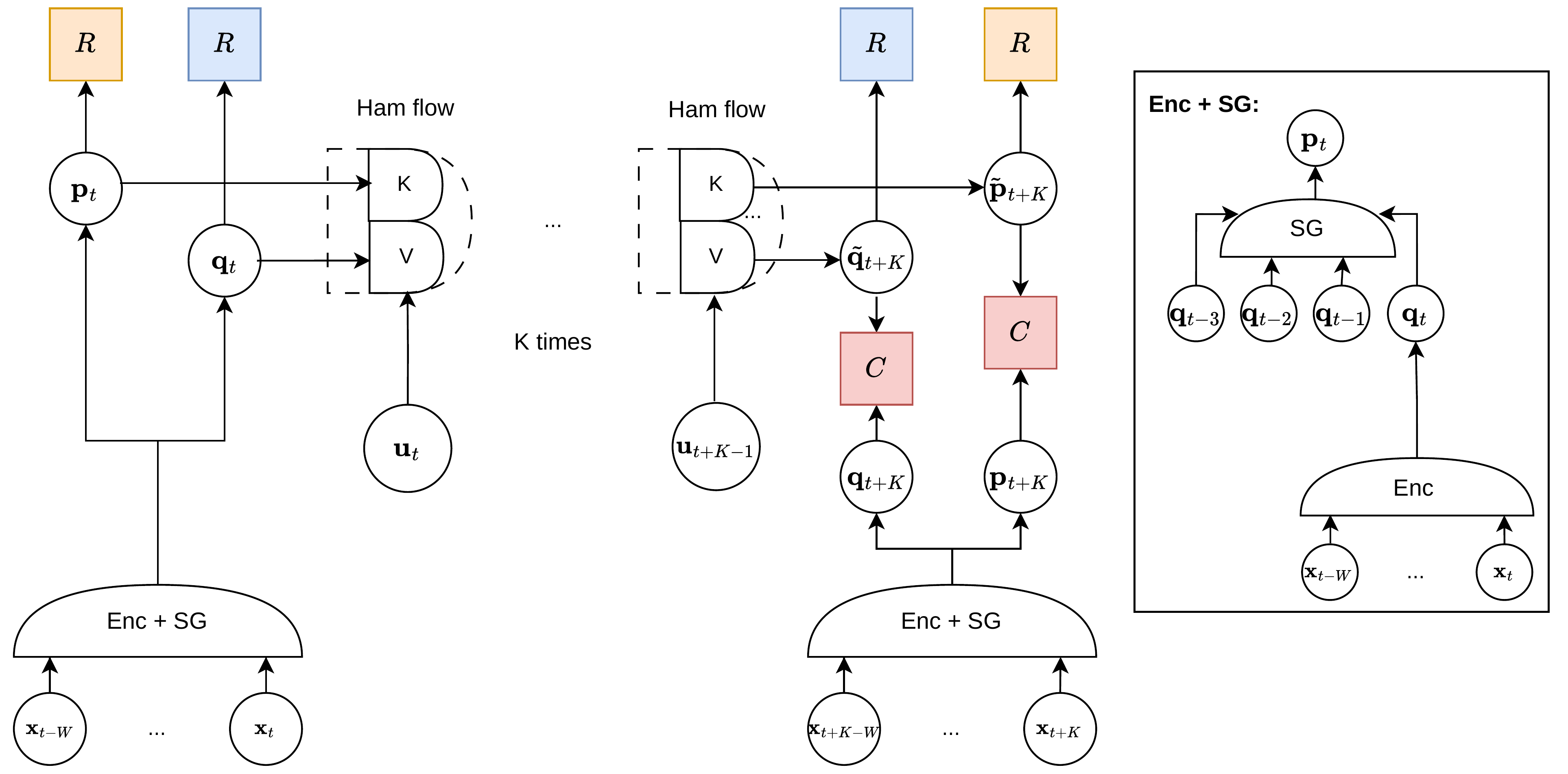}{1849.92}{%
      \fill[white] (469.0,206.5) rectangle (486.5,227.0);
      \fill[white] (836.5,206.5) rectangle (854.0,227.0);
      \node[inner sep=0, anchor=base] at (477.75,225.5) {\fontfamily{phv}\fontsize{5.25}{6}\selectfont T};
      \node[inner sep=0, anchor=base] at (845.25,225.5) {\fontfamily{phv}\fontsize{5.25}{6}\selectfont T};}
    \caption{\textbf{\algo{} multi-step training.}
    \textbf{Left:} at the anchor, the encoder block (Enc + SG, a Savitzky--Golay stencil) maps a window of $W+1$ frames, drawn here as the causal window $\bx_{t-W},\dots,\bx_t$, to the latent state $(\bq_t,\bp_t)$. The Hamiltonian flow
    (Ham flow) advances it by one leapfrog step of the separable Hamiltonian
    $H_\phi = T_{\phi_1}(\bp) + V_{\phi_2}(\bq,\bu)$, with the action entering the potential,
    and is applied $K$ times to its own outputs.
    \textbf{Middle:} each prediction $(\tilde\bq_{t+k},\tilde\bp_{t+k})$ is scored by the
    prediction loss (C) against the encoding of the corresponding frames, shown here for the
    last step $k=K$. The regularizer (R) acts on the encoder's embeddings of $\bq$ (blue) and $\bp$ (orange).
    \textbf{Right:} the encoder block as instantiated in this work: a pixel encoder (Enc)
    outputs $\bq_t$, and a fixed Savitzky--Golay stencil (SG) derives $\bp_t$ from
    $\bq_{t-3},\dots,\bq_t$.}
    \label{fig:diag_\algo}
\end{figure}

\subsection{The architecture}

We train on unannotated trajectories of observations $[\bx_0,\dots,\bx_{T-1}]$ and actions
$[\bu_0,\dots,\bu_{T-1}]$. \algo{} has two learned components, an encoder $\text{Enc}_\theta$
and a Hamiltonian flow ($\text{Ham flow}_\phi$) that replaces the predictor, joined by the zero-parameter Savitzky--Golay (SG) stencil \citep{gorry1990general} that reads the momentum off the encoder trajectory (Fig.~\ref{fig:diag_\algo}):

\begin{equation}
\begin{aligned}
\text{Encoder + SG:}\quad & \bq_t =  \text{Enc}_\theta(\bx_{\mathcal{W}(t)}),
  \qquad \bp_t = \text{SG}(\bq_{t-3},\dots,\bq_t), \\
\text{Hamiltonian flow:}\quad & (\tilde\bq_{t+1},\tilde\bp_{t+1})
  = \text{Ham flow}_\phi(\bq_t,\bp_t,\bu_t),
\end{aligned}
\end{equation}
where $\mathcal{W}(t)$ is a window of $W+1$ that includes $\bx_t$ ($W{=}0$ reads $\bx_t$ alone; Appendix~\ref{appendix:architecture}).
The flow advances the state by one step of Hamilton's equations for a learned scalar
$H_\phi(\bq,\bp,\bu)$,
\begin{equation}
\dot{\bq} = \frac{\partial H_\phi}{\partial \bp}, \qquad
\dot{\bp} = -\frac{\partial H_\phi}{\partial \bq},
\label{eq:ham_eq}
\end{equation}
and $\theta$ and $\phi$ are trained end-to-end by backpropagating through the rollout, so the
encoder is shaped to produce latent states that admit this structured evolution.

\textbf{The integration.} In principle, different integration techniques can be used. We leverage the leapfrog (St\"ormer--Verlet) integrator \citep{hairer2003geometric}, which
is second-order accurate and, crucially, exactly symplectic. For a separable Hamiltonian
\begin{equation}
\label{eq:hamiltonian}
H_\phi(\bq, \bp, \bu) = T_\phi\!\left(\mathbf{S}^{-1}\bp\right) + V_\phi(\bq, \bu),
\qquad  \mathbf{S}= \operatorname{diag}\!\left(e^{\boldsymbol{\sigma}}\right),
\end{equation}

the update over a time step $\Delta t$ is
\begin{equation}
\label{eq:leapfrog}
\begin{aligned}
    \Tilde{\bp}_{t+\frac{1}{2}} &= \bp_t - \tfrac{\Delta t}{2}\,\nabla_\bq V_\phi(\bq_t, \bu_t), \\
    \Tilde{\bq}_{t+1}           &= \bq_t + \Delta t\,\mathbf{S}^{-1}\,\nabla T_\phi\!\left(\mathbf{S}^{-1}\Tilde{\bp}_{t+\frac{1}{2}}\right), \\
    \Tilde{\bp}_{t+1}           &= \Tilde{\bp}_{t+\frac{1}{2}} - \tfrac{\Delta t}{2}\,\nabla_\bq V_\phi(\Tilde{\bq}_{t+1}, \bu_t),
\end{aligned}
\end{equation}
where $\boldsymbol{\sigma} \in \mathbb{R}^d$ is learned and absorbs the relative scale of $\bq$ and $\bp$, which the self-supervised objective leaves
undetermined. Because $\mathbf{S}$ is a constant linear rescaling, the flow remains exactly symplectic.

\textbf{Action-conditioned Hamiltonian.} Under actuation the controller does work and $H_\phi$ is no longer constant, but
symplecticity is a property of the map. We condition the potential on the action so
$V_{\phi_2}=V_{\phi_2}(\bq,\bu)$ and the generalized force
$\mathbf{F}(\bq,\bu) = -\partial V_{\phi_2}/\partial \bq$ is a $\bq$-gradient for every action
by construction. Separability is
untouched and the leapfrog of Eq.~\ref{eq:leapfrog} stays \emph{exactly} symplectic for every
action sequence, provided only that $\bu_t$ is constant within a step.

\textbf{Multi-step learning.} We consider multi-step prediction over a $K$-step horizon by iteratively applying the Hamiltonian flow. The loss is the average MSE across the trajectory. Given an anchor of $A$ time steps used to form the embedding of the first time step of the episode, and $K$ self-fed leapfrog steps, the loss is
\begin{align}
  \mathcal{L} &=  \frac{1}{K}\sum_{j=1}^{K}
       \big\|\tilde \bq_{A+j} - \bq_{A+j}\big\|^{2}
      + \lambda_p\,\frac{1}{K}\sum_{j=1}^{K}
       \big\|\tilde \bp_{A+j} - \bp_{A+j}\big\|^{2}
      + \lambda_{\mathcal{S}} \mathcal{S},
  \label{eq:total}
\end{align}

where the first term is the prediction loss in the position-like embedding, the second is the prediction loss in the momentum-like embedding weighted by $\lambda_p$ and the last term $\mathcal{S}$ is the regularizer, which, in principle, can be any (e.g., SIGReg \citep{balestriero2025lejepa}) and it can be applied to the embeddings $\bq$ and $\bp$ independently, to only $\bq$, or to both embeddings concatenated. It acts on ground-truth embeddings, never
on the rollout. We compare the performance across the three placements in Table \ref{tab:episode-variance} in Appendix~\ref{appendix:results_act_pendulum}.

\textbf{Design choices.} \emph{The shape of $H_\phi$.} While the core LeHami framework is general and agnostic to the choice of Hamiltonian, in this work we adopt a separable structure. This design choice enables fast, explicit leapfrog integration without requiring iterative numerical solvers. The guarantee is not specific to separable $H$: for a non-separable Hamiltonian, e.g., exact symplecticity is retained by replacing leapfrog with a non-separable symplectic integrator \citep{tao2016explicit}. \emph{How $\bp$ is obtained.} We enforce the canonical relation $\mathbf{p} = \dot{\mathbf{q}}$ by computing $\mathbf{p}$ as the temporal derivative of generalized coordinates $\mathbf{q}$ via a fixed Savitzky--Golay (SG) stencil \citep{gorry1990general} on the encoder trajectory. This is a least-squares quadratic fit over $\bq$ that is then differentiated at time step $t$.  The stencil
is linear with constant coefficients, so it adds no learned parameters and cannot itself
collapse. Computing $\bp$ as $\dot\bq$ assumes unit mass. A mass that is constant within an episode, possibly different between episodes, is absorbed by the encoder through a stretched coordinate $\bq' = c\,\bq$; a mass that depends on the configuration, $\mathbf{M}(\bq)$, is not, and the cart pole of Section~\ref{sec:experiments_new} tests that case. More details in Appendix~\ref{app:design}.

%% file: experiments.tex
We evaluate our approach on five environments rendered as $64\times64$ RGB images.
In the conservative pendulum there is no actuation, and episodes vary
across three rod lengths. In the actuated pendulum, a recorded
sequence of scalar torques $(u_t)_t$, drawn uniformly, does work on the system, so energy is
no longer conserved; episodes vary across the same three rod lengths and three masses, since under actuation the mass sets the response to the
torque. In both environments the true dynamics preserve phase-space volume, which is the
property \algo~builds in; the actuated pendulum tests it once energy conservation no longer
holds. Appendix~\ref{appendix:datasets} describes each in full (Fig.~\ref{fig:datasets}).
To test beyond the pendulum, we add the mass-spring and the two-body system of the DeepMind Hamiltonian Dynamics Suite \citep{botev2021priors}, and a cart pole with Gymnasium's physics and no force \citep{towers2024gymnasium}. The cart pole has two degrees of freedom and a mass that depends on the pole angle, so it breaks the unit-mass assumption of Section~\ref{sec:ham_jepa}. Appendix~\ref{app:add-envs} describes them.

\subsection{Experimental setup}

\begin{figure}[t]
    \centering
\includegraphics[width=0.8\linewidth]{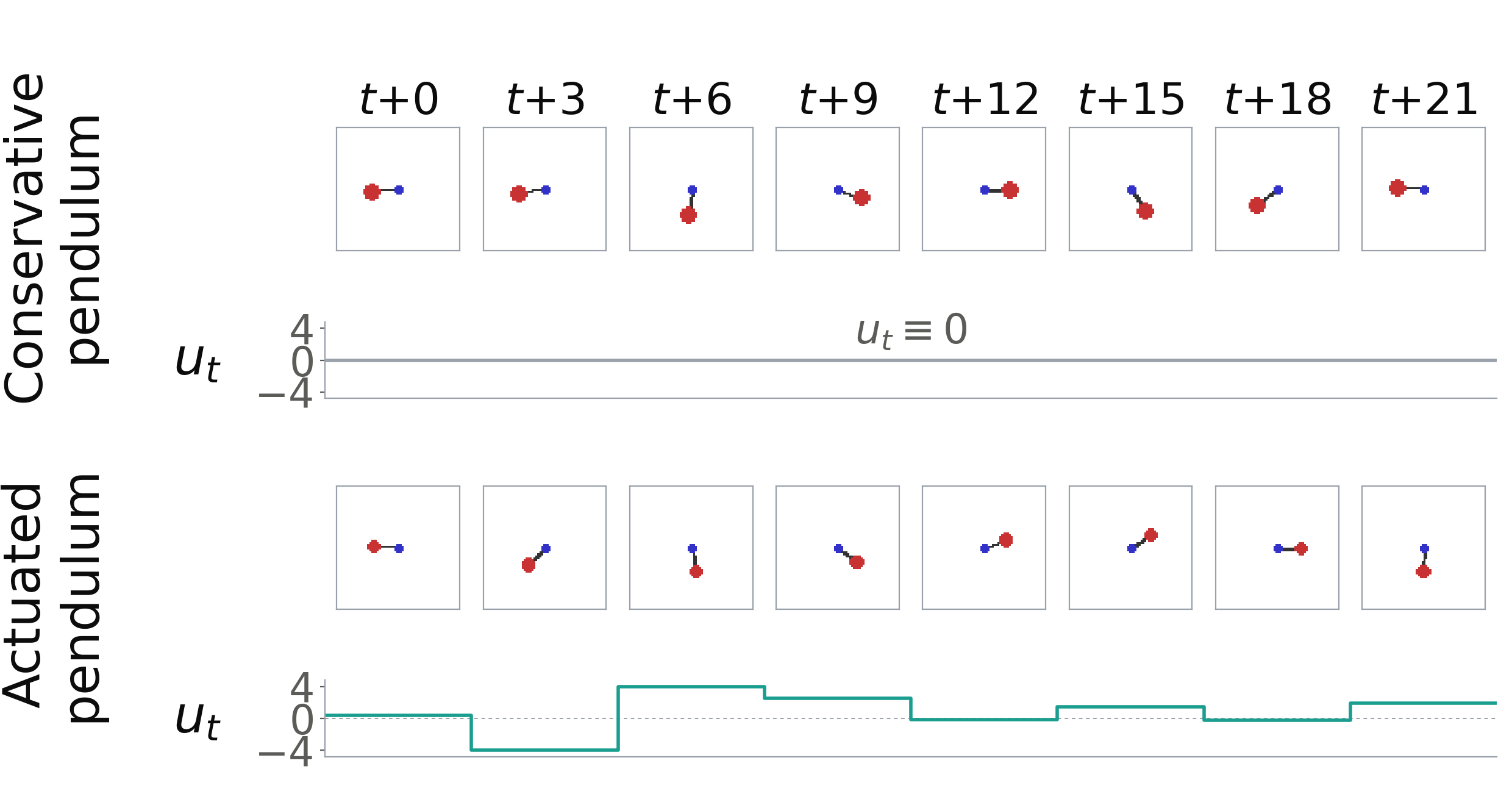}
    \caption{\textbf{The two pendulum environments.} \textbf{Top:} the conservative pendulum, with no actuation. \textbf{Bottom:} the actuated pendulum. Beneath each
row is its respective torque $u_t$, drawn as the piecewise-constant signal, held constant for three frames.}
    \label{fig:datasets}
\end{figure}

\textbf{Methodology.} \algo~is designed as a drop-in replacement for the predictor of a
JEPA world model, so we compare against LeWorldModel
(LeWM) \citep{maes2026leworldmodel}, the state-of-the-art among JEPA world models trained end-to-end from pixels, which reported gains over PLDM \citep{sobal2025learning} and DINO-WM
\citep{zhou2024dino}. LeWM's predictor is unstructured: a causal transformer that takes
three context embeddings to predict the next one. \algo~is encoder-agnostic and it replaces only this predictor, plus
the fixed Savitzky--Golay (SG) stencil that computes $\bp$ (Section~\ref{sec:ham_jepa}). The
main text uses a temporal ConvNet \citep{tran2018closer} (Appendix Fig.~\ref{diagram:convet}) and Appendix~\ref{appendix:results_cons}
repeats the comparison with a tiny ViT \citep{wu2022tinyvit}. Moreover, \algo{} is plugged on top of a frozen pre-trained encoder with DINOv2 features \citep{oquab2023dinov2}, yielding DINO-LeHami. We do not compare against
reward-driven agents such as Dreamer \citep{hafner2019dream} or TD-MPC \citep{hansen2022temporal} variants, since they require task rewards
whereas our setting, like LeWM's, is reward-free, goal-conditioned planning. More details on methodology in Appendix~\ref{appendix:methodology}.

\textbf{Same-state controls.} To separate the Hamiltonian step from the augmented state, we keep \algo{}'s encoder (the per-frame ConvNet), state $[\bq, \mathrm{SG}(\bq)]$, losses, regularizer and 8-step training rollout, and replace only the predictor: an MLP with the same number of parameters as \algo{}'s flow, and by LeWM's transformer, both with the same augmented state as \algo{}. LeWM with its own state, from one frame (the baseline above) or three, is compared in Appendix~\ref{app:add-lewm}. Results are over three training seeds, next to two baselines that use no model: repeat the start, and predict the mean.

\textbf{Probing.} The pendulum's state is fully described by the angle $\theta$ and the angular velocity $\omega$,
and we score their prediction along the rollout by MAE against ground truth; for $\theta$ we use
MAE rather than $R^2$, whose normalising variance is undefined on the circle. Three probe tiers
of increasing capacity read the latent: a ridge, a $k$-NN ($k=25$) and a two-layer MLP. Each is fitted once, either on the encoder's own features (the anchor, K=0) or on the model's
one-step output (K=1), and
applied frozen. We report the results in the test split of episodes. The $k$-NN fitted at K=1 puts the models on equal footing
because LeWM's encoder is per-frame, and only its predictor sees more than one time step. We use it throughout the main text. On the other three environments it reads the position and the velocity: $q$ and $p$ for the mass-spring, the position and velocity of a body for two-body, and the pole angle $\phi$ and its rate $\dot\phi$ for the cart pole.

\subsection{Rollout collapse, measured on the data} \label{sec:collapse}

Proposition~\ref{prop:volume} holds over the whole latent space, for any learned $H_\phi$, trained or random. The data, however, lie on a thin surface inside it, with as many dimensions as the physical state, and a map that keeps the total volume could still squeeze that surface while stretching other directions. We therefore measure collapse on the data, with four ratios of the rollout to the truth (the encoder's latents of the true frames at the same step). Each is $1$ when the rollout keeps the true geometry and below $1$ when it contracts: the \emph{distance between rollouts} started from different frames; the \emph{tangent area}, how much the rollout stretches a small patch of the data surface compared with the true motion; the \emph{spread} of the cloud of all rollout states; and its \emph{effective rank} (Appendix~\ref{app:add-measures}).

\begin{table}[t]
\centering
\caption{\textbf{Rollout collapse on the data, at the longest step.} Each cell gives four ratios of the rollout to the truth, separated by slashes: distance between rollouts / tangent area / spread / effective rank; $1$ is right, below $1$ the rollout contracts, above $1$ it drifts. Mean over three seeds (largest s.d.\ $0.17$). Steps $K$: $30$ for the pendulums, $50$ for the mass-spring, $70$ for two-body and the cart pole.}
\label{tab:collapse}
\resizebox{\linewidth}{!}{%
\begin{tabular}{lccccc}
\toprule
predictor & cons.\ pendulum & act.\ pendulum & mass-spring & two-body & cart pole \\
\midrule
\algo{} & 1.03/1.35/1.04/1.25 & 1.00/1.04/1.01/1.21 & 1.10/1.22/1.15/1.57 & 1.50/0.94/1.59/0.83 & 1.33/0.58/1.42/0.94 \\
MLP, matched & 0.81/0.49/0.78/0.56 & 0.91/0.68/0.90/0.84 & 0.73/0.17/0.70/0.37 & 0.85/0.75/0.85/0.94 & 0.72/0.22/0.71/0.56 \\
Transformer (LeWM's) & 0.86/0.62/0.83/0.70 & 0.92/0.81/0.90/0.90 & 0.83/0.37/0.80/0.59 & 0.84/0.76/0.84/0.95 & 0.63/0.20/0.64/0.61 \\
\bottomrule
\end{tabular}}
\end{table}

\begin{figure}[t]
    \centering
    \includegraphics[width=\linewidth]{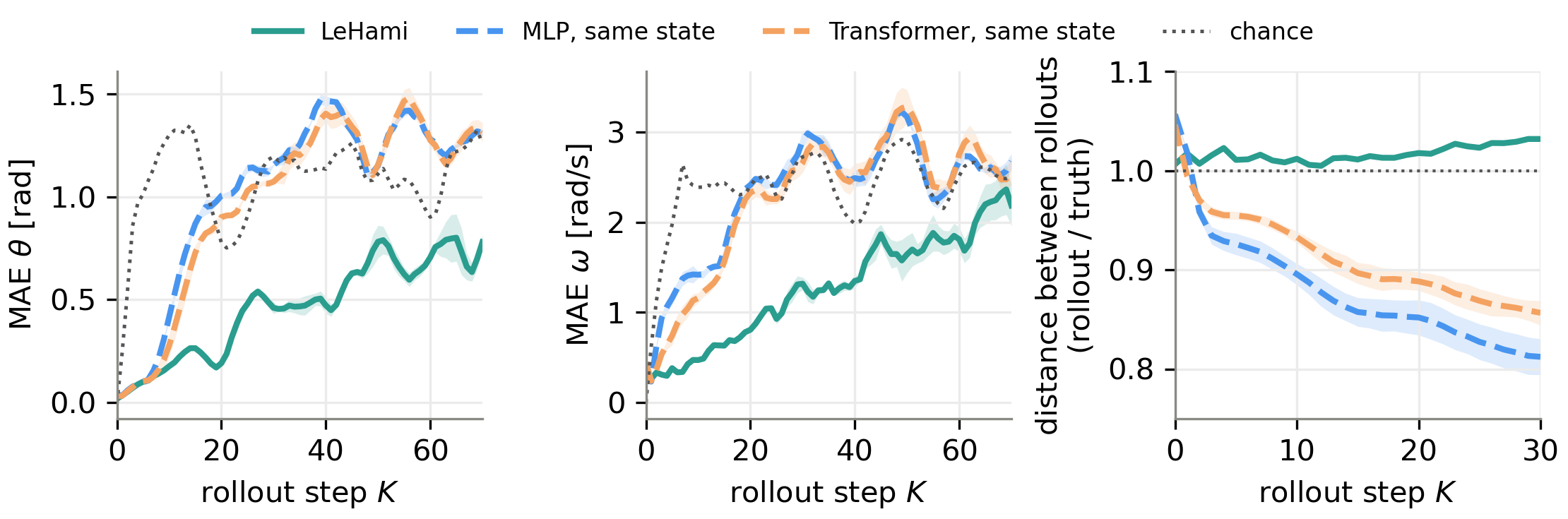}
    \caption{\textbf{Prediction performance during conservative pendulum rollouts.} Wrapped MAE of $\theta$ \textbf{(left)} and $\omega$ \textbf{(middle)} across the test rollout steps ($K$). $25$-NN probe fitted at $K=1$ and frozen; shaded: s.d.\ over three seeds. All three models share the encoder and the state $[\bq,\mathrm{SG}(\bq)]$; only the predictor differs. Dotted: chance (the better of repeating the start and predicting the mean). \textbf{(Right)} Median distance between rollouts started from different test frames, over the true distance between the same frames (1, dotted, means no contraction). The two free predictors contract steadily; \algo{} does not. LeWM with its own embedding $\bz$ and the other environments are in Appendix Figs.~\ref{fig:app-collapse-steps} and~\ref{fig:app-physics-steps}.}
    \label{fig:cons-physics}
\end{figure}

Table~\ref{tab:collapse} shows that both free predictors contract the rollout on every environment and on all four measures. Fig.~\ref{fig:cons-physics} (right) shows how this builds up on the conservative pendulum: the distance between rollouts falls steadily with the rollout step for the two free predictors, while \algo{}'s stays at $1$. \algo{}'s rollouts never merge, but they are not exact: past the $8$ trained steps they drift (values above $1$), and on two-body and the cart pole they squeeze some directions while stretching others, keeping the total volume; their distance and spread stay above $1$ there, so the rollouts move apart, not together. On the cart pole (tangent area $0.58$) the mass depends on the pole angle, so $\bp=\dot\bq$ is not the true momentum and the true motion does not keep volume in \algo{}'s coordinates, which a volume-preserving step cannot follow. On two-body (tangent area $0.94$, rank $0.83$) the learned energy does not match the physical one at $0.2$\,s per frame, and the rollouts drift along the orbit. LeWM on its own state, and all four measures against the step, are in Appendix Fig.~\ref{fig:app-collapse-steps}. A probe's squared error can be written as offset, amplitude and loss of correlation (Appendix~\ref{app:add-physics}). On the mass-spring, collapse (the amplitude term) is $79\%$ and $83\%$ of the free predictors' position error ($5\%$ for \algo{}; Table~\ref{tab:error-split}). On the other environments their error is mostly loss of correlation, which LeWM on its own state also shows without contracting. With the same $T+V$ and an explicit Euler step, the rollouts expand (distance $1.71$) and the angle error reaches chance by $K=30$ ($1.29$ against $0.46$; Table~\ref{tab:controls}): the gain needs the symplectic step.

\subsection{Evaluating Complete Phase-Space Physics Capture} \label{sec:physics}
 The state of a mechanical system is position and momentum, so a world model scored on position alone is tested on half of it. We probe both, position and velocity, at the anchor and along the rollout; \algo{}'s latent carries velocity by construction, which a per-frame baseline cannot.

On the conservative pendulum, at the anchor, the $K$NN probe reads
the angular velocity significantly better (see the anchor
row in Table~\ref{tab:cons-knnk1}). One image does not show motion, so a per-frame encoder, as in baseline, cannot
supply it. While \algo~also encodes per frame, it reads velocity from the Savitzky--Golay channel, which does not require any training. During rollout, Fig.~\ref{fig:cons-physics} compares \algo{} with the two predictors given the same state, so only the step differs. From $K\approx20$ on they are no better than chance, while \algo{} stays below it in angle up to $K=70$; the right panel shows their rollouts contract. LeWM on its own state does not contract, yet its error also grows: its predictions keep the true spread but lose track of the true state (Appendix~\ref{app:add-physics}). The learned $H_\phi$ makes that each rollout stays on its own level set
of $H_\phi$, as the true energy confines the physical system, while LeWM's rollouts cross
them (Appendix Figs.~\ref{fig:latent_space_cons_pendulum}--\ref{fig:physics_vs_\algo_cons_pendulum}).

This advantage does not come from a larger model or a larger training budget: \algo~has lower $\theta$ and $\omega$ error than LeWM at every matched training budget, over three seeds; its checkpoint at $0.75$\,PFLOP is on par with LeWM's at $6$\,PFLOP ($0.30$ against $0.33$\,rad) (Appendix Table~\ref{tab:exp096-isocompute}). The advantage also holds with a frozen pre-trained encoder: DINO-LeHami outperforms DINO-WM \citep{zhou2024dino} (Appendix Fig.~\ref{fig:dino-powered}).

\begin{table}[t]
\centering
\caption{\textbf{Physical prediction at the longest step, five environments.} MAE of the position and the velocity read from the rollout by the $25$-NN probe, mean $\pm$ s.d.\ over three seeds. Pendulums: $\theta$ [rad], $\omega$ [rad/s], $K=70$ and $K=30$, the step of Table~\ref{tab:collapse}. Mass-spring: $q$, $p$, $K=50$. Two-body: position and velocity of a body, $K=70$. Cart pole: $\phi$ [rad], $\dot\phi$ [rad/s], $K=70$. Bold: lowest of the three predictors given the same state. The two baselines use no model.}
\label{tab:physics-all}
\resizebox{\linewidth}{!}{%
\begin{tabular}{llccccc}
\toprule
environment & & \algo{} & MLP, matched & Transformer (LeWM's) & repeat the start & predict the mean \\
\midrule
\multirow{4}{*}{cons.\ pendulum} & position, $K=30$ & $\mathbf{0.460 \pm 0.028}$ & $1.154 \pm 0.002$ & $1.073 \pm 0.009$ & 1.213 & 1.195 \\
 & velocity, $K=30$ & $\mathbf{1.318 \pm 0.070}$ & $2.905 \pm 0.106$ & $2.658 \pm 0.110$ & 2.813 & 2.723 \\
 & position, $K=70$ & $\mathbf{0.779 \pm 0.034}$ & $1.310 \pm 0.015$ & $1.313 \pm 0.046$ & 1.259 & 1.324 \\
 & velocity, $K=70$ & $\mathbf{2.185 \pm 0.227}$ & $2.672 \pm 0.085$ & $2.566 \pm 0.144$ & 2.741 & 2.499 \\
\midrule
\multirow{4}{*}{act.\ pendulum} & position, $K=30$ & $\mathbf{0.596 \pm 0.027}$ & $0.811 \pm 0.020$ & $0.707 \pm 0.004$ & 1.108 & 1.019 \\
 & velocity, $K=30$ & $\mathbf{1.458 \pm 0.059}$ & $2.089 \pm 0.052$ & $1.858 \pm 0.032$ & 2.733 & 2.607 \\
 & position, $K=70$ & $\mathbf{0.908 \pm 0.067}$ & $1.073 \pm 0.027$ & $1.076 \pm 0.015$ & 1.186 & 1.168 \\
 & velocity, $K=70$ & $\mathbf{2.280 \pm 0.206}$ & $2.721 \pm 0.033$ & $2.801 \pm 0.086$ & 3.222 & 2.541 \\
\midrule
\multirow{2}{*}{mass-spring} & position & $\mathbf{0.074 \pm 0.016}$ & $0.223 \pm 0.005$ & $0.164 \pm 0.018$ & 0.422 & 0.336 \\
 & velocity & $\mathbf{0.086 \pm 0.009}$ & $0.254 \pm 0.004$ & $0.201 \pm 0.019$ & 0.436 & 0.347 \\
\midrule
\multirow{2}{*}{two-body} & position & $0.421 \pm 0.047$ & $\mathbf{0.148 \pm 0.073}$ & $0.183 \pm 0.037$ & 2.502 & 1.456 \\
 & velocity & $0.112 \pm 0.013$ & $\mathbf{0.056 \pm 0.034}$ & $0.062 \pm 0.019$ & 0.688 & 0.423 \\
\midrule
\multirow{2}{*}{cart pole} & position & $\mathbf{1.041 \pm 0.098}$ & $1.190 \pm 0.093$ & $1.104 \pm 0.081$ & 2.276 & 1.285 \\
 & velocity & $2.314 \pm 0.192$ & $\mathbf{2.140 \pm 0.101}$ & $2.248 \pm 0.299$ & 3.778 & 2.409 \\
\bottomrule
\end{tabular}}
\end{table}

\textbf{Five environments.} Table~\ref{tab:physics-all} extends this to all five environments. Given the same state, \algo{} is the most accurate of the three predictors in position on four of five, and beats both baselines in position everywhere. In velocity, \algo{} is the most accurate on the pendulums and the mass-spring; on two-body and the cart pole the free predictors are as accurate or better, and on the cart pole the probe reads $\dot\phi$ poorly even before any rollout. LeWM with a three-frame state, fully learned, is more accurate on four of five (Appendix~\ref{app:add-lewm}); combining it with a predictor that does not collapse is natural follow-up work. The gain depends on the horizon: at $K=10$ the free predictors are more accurate on the actuated pendulum (angle $0.20$ and $0.26$ against $0.32$\,rad, only on episodes that go over the top; Table~\ref{tab:episode-split}) and the cart pole. On two-body at the suite's $0.05$\,s per frame \algo{} fails ($K=70$ position $0.47$ against $0.05$ to $0.09$): its $8$ training steps cover $1$ to $6\%$ of an orbit, too little to learn the force (Appendix Table~\ref{tab:twobody-005}).

\textbf{When incorporating actions.} The actuated pendulum shows the same pattern at the anchor. Across three probe tiers (Appendix Table~\ref{tab:physics_prediction_anchor}), \algo's velocity is decoded well by the non-linear probes while LeWM's stays below the predict-the-mean floor; LeWM decodes $\theta$ more accurately, a gap that narrows with probe capacity. As on the conservative pendulum, a per-frame encoder carries no velocity, while \algo{} reads it from the SG channel (Appendix Fig.~\ref{fig:canonical_actuated_pendulum}). Along the rollout these advantages decay (Appendix Fig.~\ref{fig:actuated_pendulum_physics}): \algo{}'s in velocity by $K \approx 10$, once LeWM's predictor has seen enough frames to infer motion, and LeWM's in $\theta$ by $K \approx 40$; at $K=70$, \algo{}'s angle error is $0.91$\,rad against $0.98$ for LeWM (Appendix~\ref{app:add-lewm}).

\subsection{\algo~is a world model with time re-stepping and invertibility} \label{section:experiments_time}

\algo{} runs at any rate, and in either direction, because $H_\phi$ is a continuous-time vector field. We showcase this capability by holding the learned $H_\phi$ fixed and integrating it at $s\Delta t$, taking $1/s$ sub-steps per observed frame, and compare against a Runge--Kutta integration of the true ODE from the same initial condition. Refining below the observation rate leaves the predicted trajectory invariant
(Fig.~\ref{fig:dt-reversal} (left)); coarsening breaks it, as an explicit
symplectic integrator leaves its stability band. The same structure makes the step invertible given the action history: running it with $\Delta t \to -\Delta t$ and the recorded actions in reverse order undoes the rollout for any learned $H_\phi$, which we check in floating point as an implementation test (round-trip error below $10^{-4}\%$ after $30$ steps). What depends on the learned $H_\phi$ is whether the backward rollout lands on the true past. From a mid-episode frame of a held-out episode, we run the step backward with the actions that drove the earlier frames, and score it against the true earlier frames (Fig.~\ref{fig:dt-reversal}, right). On both pendulums the backward rollout recovers the earlier angles as well as the forward one predicts the later ones, and both stay well below chance up to $20$ steps, under actuation too. A free predictor has no backward step.

\begin{figure}[t]
    \centering
    \includegraphics[width=\linewidth]{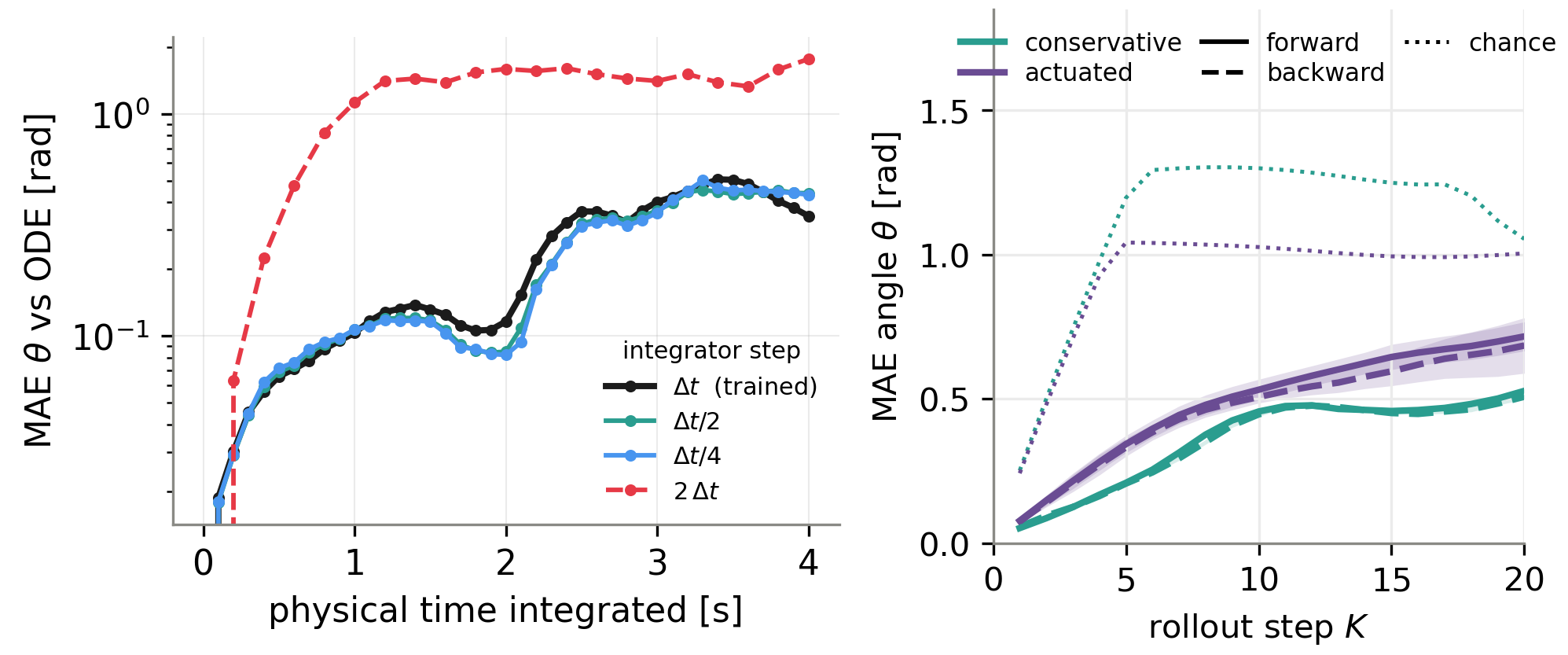}
    \caption{\textbf{The learned flow behaves like a continuous-time vector field,
    and runs backward to the true past.}
    \textbf{Left:}~MAE of $\theta$ integrated with step sizes never seen in
    training (to the ground-truth from the integration of the ODE), conservative pendulum. The model is trained only at $\Delta t = 0.1$\,s; halving and quartering the step (2 and 4 sub-steps per frame, respectively)
    traces the same trajectory, while doubling it breaks.
    \textbf{Right:}~MAE of $\theta$ after $K$ steps forward (solid) and backward with the recorded actions (dashed), from mid-episode frames of held-out episodes, conservative (teal) and actuated (purple) pendulums; MLP probe, three seeds. Dotted: chance (Appendix~\ref{app:add-backward}).}
    \label{fig:dt-reversal}
\end{figure}

\subsection{\algo{} can be used for planning} \label{sec:planning}

We also plan with \algo{} on the actuated pendulum, with CEM under a budget of $15$\,TFLOP per episode (Appendix~\ref{app:planning}). Its predictor is cheaper per step ($2.23$M against $3.31$M parameters), so within the budget it can replan more often, and it represents the sign of the velocity by construction, which a LeWM encoder with a tied temporal kernel cannot. It does not yet reach goals better: LeWM has the highest success on every clause, from one frame on the angle and from three frames with an untied kernel on the velocity and on both (Appendix Table~\ref{tab:swingup-best-clause}). Planning backward from the goal could help close this gap.

%% file: background.tex
\section{Background in Hamiltonian mechanics} \label{appendix:background_hamiltonian}

A physical Hamiltonian system with $n$ degrees of freedom evolves on a phase space $\cM = \R^{2n}$,
coordinatized by generalized positions $\bq \in \R^n$ and conjugate momenta $\bp \in \R^n$
\citep{arnold1989mathematical}. Its evolution is generated by the Hamiltonian, which is a scalar function $H :
\R^{2n} \to \R$ that represents the system's total energy, through Hamilton's equations
\begin{equation}
  \dot{\bq} = \frac{\partial H}{\partial \bp},
  \qquad
  \dot{\bp} = -\frac{\partial H}{\partial \bq}.
  \label{eq:hamilton_eqs}
\end{equation}
The two main consequences of this that we use in this work are the following. The flow $\phi_t$ of
Eq.~\ref{eq:hamilton_eqs} is \emph{symplectic}: it preserves $\Omega = \sum_i dq_i \wedge
dp_i$, and hence the phase-space volume due to Liouville's theorem \citep{arnold1989mathematical}. Thus, a phase-space density $\rho$ is transported without distortion, $\partial_t \rho = -\{\rho,
H\}$ with $\{\cdot,\cdot\}$ the Poisson bracket. 

\paragraph{Phase-space density.} Rather than a single initial condition one works with a
density $\rho(\bq,\bp,t)$ normalised to unit mass. Its evolution obeys the Liouville
equation $\partial_t \rho = -\{\rho, H\}$, where $\{\rho,H\} = \partial_\bq \rho \cdot
\partial_\bp H - \partial_\bp \rho \cdot \partial_\bq H$. Equivalently $d\rho/dt = 0$: the
density seen by an observer moving with the flow is constant, so the flow is an
incompressible deformation of phase space.

\paragraph{Symplectic form.} The two-form $\Omega = \sum_{i=1}^n dq_i \wedge dp_i$ is
closed and non-degenerate, and the Hamiltonian flow $\phi_t$ preserves it exactly. The
wedge product records oriented area in the $(q_i,p_i)$ plane and changes sign under
exchange of its arguments, which is the canonical asymmetry between position and momentum.
Preservation of $\Omega$ implies preservation of the volume form $\Omega^n/n!$
(Liouville's theorem): a cloud of initial conditions may stretch, rotate and shear, but its
volume is invariant. For $N$ particles the form extends additively,
$\Omega = \sum_{j=1}^N \sum_{i=1}^n dq_i^{(j)} \wedge dp_i^{(j)}$, and the statement holds
on the full $2nN$-dimensional space.

\subsection{Comparison with prior Hamiltonian models} \label{app:compare-prior}
Table~\ref{tab:compare-prior} sets \algo{} beside the models of Section~\ref{sec:related}. \algo{}
brings Hamiltonian learning into JEPA world models: it learns from video without a decoder, puts
the action in the potential, obtains $\bp$ as the time derivative of $\bq$, and is rolled out far
past its training horizon. The properties of the leapfrog step itself are classical and shared
with SRNN, HGN and HamJEPA.

\begin{table}[h]
\centering
\caption{\textbf{\algo{} beside prior Hamiltonian models:} HNN \citep{greydanus2019hamiltonian}, SRNN \citep{chen2019symplectic}, HGN \citep{toth2019hamiltonian} and HamJEPA \citep{alvarez2026beyond}. \emph{Rolled out in time}: the step is
applied to its own output, one step per observed frame. \emph{Volume kept at every step}: the
discrete step used at inference is exactly symplectic.}
\label{tab:compare-prior}
\small
\begin{tabular}{*{6}{>{\raggedright\arraybackslash}p{1.85cm}}}
\toprule
 & HNN & SRNN & HGN & HamJEPA & \algo{} (ours) \\
\midrule
learns from & states, or pixels & states & pixels & static images & video \\
decoder & for pixels & no & yes & no & no \\
actions & no & no & no & no & yes, in $V(\bq,\bu)$ \\
rolled out in time & yes (RK45) & yes (leapfrog) & yes (leapfrog) & no (two augmented views) & yes (leapfrog) \\
momentum $\bp$ & given, or encoded & given & inferred by the encoder & half of the latent & $\mathrm{SG}(\bq)$, time derivative of $\bq$ \\
volume kept at every step & no & yes & yes & within one prediction & yes, any horizon and actions \\
evaluated by & state error, energy & state error & pixel error, energy & classification & physical state from probes, collapse on the data \\
\bottomrule
\end{tabular}
\end{table}

%% file: appendix_theory_design.tex
\section{Theory} \label{appendix:theory}

\subsection{Symplecticity of the leapfrog step: formal statement and proof}
\label{app:volume}

Proposition~\ref{prop:volume} is stated in the main text in geometric terms. Here we give the
regularity hypotheses, render the three consequences in explicit form. Throughout, $\bz = (\bq,\bp) \in \R^{2d}$ and
\[
\mathbf{J} \;=\; \begin{pmatrix} \bzero & \bI_d \\ -\bI_d & \bzero \end{pmatrix}
\in \R^{2d \times 2d},
\qquad \mathbf{J}^{\!\top} = -\mathbf{J} = \mathbf{J}^{-1}.
\]

A $C^1$ map $\Psi : \R^{2d} \to \R^{2d}$ is \emph{symplectic} if
$D\Psi(\bz)^{\!\top} \mathbf{J}\, D\Psi(\bz) = \mathbf{J}$ for every $\bz$, where
$D\Psi(\bz) \in \R^{2d\times 2d}$ denotes the Jacobian of $\Psi$ at $\bz$,
i.e.,\ $[D\Psi(\bz)]_{ij} = \partial \Psi_i/\partial z_j$.

\begin{proposition}[Formal form of Proposition~\ref{prop:volume}]
\label{prop:volume-app}
Let $T_{\phi_1} \in C^2(\R^d)$ and $V_{\phi_2} \in C^2(\R^d \times \mathcal{U})$, and let
$\Psi_h$ be the leapfrog step of Eq.~\ref{eq:leapfrog} generated by
$H_\phi = T_{\phi_1}(\bp) + V_{\phi_2}(\bq,\bu)$ with the control $\bu$ held fixed within the
step. This means that the leapfrog step can be written as $\Psi_h = A_{h/2} \circ B_h \circ A_{h/2}$
with
\begin{equation}
\label{eq:leapfrog-factors}
A_{h/2}(\bq,\bp) = \bigl(\bq,\; \bp - \tfrac{h}{2}\,\nabla_{\bq} V_{\phi_2}(\bq,\bu)\bigr),
\qquad
B_{h}(\bq,\bp) = \bigl(\bq + h\,\nabla_{\bp} T_{\phi_1}(\bp),\; \bp\bigr).
\end{equation}

Then $\Psi_h$ is a $C^1$ diffeomorphism of $\R^{2d}$ that is exactly symplectic, which means that
\begin{equation}
\label{eq:symplectic-jacobian}
D\Psi_h(\bz)^{\!\top}\, \mathbf{J}\, D\Psi_h(\bz) \;=\; \mathbf{J}
\qquad \text{for every } \bz \in \R^{2d},
\end{equation}
for every step size $h > 0$, every control $\bu \in \mathcal{U}$ and every parameter
$\phi = (\phi_1,\phi_2)$. Consequently
$\det D\Psi_h \equiv 1$, Lebesgue measure volume on $\R^{2d}$ is $\Psi_h$-invariant,
and for every horizon $K \in \mathbb{N}$ the three below conditions are met:
\begin{enumerate}
\item[(i)] \emph{(no uniform contraction)} the singular values of $D(\Psi_h^K)(\bz)$ with $\Psi_h^K := \Psi_h \circ \cdots \circ \Psi_h$ satisfy
$\prod_{i=1}^{2d} \sigma_i = 1$, hence $\sigma_{\max} \ge 1 \ge \sigma_{\min}$ at every
$\bz \in \R^{2d}$;

\item[(ii)] \emph{(no asymptotically stable fixed point)} $\Psi_h$ admits no
asymptotically stable fixed point, where ``asymptotically stable'' means: a fixed point
$\bz^\ast$ that is Lyapunov stable and attracts every point of some open neighbourhood
of itself;

\item[(iii)] \emph{(no merging)} $\Psi_h^K$ is a composition of diffeomorphisms, hence itself
a diffeomorphism of $\R^{2d}$, in particular injective. Consequently, if $\bz_1 \neq \bz_2$
then $\Psi_h^K(\bz_1) \neq \Psi_h^K(\bz_2)$ for every $K \in \N$: two latents that differ at the
anchor remain distinct at every horizon.
\end{enumerate}
\end{proposition}

\begin{proof}
We start by showing that each of the  half-steps are symplectic shears. For this, we fix $\bu \in \mathcal{U}$ and from Eq.~\ref{eq:leapfrog-factors} we compute their Jacobians as
\[
DA_{h/2}(\bq,\bp) = \begin{pmatrix} \bI_d & \bzero \\ -\tfrac{h}{2}\,\nabla^2_{\bq\bq}V_{\phi_2}(\bq,\bu) & \bI_d \end{pmatrix},
\qquad
DB_{h}(\bq,\bp) = \begin{pmatrix} \bI_d & h\,\nabla^2_{\bp\bp}T_{\phi_1}(\bp) \\ \bzero & \bI_d \end{pmatrix},
\]
which are well defined and continuous since $T_{\phi_1}, V_{\phi_2} \in C^2$. Both are unipotent
block-triangular, so $\det DA_{h/2} = \det DB_h = 1$.
By Clairaut's theorem, the $C^2$ regularity implies that the Hessians
$\nabla^2_{\bq\bq}V_{\phi_2}$ and $\nabla^2_{\bp\bp}T_{\phi_1}$ are symmetric.
This symmetry gives
$DA_{h/2}^{\!\top}\mathbf{J}\,DA_{h/2} = \mathbf{J}$ and
$DB_h^{\!\top}\mathbf{J}\,DB_h = \mathbf{J}$. Thus, both maps are symplectic.
Separately, each map has an explicit global $C^1$ inverse:
$A_{h/2}^{-1} = A_{-h/2}$ and $B_h^{-1} = B_{-h}$.
Hence $A_{h/2}$ and $B_h$ are $C^1$ diffeomorphisms of $\R^{2d}$.
Together, these facts show that both maps are exactly symplectic $C^1$ diffeomorphisms
with unit Jacobian determinant, for every $h>0$, $\bu\in\mathcal{U}$ and $\phi$.

Now we show that these properties survive composition: given $\Psi,\Phi$ $C^1$ diffeomorphisms of $\R^{2d}$, exactly symplectic, thus $\Psi\circ\Phi$ is. The chain rule shows that  $D(\Psi\circ\Phi)^{\!\top}\mathbf{J}\,D(\Psi\circ\Phi)
= D\Phi^{\!\top}(D\Psi^{\!\top}\mathbf{J}\,D\Psi)D\Phi = D\Phi^{\!\top}\mathbf{J}\,D\Phi = \mathbf{J}$,
and $\det D(\Psi\circ\Phi) = \det D\Psi \cdot \det D\Phi = 1$. Thus, a composition of diffeomorphisms
is a diffeomorphism.

With these two properties we can affirm that $\Psi_h = A_{h/2}\circ B_h \circ A_{h/2}$ is thus $C^1$ diffeomorphism of $\R^{2d}$ and exactly symplectic. Likewise, every iterate $\Psi_h^K := \Psi_h\circ\cdots\circ\Psi_h$ ($K$ times) is a composition of $C^1$ diffeomorphisms of $\R^{2d}$, exactly symplectic. Thus, $\Psi_h^K$ also is, with
$\det D(\Psi_h^K) \equiv 1$, for every $h>0$, $K  \in \mathbb{N}$, $\bu\in\mathcal{U}$, $\phi$. With this, the proof for the three conditions are straightforward.

\textbf{Proof of (i).}
$\prod_{i=1}^{2d} \sigma_i(M) = |\det M|$ for any square matrix $M$. By taking $M = D(\Psi_h^k)(\bz)$ and
using $\det D(\Psi_h^k)\equiv 1$, it follows that $\prod_{i=1}^{2d}\sigma_i = 1$, which means that
$\sigma_{\max}\ge 1 \ge \sigma_{\min}$ at every
$\bz \in \R^{2d}$.

\textbf{Proof of (ii).}
Suppose $\bz^\ast$ is an asymptotically stable fixed point of $\Psi_h$, attracting
an open neighbourhood $U$. Choose a closed ball $C=\overline{B_r(\bz^\ast)}\subset U$
with $r>0$. Lyapunov stability and attraction imply uniform convergence on $C$:
for any $\varepsilon>0$, stability gives an open neighbourhood $W$ of $\bz^\ast$ whose
forward iterates stay in $B_\varepsilon(\bz^\ast)$. Each point of $C$ enters $W$
after finitely many steps; by continuity, a neighbourhood of that point enters $W$
at the same step. Compactness gives a finite cover of $C$ by such neighbourhoods,
so taking the largest of their entry times gives $K_\varepsilon$ such that
\[
\Psi_h^K(C)\subseteq B_\varepsilon(\bz^\ast)
\qquad\text{for all }K\ge K_\varepsilon.
\]
Volume preservation then gives
\[
0<\mathrm{vol}(C)=\mathrm{vol}(\Psi_h^K(C))
\le\mathrm{vol}(B_\varepsilon(\bz^\ast))=c_{2d}\varepsilon^{2d},
\]
where $c_{2d}$ is the volume of the unit ball in $\R^{2d}$.
Letting $\varepsilon\to0$ yields a contradiction. Hence no such fixed point exists.

\textbf{Proof of (iii).}
As shown, $\Psi_h^k$ is a diffeomorphism of $\R^{2d}$, which means bijective, and hence injective. As a consequence,
$\bz_1 \neq \bz_2 \implies \Psi_h^k(\bz_1) \neq \Psi_h^k(\bz_2)$ for every $k\in\N$.
\end{proof}

\subsection{Scope and limitations of the guarantee}
\label{app:remarks}

\begin{remark}[Shape, not size]
\label{rem:volume-limits}
A cloud that cannot be squeezed can still be sheared and stretched arbitrarily, so
symplecticity places no bound on the growth of rollout error and does not make the rollout
accurate. What it removes is the degenerate failure in its three forms: the rollout that
shrinks in every direction at once, the one that settles onto an attractor and stops moving,
and the one that merges states which started apart.

An unconstrained predictor $F_\theta$ is under no such constraint, and the one-step objective
pushes it the other way. Under any observation or latent noise, the minimiser of
$\E\lVert F_\theta(\bz_t) - \bz_{t+1}\rVert^2$ is the conditional mean
$\E[\bz_{t+1} \mid \bz_t]$, which averages away whatever the anchor leaves undetermined: it
contracts volume, and it sends latents that differ only in the averaged-out directions to a
common image. Volume loss and merging are therefore the trained optimum for such a map and iterating it drives the rollout onto exactly the attracting set
the proposition forbids for $\Psi_h$. The proposition says the Hamiltonian rollout cannot
enter that regime at any step size $h$, any parameter $\phi$ or any horizon.
\end{remark}

\begin{remark}[Regularization is not made redundant]
\label{rem:sigreg}
Proposition~\ref{prop:volume} states that the flow carries phase-space volume
exactly and its does not state nothing about the shape of the distribution that volume belongs to. Due to this, regularization techniques are not made redundant. In this work, we leverage the SIGReg \citep{balestriero2025lejepa} as regularizer, as it is the one used in LeWM and has shown state-of-the-art performance \citep{maes2026leworldmodel}.
For the isotropic Gaussian $\Normal(\bzero, \bI)$ installed by SIGReg at the encoder to be
invariant under the flow, the radius $\lVert(\bq,\bp)\rVert^2$ would itself have to be a
conserved quantity of the generator, that is $\{\lVert(\bq,\bp)\rVert^2, H_\phi\} = 0$. This is
a genuine restriction: it is a functional constraint on $H_\phi$, not a consequence of
symplecticity, and a general learned Hamiltonian, which is action-conditioned, and parameterised by
free-form networks as we do here, has no reason to satisfy it.

The two mechanisms are therefore complementary rather than substitutes. SIGReg shapes the
encoder distribution, during training and in expectation over batches, which is where a
collapse of the representation would originate. Symplecticity carries the volume of
whatever distribution it is handed, at every inference step and for every input, which is
where an iterated map would otherwise contract or explode.
\end{remark}

\begin{remark}[Geometric prior, no physical modeling]
\label{rem:hypothesis-class}
We are not modeling the physics of the observed system. $H_\phi$ is never fit to a known
Hamiltonian and is not claimed to recover one: the true equations of motion enter only as
evaluation targets for frozen read-out probes, never as a training signal. This is the point
of departure from Hamiltonian neural networks \citep{greydanus2019hamiltonian} and SRNN
\citep{chen2019symplectic}, where the learned $H_\phi$ is a surrogate for the physical Hamiltonian
and recovering it is the goal.

The prior models the latent representation space. Writing $\mathcal{F}$ for the
maps a JEPA predictor may realize, we restrict it to the symplectomorphisms reachable as one
leapfrog step of a learned separable generator $H_\phi(\bq_t, \bp_t, \bu_t)$; the content of
that restriction is geometric (Proposition~\ref{prop:volume}), and it governs the dynamics, or evolution, latent space. The canonical split is structural in
the same sense: the encoder emits $\mathbf{s}_t = (\bq_t, \bp_t)$ with $\bq_t, \bp_t \in \R^d$
and the flow operates on that split by construction, but $\bq_t$ is not assumed to encode
position, nor $\bp_t$ momentum, in any observation-space sense.
\end{remark}

\subsection{Why Hamiltonian and not Lagrangian} \label{hamiltonian_no_lagrangian}

The two formalisms describe the same
mechanics in different coordinates, so the choice is made on what each buys a predictor,
and two things separate them here. First, the guarantee we rely on is a statement about phase
space: symplecticity, and the volume preservation it implies
(Proposition~\ref{prop:volume}), are properties of a flow on $(\bq,\bp)$ and have no equally
cheap counterpart on the tangent bundle $(\bq,\dot\bq)$ where a Lagrangian lives. Second, the
Euler--Lagrange equations are second order in $\bq$ while Hamilton's equations are first order
on phase space and a separable $H_\phi$ makes the leapfrog step explicit, which is what keeps the
map exactly symplectic at any $h$ rather than approximately so. Lagrangian neural networks \citep{cranmer2020lagrangian} learn the Lagrangian directly; they keep energy but have no counterpart of the volume guarantee for the rollout map.

\subsection{Design choices in the flow and the encoder}
\label{app:design}

\textbf{The shape of the Hamiltonian.} In principle, there are multiple ways to structure the Hamiltonian. The simplest Hamiltonian is a separable one, where the kinetic term $T_\phi$ depends only on the momentum-like coordinates and the potential term $V_\phi$ depends only on the position-like coordinates and the action. This is a common choice in classical mechanics, and it allows for efficient integration using the leapfrog method, as discussed before. In this work, we focus on separable Hamiltonians, but more general forms can be considered as well.

We consider an action-conditioned Hamiltonian parametrized as
$H_{\phi}$, where the action $\mathbf{u}_t$ represents an external control input that influences the evolution of
the latent state. This can be interpreted as inducing an action-dependent vector field,
where the action modulates the dynamics through the Hamiltonian. We condition the potential on the action:
\begin{equation}
\label{eq:H-action-potential}
H_\phi(\mathbf{q},\mathbf{p},\mathbf{u})
\;=\; T_{\phi_1}(\mathbf{p}) \;+\; V_{\phi_2}(\mathbf{q},\mathbf{u}),
\end{equation}
so that the generalized force is $\mathbf{F}(\mathbf{q},\mathbf{u}) =
-\partial V_{\phi_2}/\partial \mathbf{q}$ and the dynamics remain in strict
Hamiltonian form for every fixed action.

Let the action be piecewise constant,
taking value $\bu_t$ on step $t$, and write $\bz_t = (\bq_t,\bp_t)$. Decompose the change in
the Hamiltonian across a step into the part due to the action changing and the part due to
the flow,
\begin{equation}
\label{eq:energy-budget}
H_\phi(\bz_{t+1},\bu_{t+1}) - H_\phi(\bz_{t},\bu_{t})
\;=\; \underbrace{H_\phi(\bz_{t+1},\bu_{t+1}) - H_\phi(\bz_{t+1},\bu_{t})}_{W_t}
\;+\; \underbrace{H_\phi(\bz_{t+1},\bu_{t}) - H_\phi(\bz_{t},\bu_{t})}_{r_t},
\end{equation}

where $W_t$ is the work performed by the change of action, which is a physical transfer of energy that any actuated model must reproduce, and $r_t$ is the
integrator residual. Summing, $H_\phi(\bz_K,\bu_K) - H_\phi(\bz_0,\bu_0) = \sum_t W_t +
\sum_t r_t$, so the model's energy tracks the accumulated work up to $\sum_t r_t$.

Placing the prior on the latent map rather than on the observational world (see Remark~\ref{rem:hypothesis-class}) does not make it vacuous.
If the observed dynamics contract phase-space volume, no dimensionally-matched re-coordinatization admits a symplectic latent map that reproduces them.
While an over-parameterized symplectic space could theoretically embed volume-contracting data by expanding orthogonal auxiliary dimensions to satisfy Liouville's theorem, doing so forces the model to learn unphysical, exponentially diverging latent variables.
Without such artificial dimensional expansion, dissipation is the primary property of the observed system that a purely symplectic representation cannot absorb.

Because both systems studied in this work preserve volume, we leave dissipative modeling to future work.
This extension takes the same structured form as the rest of the model: introducing a learned matrix $D_{\phi_3} = L_{\phi_3}L_{\phi_3}^{\top} \succeq 0$ into the momentum equation, which is positive semi-definite by construction so that it can only dissipate, smoothly recovering the symplectic model at $L_{\phi_3} = \mathbf{0}$.

\textbf{The encoder and the SG stencil.} Our proposal to model the dynamics in the latent representation space with a Hamiltonian flow is independent of the encoder architecture. However, in the self-supervised setting, the embeddings are being learned, together with its Hamiltonian evolution. For this reason, the encoder has to be sufficiently expressive to encode canonical coordinates. Encoding $(\bq, \bp)$ is not just separating the embeddings as $(\bz_1, \bz_2) = \bz$, in which we just split the embedding into two.

We enforce $\bq$ and $\bp$ to be canonical by enforcing that the momentum-like coordinates are
the time derivative of the position-like coordinates ($p \equiv \dot{q}$ with unit mass). This is a common assumption in classical mechanics, where the momentum is
defined as the derivative of the Lagrangian with respect to the velocity. We can enforce this by making the encoder produce the momentum-like coordinates as
the differentiation of the position-like coordinates. Specifically, we can do this with numerical methods. In this work, we use a Savitzky--Golay filter \citep{gorry1990general} to first approximate the position-like coordinates with a polynomial
and then compute the momentum-like coordinates as the derivative.

Concretely, $\bp_t$ is a fixed linear functional of the encoder trajectory around
frame $t$. We fit a quadratic in time to the encoded positions in a window by least squares whose coefficients are fixed in closed form by the window offsets and polynomial order alone. The least-squares fit sees the encoder values only as its right-hand side, never in its design matrix. Then, its derivative is evaluated at $t$, which for a fixed window is a constant stencil. Three properties of this choice matter. The estimate
is at frame $t$ rather than half a step behind it, unlike a backward difference, so $\bq_t$
and $\bp_t$ refer to the same instant and the pair is canonical in the only sense we can
enforce. The stencil is linear in the encoder output with constant coefficients (i.e., no training involved), so it adds
no parameters and cannot itself collapse: whatever the momentum block carries, it carries
because the encoder trajectory moves. And a least-squares fit over several frames attenuates
the high-frequency amplification that differentiation applies to a noisy signal, which a
two-point difference does not, at the price of a truncation bias.

At the rollout anchor, in training and at inference alike, we use a four-frame causal
window, since a centred window would read the frames the model is about to predict; the
flow then evolves with no further stencil. The momentum targets used in
training are computed with a five-frame centred window, $t-2,\dots,t+2$, on the encoded $\bq$'s.
The causality constraint does not apply to the targets: they are read off
the encoder trajectory only to be compared with the rollout, never fed into it, so frames on
both sides of $t$ are available and using them leaks nothing. Given that choice, the centred window is
the better estimator on both counts of the bias--variance trade-off above. Its weights,
$(-2,-1,0,1,2)/10$, are antisymmetric in time, which cancels every even-order error term, so
the truncation bias starts at $\tfrac{3.4}{6}\,\dddot{\bq}\,\Delta t^{2}$ with no
$\Delta t^{3}$ term, against $\tfrac{4.7}{6}\,\dddot{\bq}\,\Delta t^{2}$ plus a
$\Delta t^{3}$ term for the four-frame causal window. It also amplifies independent
per-frame encoder noise $24.5\times$ less (a variance gain of $\sum_m w_m^2 = 0.10$ against
$2.45$). Since the targets enter the loss at every rollout step, a noisier target would inject
that variance into every gradient. The antisymmetry has one further consequence: reversing
time maps $\bp_t \mapsto -\bp_t$ exactly, the transformation law of a momentum under time
reversal, which a one-sided window does not satisfy.

\textbf{Why there is no mass.} In classical mechanics the momentum is $\bp = \mathbf{M}\dot{\bq}$
and the kinetic energy is $\tfrac12\bp^\top\mathbf{M}^{-1}\bp$, so a physical model must know the
mass matrix $\mathbf{M}$. \algo{} learns no mass, and the reason is the stencil: defining
$\bp \equiv \dot{\bq}$ already fixes $\mathbf{M}=\mathbf{I}$ in the learned coordinates. To see
it, combine the definition with Hamilton's first equation,
$\dot{\bq} = \nabla_\bp H_\phi$. On the data, a flow that reproduces the encoder trajectory
must satisfy
\begin{equation}
    \nabla_\bp H_\phi(\bq,\bp,\bu) \;=\; \mathbf{S}^{-1}\nabla T_\phi\!\left(\mathbf{S}^{-1}\bp\right) \;=\; \bp ,
\end{equation}
and integrating in $\bp$ gives
\begin{equation}
    T_\phi\!\left(\mathbf{S}^{-1}\bp\right) \;=\; \tfrac{1}{2}\|\bp\|^2 + C_\phi ,
\end{equation}
the kinetic energy of unit mass, whatever $\mathbf{S}$ is. A learned $\mathbf{M}$
would therefore be redundant rather than missing. For a system of constant mass the encoder can absorb $\mathbf{M}$ through
the linear re-coordinatization $\bq \mapsto \mathbf{M}^{1/2}\bq$, under which
$\tfrac12\dot{\bq}^\top\mathbf{M}\dot{\bq}$ becomes $\tfrac12\|\dot{\bq}\|^2$; this is a
case of the latent-map freedom of Remark~\ref{rem:hypothesis-class}. For one degree of freedom this is the stretch $q' = \sqrt{m}\,q$ (for the pendulum, $q' = \sqrt{m}\,L\,\theta$), whose momentum $p' = \dot q'$ carries the same kinetic energy with unit mass. The stretch may differ between episodes: on the actuated pendulum the encoder sees the mass in the frame (the bob is drawn larger), so one model learns from episodes with different masses, and the mass stays readable from the rollout (Table~\ref{tab:mass-r2}). What it needs is a mass that is constant within an episode. A configuration-dependent $\mathbf{M}(\bq)$ cannot be absorbed this way; the cart pole (Appendix~\ref{app:add-envs}) is such a system. This is also why the scale $\mathbf{S}$ below
stays close to the identity. Two caveats. First, the argument constrains $T_\phi$ only on the
support of the data and only to the accuracy of the stencil and of the leapfrog position step
(which uses $\nabla T_\phi$ at the half-step momentum), so we keep $T_\phi$ as a 3-layer MLP,
following Hamiltonian neural networks \citep{greydanus2019hamiltonian}. The quadratic is
what that MLP must approximate at convergence, not a constraint we impose. Second, the
re-coordinatization exists only when $\mathbf{M}$ is constant. When the mass matrix depends on
configuration, $\mathbf{M}(\bq)$ as in a double pendulum, no global linear change of
coordinates removes it, $\bp=\dot{\bq}$ is a velocity rather than a canonical momentum, and a
separable $T_\phi(\bp)$ is mis-specified. Lifting this requires learning $\bp$ jointly
instead of reading it off with a fixed stencil.

\textbf{Adaptation to the self-supervised setting.} In the supervised setting,
$(\bq, \bp)$ are the canonical coordinates of the physical system. Here, in the self-supervised setting, these coordinate-like embeddings are learned jointly with their evolution, so nothing
fixes the relative scale of $\bq$ and $\bp$. Rescaling the momentum,
$\bp \mapsto \mathbf{S}\bp$, together with $T_\phi(\cdot) \mapsto T_\phi(\mathbf{S}^{-1}\,\cdot)$,
leaves the trajectory of $\bq$ unchanged, so this scale is a gauge freedom that the
objective cannot determine. This scaling is learned:
$\mathbf{S} = \operatorname{diag}(e^{\boldsymbol{\sigma}})$ with $\boldsymbol{\sigma}$ learned
(Eq.~\ref{eq:hamiltonian}) and the
leapfrog of $H_\phi$ in Eq.~\ref{eq:leapfrog}  remains exactly symplectic; $\mathbf{S}$ enters only the position
step, through $\nabla_\bp H_\phi = \mathbf{S}^{-1}\nabla T_\phi(\mathbf{S}^{-1}\bp)$.

$\boldsymbol{\sigma}$ is initialised at $\mathbf{0}$, so $\mathbf{S} = \mathbf{I}$, and is
learned jointly with all other parameters (encoder, projector, $T_\phi$ and $V_\phi$) by
back-propagating the multi-step rollout loss through the leapfrog updates of
Eq.~\ref{eq:leapfrog}. Since $\mathbf{S}$ enters only through $T_\phi$, its gradient
arrives through the position steps, via $\nabla_\bp H_\phi = \mathbf{S}^{-1}\nabla
T_\phi(\mathbf{S}^{-1}\bp)$, and hence through every later prediction of the rollout.
In practice $\mathbf{S}$ remains close to the identity (diagonal entries between 0.95 and
1.09 across our trained models), since a per-dimension rescaling of $\bp$ can equally be
absorbed by the input weights of the MLP $T_\phi$.
Because $\mathbf{S}$ enters through $H_\phi$ rather than being inserted into the standard leapfrog of $H_\phi$ (Eq.~\ref{eq:leapfrog} ), $\mathbf{S}$ appears only
in the position step because $\nabla_\bp H_\phi = \mathbf{S}^{-1}\nabla T_\phi(\mathbf{S}^{-1}\bp)$
while $\nabla_\bq H_\phi = \nabla_\bq V_\phi$. This makes the map to remain exactly symplectic.

\textbf{Regularization.} As discussed in the main text, in principle, the regularization can be applied in different ways over the embeddings produced by \algo: $(\bq, \bp)$. Which is best is an empirical question our experimental results discuss. We apply step-wise $\mathrm{SIGReg}$ implemented by LeWM \citep{maes2026leworldmodel}; they differ only
in which embedding block contains. We show each of them below and they are the $\mathcal{S}$ in the training loss (Eq. \ref{eq:total}).

We index the frames of a training clip by $0,\dots,T-1$, as in Eq.~\ref{eq:total}. The anchor sits at
frame $A$, so its causal window is frames $0,\dots,A$ ($A=3$, four frames). The rollout then
predicts frames $A+1,\dots,A+K$ ($K=8$). Each target $\bp_{A+j}$ is built from the centred
window of half-width $h=2$, so the last target reads frames up to $A+K+h$, and a clip has
$T = A+K+h+1 = 14$ encoded frames. SIGReg is applied at every frame where the centred momentum
is defined,
\[
\mathcal{T} = \{h,\dots,T-1-h\} = \{h,\dots,A+K\}, \qquad |\mathcal{T}| = 10,
\]
which covers the anchor and all $K$ targets, plus the frames $h,\dots,A-1$ inside the anchor's
causal window. For the temporal encoders of window $W>1$ these indices are in latent time. The
encoder consumes $W-1$ lead-in pixel frames, so a clip holds $T+W-1$ pixel frames. One way is to apply SIGReg over the joint embedding. Given each $\bq$ and $\bp$ of size $d$ and a batch of $B$ embeddings, $\mathbf{Z}_t$ stacks the concatenated phase-space
embedding, step-wise over the full interior:
\[
\mathbf{Z}_t = \bigl[(q_t^{(i)}, p_t^{(i)})\bigr]_{i=1}^{B} \in \mathbb{R}^{B \times 2d},
\qquad
\mathcal{S} = \frac{1}{|\mathcal{T}|}\sum_{t \in \mathcal{T}} \mathrm{SIGReg}(\mathbf{Z}_t).
\]

Another way is to split the application. Same $\mathcal{T}$, but $\bq$ and $\bp$ are never concatenated and step-wise Gaussianity is tested on each block separately and summed:
\[
\mathbf{Q}_t = \bigl[q_t^{(i)}\bigr]_{i=1}^{B},\ \ \mathbf{P}_t = \bigl[p_t^{(i)}\bigr]_{i=1}^{B} \in \mathbb{R}^{B \times d},
\qquad
\mathcal{S} = \frac{1}{|\mathcal{T}|}\sum_{t \in \mathcal{T}} \Bigl[\mathrm{SIGReg}(\mathbf{Q}_t) + \mathrm{SIGReg}(\mathbf{P}_t)\Bigr].
\]

We apply SIGReg across all the time steps in the batch because, as discussed in Section \ref{sec:ham_jepa} in the main text and in Remark \ref{rem:sigreg}, the preservation of volume, which we obtain through the symplectic structure, is not redundant to what SIGReg enforces.

%% file: appendix_experiments.tex
\section{Description of the experimental set up} \label{appendix:experimenta_setup}

\subsection{Methodology} \label{appendix:methodology}

\textbf{What is held fixed across the models.} \algo~is implemented on top of the \texttt{stable-worldmodel} framework
\citep{maes_lld2026swm} and its \texttt{stable-pretraining} training harness
\citep{balestriero2025stable}, the same stack LeWorldModel
\citep{maes2026leworldmodel} is built on. Encoder, projector, action encoder,
optimiser and data pipeline are therefore shared with the baseline, and the
only component that differs between the two models is the predictor and SG stencil that \algo{} implements. Both models
of every comparison see identical data, identical clip construction and an identical
multi-step objective, the same number of epochs and the same optimizer and learning-rate schedule, so a difference in rollout behavior is not attributable to the
training scheme.

\textbf{The data split and checkpoint selection.} We split the datasets into training, validation and test. Checkpoints are selected on the
validation split. Both models use the same total latent
dimension (64), split by the \algo{} into $\bq$ and $\bp$ of half that width each (32) and used
whole by the baseline, which has no momentum-like block. We assess performance as the predictor is unrolled $K=8$ self-fed steps from an anchor. For \algo~we construct the anchor from the causal Savitzky--Golay stencil (Appendix~\ref{app:design}), so it sits at latent index $A = 3$, with full backpropagation through time, the momentum term in the loss (Eq.~\ref{eq:total}) carries $\lambda_p = 1.0$ and the distributional term
$\lambda_S = 0.25$ is common among the models, and $\lVert\cdot\rVert^2_2$ denotes the per-dimension mean rather than
the sum, which differs from the sum by a constant $1/d$ absorbed into those weights. LeWM's predictor is a causal, action-conditioned transformer over
the three most recent latents while in \algo~it is the Hamiltonian flow that evolves one state at the time. Regarding the SIGReg regularizer, while it is applied on $\bz$ for the baseline, in \algo~we tried out different places. If not said otherwise, SIGReg is applied on $\bq$ only. Refer to the discussion on this for the actuated pendulum in Appendix \ref{appendix:results_act_pendulum}.

\subsection{Implementation of \algo} \label{appendix:architecture}

\textbf{The Hamiltonian flow.} It a separable Hamiltonian $H_\phi(\bq,\bp,\bu) = T_{\phi_1}(\bp) + V_{\phi_2}(\bq,\bu)$ with two independent three-layer
$\tanh$ MLPs of width 1024, integrated by one exactly symplectic kick--drift--kick leapfrog
step per frame. For \textsc{LeWM} the projector output is a single $z_t \in \R^{64}$ and the predictor is the baseline's causal autoregressive transformer
with a sliding context of three latents. For \algo, the Hamiltonian autograd path is pinned
to fp32.

\textbf{The encoder.} We implemented both a tiny ViT \citep{wu2022tinyvit} and a temporal ConvNet \citep{tran2018closer}. For this temporal encoder, when there is no time window (W=1) it means that the encoder takes one pixel frame $\bx_t$ and produces its respective embedding $\bq_t$. Then, with these embeddings $\{\bq_{t-3}, \bq_{t-2}, \bq_{t-1}, \bq_{t}$ corresponding to multiple time steps, the $\bp_t$ is estimated. We also consider the temporal ConvNet with a temporal window in which pixel frames corresponding to multiple time steps are encoded to produce the embedding corresponding to a single time step. The general form of this architecture embedded into the \algo~framework is shown in Fig. \ref{diagram:convet}. The encoder is a \emph{(2+1)D} network, meaning the spatiotemporal convolution is factored into two spatial dimensions and one temporal one instead of a joint 3D kernel: a per-frame convolutional stem, one
temporal mix, then the unchanged spatial trunk.

 \input{diagram_convnet}

This encoder implementation allows to train across recording rates.
The step size is an argument to the Hamiltonian flow rather than a property of its
weights, so one learned $H_\phi$ can in principle be trained on episodes recorded at
different rates. We describe the mechanism here for completeness, but no experiment in
this paper uses it: every model is trained and evaluated at a single rate, and the
conservative dataset is restricted to its $\Delta t = 0.1\,$s slice for exactly this reason
(Table~\ref{tab:datasets}). We leave the mixed-rate comparison to future work.

A dataset that mixes recording rates stores each
episode's frame spacing in a \texttt{dt\_frame} column; write $\Delta t_{\text{ep}}$
for an episode's value and $\Delta t_{\text{ref}}$ for a single rate fixed once at
config time (\texttt{model.dt\_ref}). The leapfrog step carries a $\Delta t^2$
factor, so a $V_\phi(\bq)$ fitted at one rate is wrong by $(\Delta t/\Delta t')^2$ at
another: training directly on mixed-rate episodes would hand $T_\phi,V_\phi$ a
target that shifts with which episode a clip happens to come from. Dividing by
$\Delta t_{\text{ref}}$ removes this factor, so one learned $H_\phi$ stays valid
across every rate. We therefore integrate not in $\Delta t_{\text{ep}}$ but in the
dimensionless ratio $h = \Delta t_{\text{ep}} \,/\, \Delta t_{\text{ref}}$

Within an episode there is a single frame spacing and advancing one frame is one step. As \algo~parametrizes the time step inside the Hamiltonian flow, allowing to input, in principle, any arbitrary time step, episodes recorded at different rates could be trained on together. To do so, a recording rate is set as the reference, $\Delta t_{\text{ref}}$, and let the
integrator step be that episode's spacing measured in units of it. An episode recorded at the reference rate is integrated with $h = 1$, one recorded at twice
the spacing with $h = 2$, and so on. The learned $H_\phi$ is then the same object for every
episode, expressed in reference time, and all the rate information the model needs travels
with the clip as a single scalar. This is available to a Hamiltonian predictor precisely
because its step size is an argument; an unconstrained predictor has no slot to put $h$ in, so the asymmetry is structural. We do not report a mixed-rate head-to-head,
because we did not run one. Note that the
momentum must ride the same clock as the position update, so it is likewise divided
by $h$: $p_t = \mathrm{SG}(q_{t-2..t+2})/h$, so that p and the h-scaled position update share one clock.

\textbf{The SG stencil to compute the momentum-like embedding.}
The momentum-like embedding is never learned. It is always read off the encoder
trajectory by one of two fixed Savitzky--Golay stencils, least-squares quadratic fits
differentiated at $t$ (Appendix~\ref{app:design}):
\begin{itemize}
  \item the \emph{causal} stencil, over $\bq_{t-3}, \ldots, \bq_t$, gives the momentum that
  \emph{initialises} a rollout (the anchor);
  \item the \emph{centred} stencil, over $\bq_{t-2}, \ldots, \bq_{t+2}$, gives the momentum
  \emph{targets} of the loss (Eq.~\ref{eq:total}), the $\bp$ seen by SIGReg, and the reference
  $\bp$ in evaluation.
\end{itemize}
The anchor needs the causal stencil because the frames after $t$ are the ones the model is
about to predict, and reading them would leak the targets. The anchor therefore sits at latent
index $A = 3$, the first index with a full causal window.

\textbf{There is no train--test shift in the momentum estimate.} The state the flow is fed is
built by the same causal stencil everywhere: in training, in the rollout evaluation, and in
planning, both at the start state and at the goal (Appendix~\ref{app:planning}). The centred
stencil is used only to build targets, never an input. After the anchor, $\bp$ is a state
variable of the flow, and no stencil is applied to predicted latents.

The two stencils estimate the same derivative to the same order but differ in time parity.
The centred weights are antisymmetric in the offset, so reversing the frame order sends $\bp_t$
to $-\bp_t$ exactly. The causal stencil has no such symmetry: under reversal it becomes a
forward stencil rather than minus itself. The reversal diagnostic starts from an anchor whose
$\bp$ comes from the causal stencil, which is one reason it is reported as approximate rather
than exact, even for a model whose $T_{\phi}$ were perfectly even.

\subsection{Datasets and their architectures} \label{appendix:datasets}

\subsubsection{Conservative pendulum.} The dynamics of a conservative pendulum of length $L$ and angle $\theta$ is described by
\begin{equation}
    \ddot\theta = -\frac{g}{L}\sin\theta,
    \label{eq:pendulum}
\end{equation}
with $g = 9.81\,\mathrm{m\,s^{-2}}$.

Our environment integrates this differential equation with a fourth-order Runge--Kutta scheme. We consider three different rod lengths:  $L \in \{0.7, 1.0, 1.4\}\,\mathrm{m}$ rendered to scale, as shown in Fig.\ref{fig:rod-lengths}, 300
episodes of 80 frames with a time step of $\Delta t = 0.1\,\mathrm{s}$ which are split into $210/45/45$ train/val/test at the episode level. Initial conditions are drawn independently per episode: $\theta_0 \sim \mathcal{U}(-\pi,\,\pi)$ and the initial velocity is drawn in dimensionless units so that the phase-space coverage is matched across rod lengths such that $\dot\theta_0 = \tilde\omega\sqrt{g/L}$ with $\tilde\omega \sim \mathcal{U}(-a,a)$ and $a = 2\,\mathrm{rad\,s^{-1}}/\sqrt{g/L_{0}} \approx 0.639$ at the reference length $L_{0} = 1\,\mathrm{m}$, so that the bound reads $|\dot\theta_{0}| \leq 2\,\mathrm{rad\,s^{-1}}$ at $L_{0}$ and scales as $\sqrt{g/L}$ across rods.

The range $|\dot\theta_0| \leq 2\,\mathrm{rad\,s^{-1}}$ is well below the separatrix
velocity $\dot\theta_{\mathrm{sep}} = 2\sqrt{g/L} \approx 6.26\,\mathrm{rad\,s^{-1}}$,
so trajectories starting near the bottom ($\theta_0 \approx 0$) always oscillate.
However, because $\theta_0$ is drawn from the full $[-\pi, \pi]$, trajectories
starting near the unstable equilibrium ($\theta_0 \approx \pm\pi$) carry large
potential energy and rotate with any nonzero $\dot\theta_0$.
The dataset therefore spans both the oscillating and rotating regimes.

Each episode corresponds to a given rod's length. The
natural frequency of the pendulum is $\omega_0 = \sqrt{g/L}$. Thus, it spans a $1.41\times$ range and a single
learned $V_\phi$ must cover all three, which makes the rod length a system-identification
target as well as a control.

\begin{figure}
    \centering
    \includegraphics[width=0.5\linewidth]{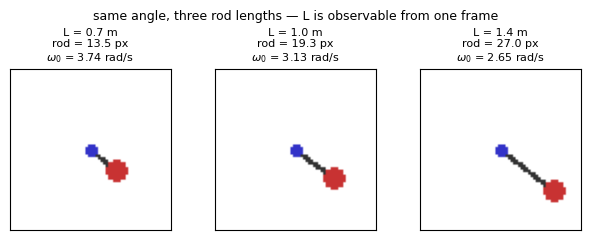}
    \caption{Frames from the three rod lengths of
    ($L \in \{0.7,1.0,1.4\}\,\mathrm{m}$), rendered to scale at $64\times64$. The rod
    length is therefore readable from a single frame, which is what makes it a positive
    control for the probes rather than a prediction target.}
    \label{fig:rod-lengths}
\end{figure}

\subsubsection{Action as a force: actuated pendulum. }

Here we consider an actuated pendulum whose rod's length is $L$ and mass $m$, whose dynamics are as the conservative pendulum in Eq.~\ref{eq:pendulum} but with an additional torque term, giving

\begin{equation}
    \ddot\theta = -\frac{g}{L}\sin\theta + \frac{u}{mL^2},
    \label{eq:pendulum_actuated}
\end{equation}

where $u$ is the applied torque.

Episodes are $64\times64$ RGB, the
manifest splits $630/135/135$ episodes at the episode level, and every rollout reported
here is driven by the recorded torque. Energy is not conserved and the potential energy $V_{\phi_2}(q,u)$ carries a time-varying input, so $H_\phi$ is explicitly
time-dependent and drift along the rollout measures how much the torque moves the learned
energy. Rod length again varies across episodes ($0.7, 1.0, 1.4$ m), as well as the rod mass ($0.5, 1.0, 2.0$ kg). In this dynamics, in contrast to the conservative pendulum, the mass impacts the torque term. Thus, a model whose latent does not carry signal of the mass, cannot know what a given torque will do.

\textbf{The torque.} The torque is held constant over $h=3$ frames ($0.3$\,s), a zero-order hold at
$3.3$\,Hz: this is $13$--$18\,\%$ of the small-oscillation period
$T=2\pi\sqrt{L/g}$ ($1.68$--$2.37$\,s over our lengths), so the drive is
band-limited to roughly $3\times$ the natural frequency and injects or removes
energy coherently over a fraction of a swing. A per-frame torque
would instead be white noise on the dynamical timescale as averaging out over a
period would leave trajectories close to the unforced pendulum. This choice also makes the torque observable
from pixels. The physical time step is $\text{dt} = 0.1 s$, so the torque is held constant through a block of 0.3 s against periods of 1.68 s (L=0.7), 2.01 s (L=1.0), 2.37 s (L=1.4), which correspond to
18\%, 15\%, 13\% of a period, with 3.4-4.8 swings per episode, which last 8s.

\subsubsection{Mass-spring, two-body and cart pole} \label{app:add-envs}
All three are rendered at $64\times64$ RGB, with $300$ episodes of $160$ frames split by episode into
$210/45/45$ for training, validation and test, as the conservative pendulum. Every model is trained
with the recipe of the conservative pendulum: the same encoder, a latent of $32+32$ dimensions for
\algo{}, an $8$-step training rollout and SIGReg.

\textbf{Mass-spring.} The \texttt{MASS\_SPRING} system of the DeepMind Hamiltonian Dynamics Suite
\citep{botev2021priors}: $H = kq^2/2 + p^2/(2m)$ with $K=2$ and $m=0.5$, so one period lasts $\pi$\,s.
Initial states lie on circles of radius $r\in[0.1, 1]$ in phase space, the bob is drawn at height
$q$, and frames are $0.05$\,s apart, as in the suite. The solution is exact.

\textbf{Two-body.} The suite's \texttt{TWO\_BODY} system: two unit masses under gravity ($g=1$),
$H=\sum_i\|\bp_i\|^2/2 - 1/\|\bq_1-\bq_2\|$, starting on circular orbits, drawn as a red and a cyan
disk. We use frames $0.2$\,s apart. At the suite's $0.05$\,s an episode of $160$ frames covers
between $0.2$ and $1.1$ of an orbit, and \algo{} fails there; at $0.2$\,s an episode covers four times more of the orbit. We report both; the $0.05$\,s results are in Table~\ref{tab:twobody-005}. At $0.05$\,s the $8$ training steps cover $1$ to $6\%$ of an orbit, a nearly straight path that shows little of the force; at $0.2$\,s they cover $4$ to $22\%$. The probe reads
the position and velocity of one body.

\textbf{Cart pole.} Gymnasium's cart pole \citep{towers2024gymnasium}, a cart of mass $1$ carrying
a uniform pole of mass $0.1$ and half-length $0.5$, with no force: the system is conservative, and
it is integrated with RK4 ($20$ substeps per frame, frames $0.02$\,s apart) rather than Gymnasium's
explicit Euler, which adds energy. RK4 only generates the frames, as for the pendulum (App.~\ref{appendix:datasets}); it is not part of the model, and \algo{}'s latent step is the leapfrog in every environment. The true cart pole cannot be stepped with an explicit leapfrog: that step needs a separable energy $H=T(\bp)+V(\bq)$, and here the kinetic energy depends on the pole angle. Our generator integrates the equations of motion in positions and velocities; there the leapfrog (velocity Verlet) needs accelerations that depend only on positions, and the cart pole's depend on velocities too. A symplectic step for it would have to be implicit, so we use a standard accurate solver to generate the data. The pole swings below the upright position and the cart drifts
at a speed of at most $0.2$. Its kinetic energy couples the cart and the pole through
$\cos\phi$, so the mass matrix depends on the pole angle $\phi$. Since $\bp=\dot\bq$ assumes a
mass that does not depend on the configuration (Section~\ref{sec:ham_jepa}), this environment
tests \algo{} outside its assumption.

\subsubsection{The data split protocol}

The two pendulums are split into training, validation and test with the same protocol,
summarised in Table~\ref{tab:datasets}. The split is $70/15/15$ and is made \emph{at the
episode level}: a whole episode belongs to exactly one split, so no frame is shared between
splits, and we verified that the episode indices of the three splits do not overlap in either
dataset. The split is also stratified by design cell (rod length for the conservative pendulum;
rod length and mass for the actuated one), so every cell contributes the same number of episodes
to each split. Training clips are cut with stride $1$ from the first $80$ frames of each
training episode, and the split is drawn with seed $42$.

For the conservative pendulum we use $300$ of the $1200$ generated episodes. The generator
sweeps four frame spacings, $300$ episodes each, and we keep only the $\Delta t = 0.1$\,s
slice, because every model in this paper is trained and evaluated at a single step size
(see the paragraph on training across recording rates in Appendix~\ref{appendix:architecture}).

\begin{table}[t]
\centering
\caption{The two pendulum environments and their episode-level splits. Both use the same protocol,
described in Appendix~\ref{appendix:datasets}.}
\label{tab:datasets}
\begin{tabular}{l cc}
\toprule
 & conservative pendulum & actuated pendulum \\
\midrule
episodes used                & 300 & 900 \\
frames per episode           & 80 & 80 \\
$\Delta t$                   & $0.10\,$s & $0.10\,$s \\
\midrule
\multicolumn{3}{l}{\emph{split, in episodes (train / val / test):}} \\
episodes                     & $\mathbf{210 / 45 / 45}$ & $\mathbf{630 / 135 / 135}$ \\
\midrule
design cells                 & 3 rod lengths & 9 ($3 L \times 3 m$) \\
episodes per cell            & 100 & 100 \\
train episodes per cell      & 70 & 70 \\
\bottomrule
\end{tabular}
\end{table}

\subsection{Evaluation}
In this subsection we describe the evaluation we perform to benchmark the models.

\subsubsection{The probes}
\label{app:probes}

Targets are $\theta$ and $\omega$. For the first one, we probe $(\sin\theta,\cos\theta)$, which are recombined by \texttt{atan2} so the reported angular error is wrapped. As described in main text, we study the performance across multiple probes and we lead with the
KNN probe trained at the first predictor's output ($K=1$). This corresponds to equal-footing as both LeWM and \algo~capture relative differences among pixel frames. While \algo~does this structurally with the embedding $\bp$, which is the approximated derivative of  $\bq$ through the SG stencil, LeWM processes the history (three time steps) through the predictor.

\subsubsection{Four measures of collapse on the data} \label{app:add-measures}
Each measure compares the rollout from many test anchors with the truth: the encoder's latents of
the true frames $K$ steps later. Each is a ratio, rollout over truth, so $1$ is right, below $1$
the rollout contracts and above $1$ it drifts.

\begin{itemize}[leftmargin=1.2em, itemsep=1pt]
\item \emph{Distance between rollouts.} For pairs of anchors $a,b$, the median of
$\|\tilde\bz^a_k-\tilde\bz^b_k\| / \|\bz^a_k-\bz^b_k\|$. Two rollouts that merge give $0$.
\item \emph{Tangent area.} The latents of real frames lie on a thin surface with as many dimensions
$r$ as the physical state ($2$ for the pendulums, the mass-spring and two-body, $4$ for the cart
pole). Around an anchor, its nearby true states form a small patch of that surface. For the
rollout and for the truth we fit the linear map from the offsets at the anchor to the offsets
after $K$ steps, and compare the areas it gives to the patch: the ratio of their determinants on
the $r$ leading directions, to the power $1/r$. Below $1$, the rollout squeezes the data surface.
\item \emph{Spread.} The root-mean-square distance between all rollout states at step $K$, over the
same quantity for the true latents.
\item \emph{Effective rank.} $(\sum_i\sigma_i)^2/\sum_i\sigma_i^2$, with $\sigma_i$ the singular
values of the centred cloud of rollout states, over that of the true latents.
\end{itemize}

Fig.~\ref{fig:app-collapse-steps} shows the four measures against the rollout step, on all five
environments, with LeWM on its own state.

\begin{figure}[h]
\centering
\includegraphics[width=\linewidth]{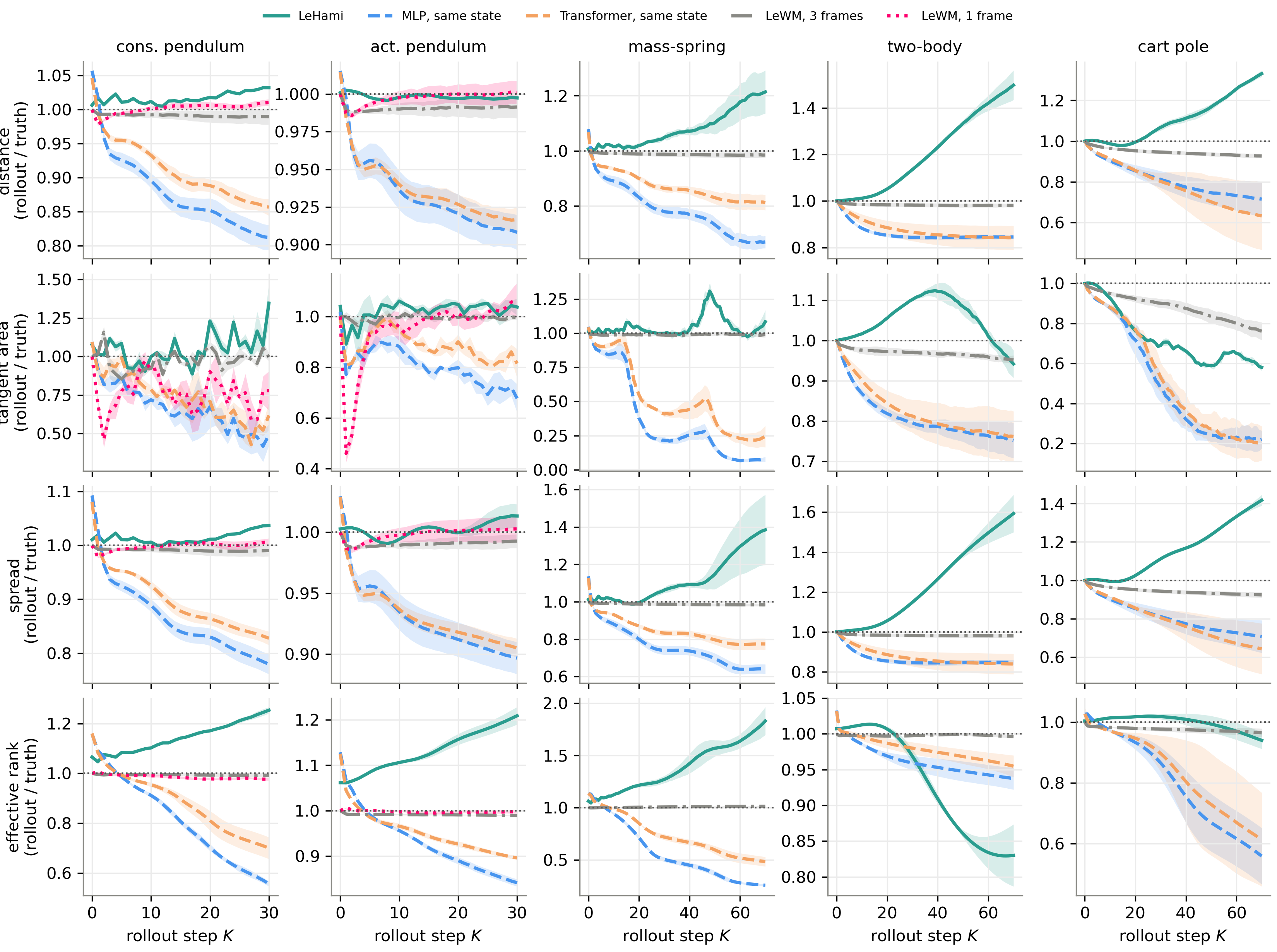}
\caption{\textbf{Collapse against the rollout step, five environments.} Rows: the four measures, rollout over truth ($1$, dotted, means the rollout keeps the true geometry; below $1$ it contracts). LeHami and the two predictors given the same state, next to LeWM on its own state from three frames (all environments) and from one frame (pendulums). Mean $\pm$ s.d.\ over three seeds. The free predictors on the same state contract on every environment and measure; LeHami's rollouts do not merge but drift above $1$ past the trained horizon, and contract some directions on two-body and the cart pole. \emph{Two-body:} LeHami's effective rank falls below $1$ after step $30$ while its distance and spread grow, so its rollouts stretch along a few directions rather than shrink, consistent with a slow drift along the orbit (its learned energy does not match the physical one at $0.2$\,s per frame). \emph{Cart pole:} its tangent area falls below $1$ because the mass depends on the pole angle, so $\bp=\dot\bq$ is not the true momentum and the true motion does not keep volume in LeHami's coordinates; LeHami's step, which does, squeezes some data directions and stretches others (distance and spread above $1$).}
\label{fig:app-collapse-steps}
\end{figure}

\textbf{Trained or random.} Proposition~\ref{prop:volume} holds for any $H_\phi$, so a \algo{} with
random weights keeps the volume of the whole latent space exactly. Table~\ref{tab:random-weights}
shows that this alone says nothing about the data: with random weights the tangent area ranges
from $0.32$ to $5.09$.

\begin{table}[h]
\centering
\caption{\textbf{Tangent area at the longest step, trained against random weights.} Mean $\pm$ s.d.\ over
three seeds; the random-weights column is one draw. Steps as in Table~\ref{tab:collapse}.}
\label{tab:random-weights}
\resizebox{\linewidth}{!}{%
\begin{tabular}{lcccc}
\toprule
environment & \algo{}, trained & \algo{}, random weights & MLP matched, trained & Transformer, trained \\
\midrule
cons.\ pendulum ($K=30$) & $1.347 \pm 0.134$ & 0.342 & $0.492 \pm 0.068$ & $0.618 \pm 0.093$ \\
act.\ pendulum ($K=30$) & $1.037 \pm 0.034$ & 5.089 & $0.677 \pm 0.056$ & $0.814 \pm 0.058$ \\
mass-spring ($K=50$) & $1.225 \pm 0.029$ & 1.340 & $0.165 \pm 0.025$ & $0.365 \pm 0.046$ \\
two-body ($K=70$) & $0.942 \pm 0.034$ & 0.322 & $0.752 \pm 0.044$ & $0.762 \pm 0.062$ \\
cart pole ($K=70$) & $0.581 \pm 0.017$ & 0.744 & $0.221 \pm 0.044$ & $0.201 \pm 0.085$ \\
\bottomrule
\end{tabular}}
\end{table}

\subsubsection{Planning}
\label{app:planning}

\emph{Results.} Table~\ref{tab:swingup-best-clause} gives the planning results.
We evaluate goal-reaching performance on goals with position and velocity clauses, separately and together, under a strict computational budget of 15 TFLOPs per episode. Under this constraint, models must balance their per-solve computational cost against their replanning interval (R). Both LeWM and \algo{} accurately decode the rod's mass and length along the whole rollout (Appendix Table~\ref{tab:param-conservation}). Since the torque's effect depends on mass and length (Appendix Eq.\ref{eq:pendulum_actuated}), both models have the information needed to predict the response to an action. Differences in planning therefore do not stem from failed system identification.

Capturing velocity dynamics within standard JEPA frameworks requires explicit architectural asymmetry. As shown in Table~\ref{tab:swingup-best-clause}, the baseline LeWM with ConvNet without time window (W=1) fails in the velocity clause because its per-frame goal embedding carries no velocity, even though its predictor infers it from context. Expanding the window with a symmetric (tied) kernel only captures speed while remaining direction-blind since it maps $\pm\omega$ to the same latent (Appendix
Table~\ref{tab:sign-blind}). As $92\%$ of the goals require the correct sign (Appendix
Fig.~\ref{fig:goal_distribution}), it performs below floor on the sign-resolved criterion of Appendix Table~\ref{tab:swingup-clause-decomp}. In contrast, expanding with untied weights breaks temporal parity, allowing the model to capture signed velocity and achieve the highest success rate. However, this unrestricted temporal formulation incurs a high computational cost per solve. To remain under the 15 TFLOP budget, LeWM W=3 (untied) is forced into larger replanning intervals.

In contrast, \algo~structurally guarantees the representation of signed velocity by embedding the system's physics. Momentum-like embedding is read off the encoder trajectory by a zero-parameter SG stencil, and the Hamiltonian flow itself has $2.23$M parameters, against LeWM's transformer and its predictor projection with $3.31$M parameters (Appendix Table~\ref{tab:multi-params}). This computational efficiency allows \algo~to evaluate more trajectories and replan more frequently (R=6 steps with 5 solves) on the velocity and joint clauses, consuming less TFLOPs. While LeWM W=3 (untied) achieves a higher success rate, \algo~provides a structurally rigorous alternative. Replacing the fixed SG stencil with a jointly learned $\bp$ embedding offers a clear path to closing this performance gap while maintaining the volume-preserving guarantees.

\begin{table}[t]
  \centering
\caption{\textbf{Goal-reaching success on the actuated pendulum, by clause.} Best success probability per model at a matched budget of at most $15$~TFLOP/episode. Mean $\pm$ s.d.\ over three model seeds. The floor probabilities (corresponding to $u = 0$) are $0.060$, $0.240$ and $0.040$ on each of the three clauses respectively. Bold marks the better model within a clause. The planner re-solves every $R$ steps towards the same goal at horizon 30, so it makes $\lceil 30/R\rceil$ solves, the last of which covers only the remaining steps.}
  \label{tab:swingup-best-clause}
  \resizebox{\linewidth}{!}{%
\begin{tabular}{llcrl}
    \toprule
    Clause & Model & Success & TFLOP & Replanning interval \\
    \midrule
    \multirow{4}{*}{Position}
      & LeWM $W{=}1$          & $\mathbf{0.800 \pm 0.019}$ & 12.87 & $R=9$, 4 solves \\
      & LeWM $W{=}3$ (tied)   & $0.633 \pm 0.020$ & 13.65 & $R=9$, 4 solves \\
      & LeWM $W{=}3$ (untied) & $0.696 \pm 0.016$ & 13.65 & $R=9$, 4 solves \\
      \cmidrule{2-5}
      & \algo~                & $0.480 \pm 0.084$ &  8.05 & $R=9$, 4 solves \\
    \midrule
    \multirow{4}{*}{Velocity}
      & LeWM $W{=}1$          & $0.184 \pm 0.027$ &  8.78 & $R=15$, 2 solves \\
      & LeWM $W{=}3$ (tied)   & $0.473 \pm 0.014$ & 13.65 & $R=9$, 4 solves \\
      & LeWM $W{=}3$ (untied) & $\mathbf{0.800 \pm 0.058}$ & 13.65 & $R=9$, 4 solves \\
      \cmidrule{2-5}
      & \algo~                & $0.527 \pm 0.019$ & 10.98 & $R=6$, 5 solves \\
    \midrule
    \multirow{4}{*}{Position and Velocity}
      & LeWM $W{=}1$          & $0.142 \pm 0.033$ & 12.87 & $R=9$, 4 solves \\
      & LeWM $W{=}3$ (tied)   & $0.400 \pm 0.037$ &  9.31 & $R=15$, 2 solves \\
      & LeWM $W{=}3$ (untied) & $\mathbf{0.636 \pm 0.025}$ & 13.65 & $R=9$, 4 solves \\
      \cmidrule{2-5}
      & \algo~                & $0.304 \pm 0.088$ & 10.98 & $R=6$, 5 solves \\
    \bottomrule
  \end{tabular}}
\end{table}

\paragraph{Protocol.}
Planning is evaluated on the actuated pendulum only, since the conservative pendulum has
no action to plan. Everything in this subsection therefore refers to the actuated pendulum,
and it completes the description in Section~\ref{sec:planning}. Goals are read off
held-out episodes, each model is scored in its own latent coordinates, and the action
sequence is optimised by CEM under MPC. Nothing in the loop observes the future.

\paragraph{Goal construction.}
Goals are drawn from test episodes, never from training data. For each of the
episodes we sample a goal frame $g$ and set the start frame to $s = g - K$, so
the target is the state the trajectory actually occupied $K$ steps after the
start. For the actuated pendulum this is a recorded $(\theta_g, \omega_g)$ that
is in general neither upright nor at rest. Building the goal this way
guarantees reachability, which the model-free pre-flight below confirms, so a
failure is attributable to the world model rather than to an impossible target.

\paragraph{Latent states, and the frames they are read from.}
Each model is given the state its own architecture defines, encoded from real
frames at two instants only. At the start, the $5$ recorded frames ending at $s$
are encoded; \algo~forms $\bq$ from the $W$ frames and $\bp = \mathrm{SG}(\bq)$ with the
same Savitzky--Golay stencil used during training, while the
baseline encodes only $\bz$ in the same way as $\bq$ (with $W$ frames) and it forms its context of $h=3$ latents and pass through the (unstructured) predictor. At the goal, \algo~encodes the
$(\bq_g, \bp_g)$ in the same way using the causal stencil with embeddings corresponding to four time steps while LeWM $\bz$ is coming from $W$ frames to produce the goal state.

\paragraph{Optimisation.}
The cost is evaluated in each model's own coordinates, so no
ground-truth state enters the objective. A state-space cost would measure the
probe rather than the world model. Actions are held over blocks of
three frames, matching the hold of the dataset, and are constrained to  to
$\lVert \ba \rVert \leq u_{\max}$. CEM uses $300$ samples, $30$ iterations and
$30$ elites. Under MPC the first $R$ actions of $\ba^{*}$
context is re-encoded from the frames the controller has by then observed, and
the problem is re-solved. The horizon shrinks to the frames remaining before the
goal instant, so the target does not recede as the episode advances; a fixed
horizon would leave the goal perpetually $H$ steps away and
success rate meaningless.

\paragraph{No leakage.}
The rollout $\tilde\bz_1, \dots, \tilde\bz_H$ is produced by self-feeding the
model: one predictor call per step, with its output fed back as the next input.
After the anchor, $\bp$ evolves under the learned flow as a state variable of the
predictor; the Savitzky-Golay stencil is applied only to encoded real frames, at
the start context and at the goal, never to predicted latents. Pixels therefore
enter at two disjoint instants, the start context and the goal window, and no
frame between them is ever encoded. No ground-truth state enters the rollout, the
cost or the ranking of candidates. Under replanning, every re-encoded context is
built from frames the controller has actually observed, all at or before the
current step. The recorded actions of the episode are used only to fill the
torques already applied at the start context, never as a candidate or as an
initialisation of the search. Success is evaluated by the simulator on the
executed trajectory, outside the planner.

\paragraph{Scoring success.} Success is scored at the goal's own timestep $s+K$, on the angle and the velocity with
separate tolerances, $|\mathrm{wrap}(\theta - \theta_g)| < \tau_\theta$ and
$|\omega - \omega_g| < \tau_\omega$. Raw errors are written for
every episode, so success as a function of tolerance is recoverable without
re-running. Scoring at the goal's timestamp rather than at any step within the
budget is necessary rather than cosmetic: under an ``ever passes near the goal''
criterion, the best of $300$ random action sequences under exact dynamics
reaches every goal, and the measurement saturates for every model at every tolerance.

Both models use the same temporal ConvNet encoder, with W=1 for \algo~and $W \in \{1,3\}$ for LeWM, and the same latent dimension d=64, with \algo~using $q$ and $p$, each of dimension $32$, and LeWM using only $z$ of dimension $64$.
SIGReg in \algo~is applied to $q$ only, while in LeWM it is applied to $z$ only. The weight of the SIGReg loss is $0.25$ in both cases, and the same number of training epochs (35) was used for both models.
Both models were trained with the same optimizer and learning rate schedule. For planning, both utilized the same CEM planner.
When training, the temporal ConvNet encoder was centered in the window, to guarantee the embeddings of \algo~are canonical. For planning,
in which we deploy online, we have to adapt. At wall-clock frame $t$, having observed frames up to $t$, the encoder can produce latents
for instants up to $t - \text{lag}$ so at planning time the context and goal latents are shifted by the encoder lag, and the rollout runs that many extra steps to reach the goal instant.

We integrate $\theta$ without wrapping, so a policy that completes $n$ revolutions ends near
$\theta_0+2\pi n$. Because the goal specifies a pose, and poses a full turn apart are
indistinguishable, we score the angle on the circle:
\begin{equation}
  \Delta\theta = \operatorname{wrap}(\theta_T-\theta_g),
  \qquad
  \operatorname{wrap}(a) = \bigl((a+\pi)\bmod 2\pi\bigr)-\pi,
  \label{eq:wrap}
\end{equation}
evaluated at the goal's own timestamp $T$. Wrapping the \emph{difference} is essential:
wrapping each angle first and subtracting introduces a spurious $2\pi$ jump at the branch cut,
reporting a near-perfect match as a maximal error. An episode succeeds when
$|\Delta\theta|<0.1$\,rad and $|\omega_T-\omega_g|<0.5$\,rad\,s$^{-1}$. As $\Delta\theta$ is
circular, we summarise it by the mean resultant length
$\bar{R}=|N^{-1}\sum_i e^{\mathrm{i}\Delta\theta_i}|$ ($0$ = uniform, $1$ = concentrated)
rather than by a linear mean and standard deviation.

\paragraph{Settings.} Goals are placed $K=30$ frames after the start, and the first
solve plans over the full horizon $H = 30$. The replanning interval is swept over
$R \in \{30, 15, 9, 6, 3\}$: $R = H$ is a single open-loop solve, and the smaller values are
MPC at increasing feedback rates from the same first-solve horizon. Actions are bounded by
$u_{\max} = 4.0$ and held for blocks of $3$ frames, as in the dataset. CEM uses $300$
samples, $30$ iterations and $30$ elites. Each configuration is run on $N = 50$ test
episodes, for $3$ model seeds $\times$ $3$ planner seeds, with tolerances
$(\tau_\theta, \tau_\omega) = (0.1\,\mathrm{rad},\ 0.5\,\mathrm{rad\,s^{-1}})$.

\section{Supplementary experimental results \label{appendix:supplementary_results}}
In this section we showcase supplementary experimental results. First, we comment on some general properties of the proposed \algo{}. To exemplify we use results corresponding to one of the two environments but the claim is in general. Then, in the following subsections we share results for each of the two environments studied.

\subsection{\algo{} replaces manual design choices with structural priors and lower compute}
\label{appendix:design-choices}

Given a sequence of pixel frames, our goal is to train a world model that captures the system's physics, specifically its higher-order dynamic states from position to its time derivatives, the velocity and the momentum it determines. In JEPA architectures, capturing these spatiotemporal differences efficiently is a central design consideration. The default implementation of LeWM \citep{maes2026leworldmodel} employs a single-frame encoder (W=1) and delegates the task of learning representations for higher-order moments entirely to a predictor that receives a three-step frame history. A direct limitation of encoding frames individually is the inability to represent goals containing velocity clauses at the single-latent level.

We assess these different design choices empirically. Fig.~\ref{fig:canonical_actuated_pendulum} shows the $R^2$ of the angle $\theta$ and angular velocity $\omega$ decoded from anchor latents ($K=0$) using a $k$-NN probe. As expected, LeWM with a single-frame encoder (W=1) fails to capture velocity, yielding an $R^2$ below zero. However, simply expanding the encoder to operate over a time window (W=3) does not inherently improve the representation of velocity and can actually deteriorate it if executed without architectural care.

Specifically, a temporal ConvNet with $W{=}3$ employs a symmetric kernel by construction when weights are tied (e.g., $[a, b, a]$ where outer taps share parameters). Consequently, the tied encoder is invariant to reversing the temporal order of the three frames in its window. Reversing time leaves the central angle unchanged while flipping $\omega \to -\omega$; the states $(\theta, +\omega)$ and $(\theta, -\omega)$ are therefore mapped to the exact same latent representation, rendering the sign of the velocity unrecoverable by any probe (see Appendix~\ref{appendix:results_act_pendulum} for a detailed breakdown). Only even functions of velocity survive: absolute velocity $|\omega|$ remains highly decodable ($R^2 = 0.92$), while signed velocity $\omega$ sits at the predict-the-mean floor.

While the parity of the kernel can be manually broken—either by untying parameters or feeding explicit frame increments $(x_t, x_t - x_{t-1})$—nothing in the standard JEPA objective forces the model to retain directional information downstream. The encoder is trained merely to make future embeddings predictable, and a representation blind to the direction of motion predicts its own sign-blind future embeddings perfectly well. The architectural choice is therefore made implicitly, with its failure modes appearing only when evaluating single-latent state readouts.

\algo{} provides a structural solution rather than an empirical workaround. Instead of requiring the temporal encoder to learn velocity implicitly, \algo~partitions the representation into canonical pairs $(\mathbf{q}, \mathbf{p})$ and extracts momentum via $\mathbf{p} = \mathrm{SG}(\mathbf{q})$, an explicitly odd stencil across latents whose parity is fixed by construction. There is no kernel to tune; the Hamiltonian latent structure dictates where $\omega$ resides, and encoder parity follows naturally.

As shown in Fig.~\ref{fig:canonical_actuated_pendulum}, the signal for predicting angular velocity $\omega$ originates almost entirely from the momentum block $\mathbf{p}$. Conversely, the angle $\theta$ is predominantly captured by the position block $\mathbf{q}$, with minor signal leakage into $\mathbf{p}$ arising because $\mathbf{p}$ is constructed as a temporal finite difference of $\mathbf{q}$.

\begin{figure}[t]
    \centering
    \includegraphics[width=\linewidth]{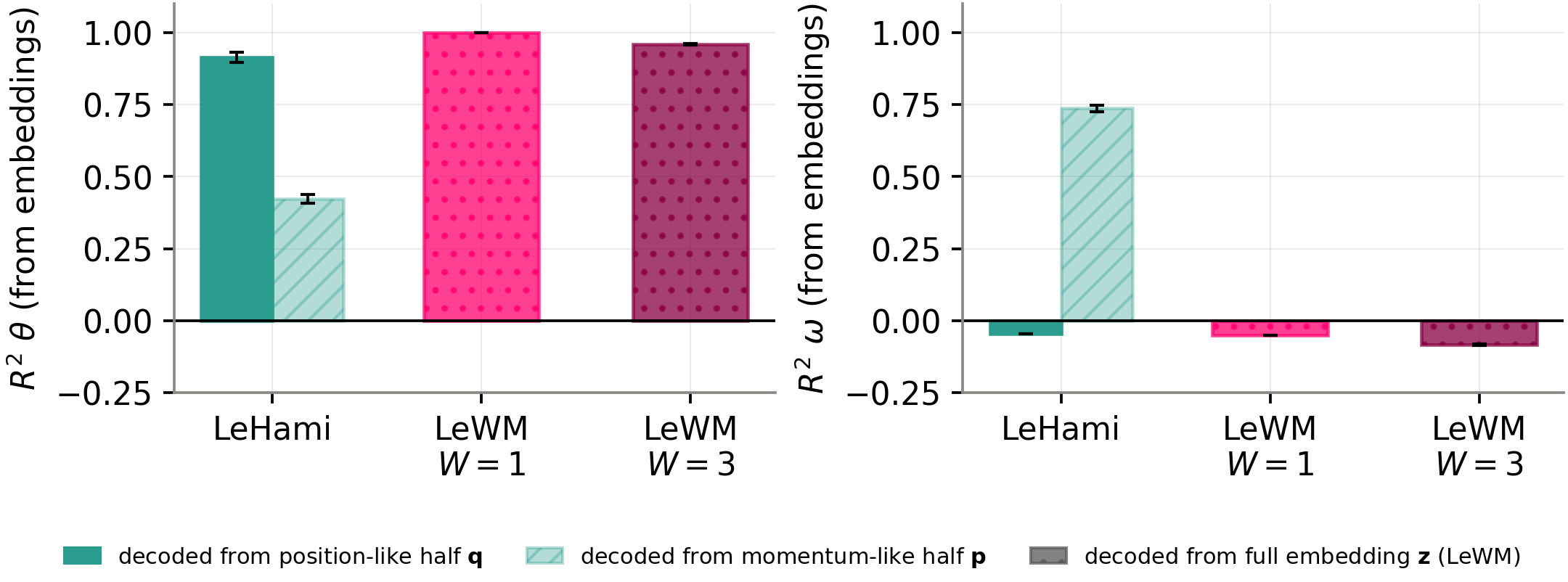}
    \caption{\textbf{The \algo~embeddings on the actuated pendulum are canonical.} $R^2$ of the pendulum angle $\theta$ \textbf{(left)} and angular velocity $\omega$ \textbf{(right)} decoded from the position half $q$ (solid) and momentum half $p$ (hatched) of the frozen encoder output. Probes are 25-NN under five episode-grouped folds fitted on embedded anchor frames ($K=0$) pooled across test episodes. $R^2$ of $\theta$ is evaluated over $(\sin\theta, \cos\theta)$. Bars indicate the mean over three training seeds; error bars denote standard deviation. LeWM has no partitioned momentum block and is reported as a single joint latent $z$.}
    \label{fig:canonical_actuated_pendulum}
\end{figure}

Furthermore, \algo~is cheaper to train. At every matched training budget on the conservative pendulum (Table~\ref{tab:exp096-isocompute}), \algo's error is $1.6$--$3.3\times$ lower on $\theta$ and $1.4$--$1.8\times$ lower on $\omega$, over three seeds. Its $0.75$~PFLOP checkpoint is on par with LeWM's at $6$~PFLOP on the angle ($0.30$ against $0.33$\,rad, within one s.d.) and better on the velocity ($1.05$ against $1.18$\,rad/s), so \algo~matches the baseline with one-eighth of the training compute. Both models cost about the same per 100-epoch run ($6.0$--$6.5$~PFLOP, $0.7$--$1.1$~GPU-h), so the gap reflects how efficiently each converts compute into prediction accuracy, not a difference in budget. \algo~also has fewer parameters (Table~\ref{tab:multi-params}), so the advantage is not explained by capacity. The first version of this paper reported one seed; the non-monotone LeWM value it had at $1.5$~PFLOP ($0.871$) came from that seed.

\begin{table}[t]
  \centering
  \caption{\textbf{MAE at a fixed training budget for the conservative pendulum.} Test error at rollout step $K=20$, read by the frozen $25$-NN probe, at the checkpoint with the lowest validation prediction loss within each compute budget (checkpoints saved every $10$ epochs). Mean $\pm$ s.d.\ over three seeds.}
  \label{tab:exp096-isocompute}
  \resizebox{\linewidth}{!}{%
  \begin{tabular}{llcccc}
    \toprule
    Metric & Model & $0.75$~PFLOP & $1.5$~PFLOP & $3$~PFLOP & $6$~PFLOP \\
    \midrule
    \multirow{2}{*}{MAE $\theta$ [rad]}
      & \algo~$W{=}1$  & $\mathbf{0.304 \pm 0.031}$ & $\mathbf{0.213 \pm 0.008}$ & $\mathbf{0.247 \pm 0.013}$ & $\mathbf{0.207 \pm 0.041}$ \\
      & LeWM $W{=}1$   & $0.642 \pm 0.048$ & $0.710 \pm 0.140$ & $0.462 \pm 0.025$ & $0.332 \pm 0.020$ \\
    \midrule
    \multirow{2}{*}{MAE $\omega$ [rad/s]}
      & \algo~$W{=}1$  & $\mathbf{1.045 \pm 0.093}$ & $\mathbf{0.939 \pm 0.062}$ & $\mathbf{0.916 \pm 0.128}$ & $\mathbf{0.854 \pm 0.075}$ \\
      & LeWM $W{=}1$   & $1.604 \pm 0.023$ & $1.694 \pm 0.408$ & $1.345 \pm 0.315$ & $1.178 \pm 0.132$ \\
    \bottomrule
  \end{tabular}}
\end{table}

\textbf{Model parameters.} The compute advantage above is not bought with capacity.
Table~\ref{tab:multi-params} counts each model's parameters from its checkpoint. The encoder is
identical across the two families, and the latent projector differs only in its output width
($32$ for $\bq$ against $64$ for $\bz$), and the SG stencil is not trainable. Thus, the difference sits in the predictor. \algo's Hamiltonian flow, the
two MLPs $T_\phi$ and $V_\phi$, has $2.23$M parameters, against $3.31$M for LeWM's
autoregressive transformer, which also needs a $0.27$M predictor projection head. In total
\algo~has about $32\%$ fewer parameters ($3.03$M against $4.44$M), and widening LeWM's encoder
to $W{=}3$ adds only $\sim$3k. The same predictor gap drives planning cost, since over $99\%$ of
a planning episode's arithmetic is spent rolling the predictor out
(Appendix~\ref{appendix:results_act_pendulum}, \emph{Planning compute}).

\begin{table}[t]
  \centering
  \caption{\textbf{Parameter counts of the models trained on the conservative pendulum environment.} The values are shown in millions,
    measured from the checkpoints. The predictor is where the two families differ:
    \algo's separable Hamiltonian flow ($T(p) + V(q)$ MLPs, \texttt{hidden\_dim} $=1024$)
    against LeWM's 6-layer autoregressive transformer. \emph{pred.\ proj.}\ is LeWM's
    predictor projection head, which \algo~has no counterpart for; the encoder and the
    latent projector are matched across families, and the $W{=}3$ encoder adds only
    $\sim$3k parameters over $W{=}1$.}
  \label{tab:multi-params}
  \begin{tabular}{lrrrrr}
    \toprule
    Model & Total & Encoder & Predictor & Projector & Pred.\ proj. \\
    \midrule
    \algo~$W{=}1$    & \textbf{3.03} & 0.315 & \textbf{2.234} & 0.465 & \textemdash \\
    LeWM $W{=}1$      & \textbf{4.44} & 0.315 & \textbf{3.310} & 0.530 & 0.268 \\
    LeWM $W{=}3$      & 4.45          & 0.318 & 3.310          & 0.530 & 0.268 \\
    \bottomrule
  \end{tabular}
\end{table}

\subsection{Conservative pendulum} \label{appendix:results_cons}

\textbf{Visualization of the latent representation space \algo{} constructs.} We take the encoded ground-truth embeddings across the frames in the episodes and across episodes and perform PCA, shown in Fig.~\ref{fig:energy-projections}. It can be seen that the trajectory is constrained to the learned $H_\phi$. This is possible because this Hamiltonian describes a conservative system, without any action. This Fig. shows empirically how the proposed \algo~imposes structure in the latent representation space that sets clear trajectories for the evolution.

\begin{figure}
    \centering
    \includegraphics[width=\linewidth]{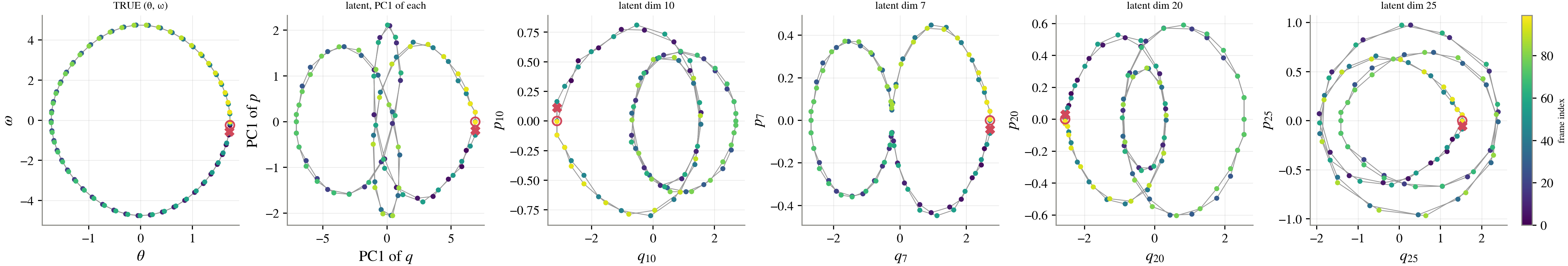}
    \caption{\textbf{Conservative pendulum: rollout latents projected onto pairs of
    principal components.} Each panel is one pair of the leading PCs of the encoder's
    $(\bq,\bp)$ over the test episodes; colour is the episode. The learned $H_\phi$ is
    conserved along each rollout, so an episode traces a closed curve in every projection
    rather than drifting across them.}
    \label{fig:energy-projections}
\end{figure}

Similarly, Fig. \ref{fig:latent_space_cons_pendulum} illustrates how the
\algo's rollout stays on a single low-dimensional closed set per episode, stratified by its learned $H_{\phi}$. The projection is obtained by mean-centring the latents per dimension and taking the SVD; the leading right-singular vectors are the components, with no whitening and no rescaling. Centring matters here because otherwise the first component points at where the cloud sits rather than at how it spreads. Otherwise PC1 just points at where the data is, not at how it spreads. Refer to Fig.\ref{fig:landscape-actuated} to see how the manifold in the latent representation space is modified by the action while still a surface of conserved energy minus work determines the trajectories. Each of the 60 test trajectories have been rolled out using the \algo~model, starting the anchor. In Fig. \ref{fig:physics_vs_\algo_cons_pendulum} shows that the learned energy is conserved and that determines a horizontal plane in which the trajectory takes place without crossing different places. We compare with the ground-truth physical system on the left in which each energy level determines a clear shape for different values of energy, restricting the trajectories to those.

 \begin{figure}
     \centering
     \includegraphics[width=\linewidth]{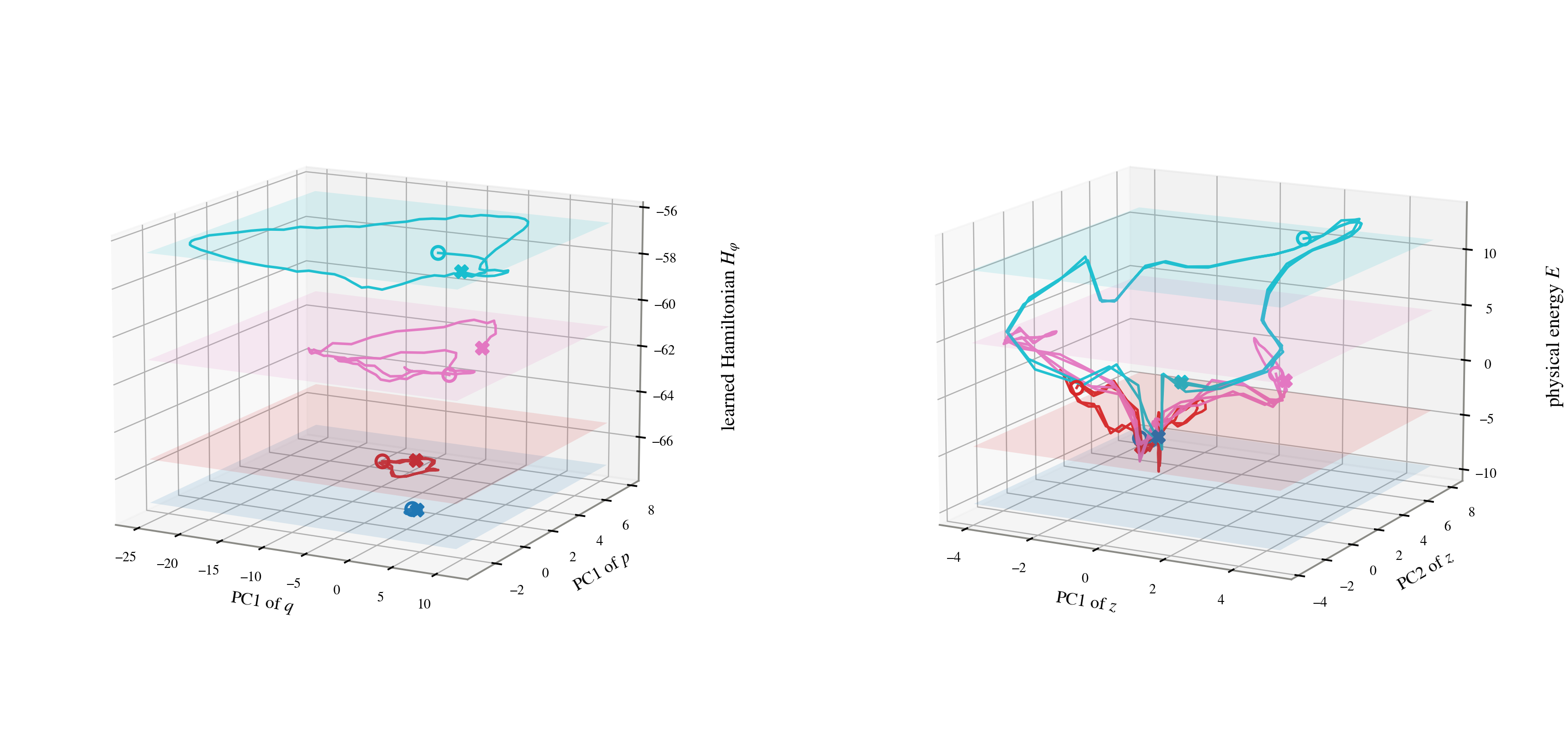}
     \caption{\textbf{Latent representation space projected into the first two principal components (PCs).} The embeddings correspond to \algo~\textbf{(left)} and unstuctured LeWM \textbf{(right)}, shown as a function (z-axis) of the Hamiltonian learned (left) and the physical energy corresponding to the episodes. The embeddings of four episodes are shown in colors. The planes with the same color correspond to the learned Hamiltonian and energy, respectively. While the rollout episode's embeddings in \algo~are restrictive in each of the planes corresponding to the (constant) respective learned Hamiltonian, in LeWM, the rollout episode's representations move across planes of energy, mixing up.}
     \label{fig:latent_space_cons_pendulum}
 \end{figure}

\begin{figure}
    \centering
    \includegraphics[width=\linewidth]{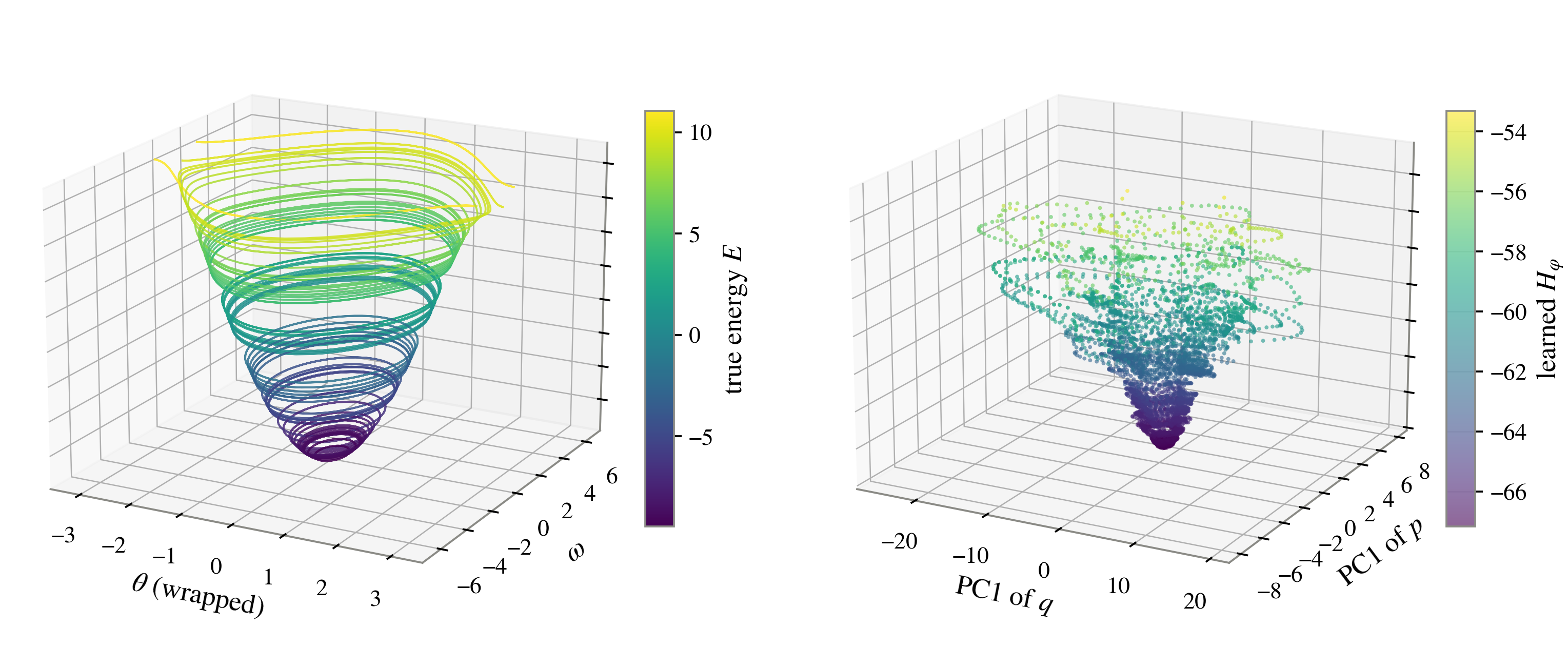}
    \caption{\textbf{The learned energy landscape against the physical one.}
    \textbf{(left)} level sets of the true energy of the conservative pendulum; \textbf{(right)}
    level sets of the learned $H_\phi$ over the same $60$ test rollouts, each started from its
    own anchor. Distinct energies give disjoint closed curves in both, so the learned generator
    stratifies the latent space the way the physical energy stratifies phase space.}
    \label{fig:physics_vs_\algo_cons_pendulum}
\end{figure}

\textbf{Rollout prediction across probes. }Linear tier underperforms non-linear: ridge recovers $\omega$ from the Hamiltonian phase space
at only $R^2 = 0.45$ against 25-NN's $0.975$ (Table~\ref{tab:cons-knnk1}), so velocity is encoded non-linearly in
$(\bq,\bp)$ and every linear number here is a lower bound. Moreover, the \algo{}'s own
momentum block degrades faster than its position block: $R^2_p$ against the reference
momentum falls to $0.613$ by $K=10$ and below zero by $K=25$, while $R^2_q$ is
still $0.950$ at $K=10$ (Table~\ref{tab:cons-latent-r2}). Both
$R^2_q$ and $R^2_p$ are computed in the latent space against the encoder's own reference
embeddings, so no angle is wrapped and the caveat that applies to $R^2(\theta)$ does not apply
here. The velocity advantage is therefore inherited from the
representation.

\begin{table}[t]
\centering
\caption{\textbf{Where the \algo{}'s two blocks lose the reference, conservative
pendulum.} $R^2$ of the predicted $\bq_k$ and $\bp_k$ against the encoded ground-truth
$\bq$ and $\bp$ at the same frame, test split. These are latent-space quantities in the model's
own coordinates, not physical ones: they measure how far the flow drifts from the encoder
trajectory it was trained to reproduce, and being self-referential they cannot by themselves
establish that a physical quantity survives. What survives decoding is in
Table~\ref{tab:cons-knnk1}..}
\label{tab:cons-latent-r2}
\begin{tabular}{l rrrrrrr}
\toprule
rollout step $K$ & 0 & 1 & 5 & 10 & 20 & 25 & 50 \\
\midrule
$R^2_q$ (position block) & $1.000$ & $0.997$ & $0.985$ & $0.950$ & $0.929$ & $0.646$ & $0.373$ \\
$R^2_p$ (momentum block) & $1.000$ & $0.935$ & $0.867$ & $0.613$ & $0.586$ & $-0.324$ & $-1.165$ \\
\bottomrule
\end{tabular}
\end{table}

\begin{table}[t]
\centering
\caption{\textbf{Performance in physics prediction of the conservative pendulum.} Wrapped MAE of angle $\theta$ and $\omega$ predicted by the $k$-NN probe (with $k=25$)
fitted on each model's own $K=1$ output and applied frozen along the rollout, averaged over
the rollout band. Ratio is LeWM $\div$ \algo, so $>1$ favours the \algo{}.}
\label{tab:cons-knnk1}
\resizebox{\linewidth}{!}{%
\begin{tabular}{l rrr rrr}
\toprule
 & \multicolumn{3}{c}{MAE $\theta$ [rad, wrapped]} & \multicolumn{3}{c}{MAE $\omega$ [rad/s]} \\
\cmidrule(lr){2-4}\cmidrule(lr){5-7}
band & \algo~& LeWM & ratio & \algo~& LeWM & ratio \\
\midrule
$K=1$--$10$  & 0.107 & 0.100 & 0.93 & \textbf{0.364} & 1.020 & \textbf{2.80} \\
$K=11$--$20$ & \textbf{0.225} & 0.287 & \textbf{1.27} & \textbf{0.669} & 1.169 & \textbf{1.75} \\
$K=21$--$70$ & \textbf{0.596} & 0.981 & \textbf{1.65} & \textbf{1.534} & 2.331 & \textbf{1.52} \\
\midrule
\multicolumn{7}{l}{\emph{anchor ($K=0$), $R^2(\omega)$ from the frozen encoder features:}
 \quad \algo~$+0.975$ \quad LeWM $-0.667$} \\
\multicolumn{7}{l}{\emph{same anchor, \algo~read by the linear tier:} \quad ridge $+0.448$,
 against the 25-NN's $+0.975$ above} \\
\bottomrule
\end{tabular}}
\end{table}

\textbf{Encoder independence.}
The encoder in the head-to-head above is a ConvNet. To check that the effect is a
property of the predictor rather than of that choice, we compare the tiny ViT  \citep{wu2022tinyvit} and ConvNet
encoder pairs on a single-rod dataset ($L = 1.0$~m, $\Delta t = 0.05$~s, $500$ episodes of
$100$ steps), holding everything else fixed. At the anchor, the
\algo{} decodes $\omega$ at $R^2 = 0.929$ with a ViT encoder and $0.939$
with a ConvNet; the unconstrained baseline scores $-0.031$ (ViT) and
$-0.018$ (ConvNet). The ordering, its sign and its magnitude are unchanged by the
encoder family, so the velocity content is attributable to the phase-space split and the
flow, not to the vision backbone. Note that this check uses the single-rod dataset, for this reason its anchor values differ from those in Table~\ref{tab:cons-knnk1}.

We showcase that the proposed \algo~can leverage a pre-trained encoder to be a world model on pre-trained visual features as DINO-WM \citep{zhou2024dino} does. For this, the pre-trained encoder, DINO-v2 \citep{oquab2023dinov2}, is followed by a projector to output $\bq \in \mathbb{R}^32$, and all the other components of \algo~remain the same. Similarly, DINO-WM has a projector to project to $64$-dimension space. This is DINO-LeHami. The loss for this model, as well as DINO-WM, is the same as the one used for both \algo~and LeWM for the training of the models we discussed their performance in Fig.~\ref{fig:cons-physics}. The evaluation is also exactly the same. Fig.~\ref{fig:dino-powered} shows that MAE in both $\theta$ and $\omega$ is reduced in comparison to their respective implementations without the pre-trained encoder. The advantage of \algo~in the angle MAE to LeWM seems to be reduced while for the velocity $\omega$ it continues to outperform, while for K larger than $60$ both models start to deviate significantly.

\begin{figure}
    \centering
    \includegraphics[width=\linewidth]{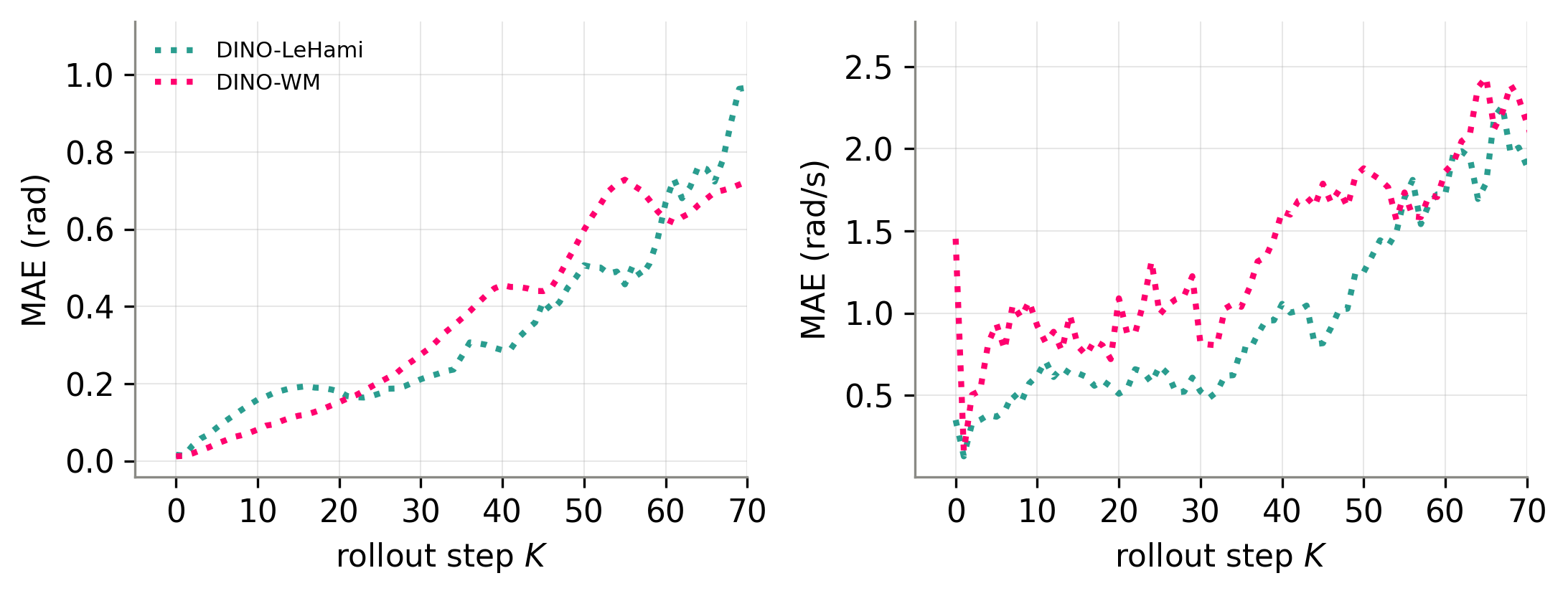}
    \caption{\textbf{DINO-LeHami vs.\ DINO-WM on the conservative pendulum.} Both models use a frozen
    DINOv2 encoder followed by a learned projector; Wrapped MAE of $\theta$ \textbf{(left)} and
    $\omega$ \textbf{(right)} along the test rollout ($K$ steps, $\Delta t = 0.05$~s), read by a
    $k$-NN probe ($k=25$) fitted at $K=1$ and kept frozen. One training seed per model, both at
    epoch $60$.}
    \label{fig:dino-powered}
\end{figure}

\textbf{Phase-space volume of the one-step map.} For a one-step predictor $F$, $\bz_{i+1} = F(\bz_i, \bu_i)$, the Jacobian
$J(\bz_i) = \partial F / \partial \bz \,|_{\bz_i}$ linearises $F$ around $\bz_i$, and
$|\det J(\bz_i)|$ is the factor by which an infinitesimal volume of states around $\bz_i$ is
scaled. By multiplicativity of the determinant, $\sum_{i<K} \log|\det J(\bz_i)|$ is the log
volume factor of the $K$-step map. We divide by the state dimension $n$ so that maps of
different sizes are comparable. A value of $0$ means volume is preserved and a negative value
means contraction. The true pendulum flow is symplectic, so $0$ is also the ground truth. For
\algo, $\bz_i = (\bq_i, \bp_i)$ and $F$ is one leapfrog step, so $\det J = 1$ exactly.

LeWM's predictor takes three context embeddings and returns one, so it has no square map from
a state to the next state. We therefore report two quantities. The full map acts on the
context, $(\bz_{t-2}, \bz_{t-1}, \bz_t) \mapsto (\bz_{t-1}, \bz_t, \bz_{t+1})$ so its Jacobian is
block-companion, and its determinant equals
$\pm\det(\partial \bz_{t+1}/\partial \bz_{t-2})$. The newest block,
$\partial \bz_{t+1}/\partial \bz_t$, has the dimension of a single embedding and measures how
the most recent state is rescaled and it is not a volume. Fig.~\ref{fig:jacobian-volume} shows both measures at the anchor of every test episode:
LeWM contracts under the full map and under the newest block alike, with no overlap with
\algo, which stays at $0$ for every state. The contraction is therefore not an artefact of
which Jacobian is chosen for LeWM's three-embedding context.

\begin{figure}[t]
    \centering
    \includegraphics[width=\linewidth]{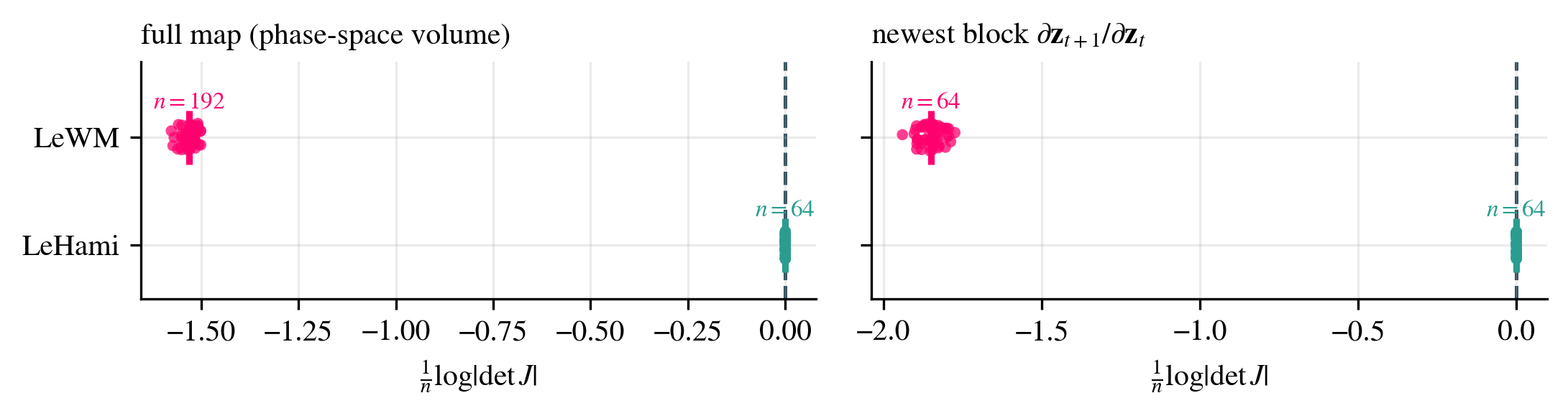}
    \caption{\textbf{Log-volume factor of the one-step latent map on the conservative
    pendulum.} One Jacobian per test episode at the anchor ($45$ states; bars mark the mean).
    \textbf{Left:} the full map, $n=64$ for \algo~on $(\bq,\bp)$ and $n=192$ for LeWM on its
    three-embedding context. \textbf{Right:} the newest block
    $\partial\bz_{t+1}/\partial\bz_t$, $n=64$ for both. The dashed line is the true flow.
    One training seed per model.}
    \label{fig:jacobian-volume}
\end{figure}

\subsection{Actuated pendulum \label{appendix:results_act_pendulum}}

\textbf{Physical state at the anchor.} Table~\ref{tab:physics_prediction_anchor} scores the encoded state at the anchor, with no rollout, so it tests the encoder and the stencil rather than the predictor; Section~\ref{sec:physics} discusses it.

\begin{table}[h]
\centering
\caption{\textbf{Performance in physics prediction of the actuated pendulum at the anchor K=0}. Mean $\pm$ s.d.over three training seeds, test split. Bold marks the better model within a probe tier. A negative $R^2$ is below the predict-the-mean floor.}
\label{tab:physics_prediction_anchor}
\begin{tabular}{l l rrr}
\toprule
model & probe & MAE $\theta$ [$10^{-2}$ rad] & MAE $\omega$ [rad/s] & $R^2(\omega)$ \\
\midrule
\multirow{3}{*}{\algo} %
  & ridge & $23.0 \pm 2.0$ & $\mathbf{1.168 \pm 0.001}$ & $\mathbf{+0.491 \pm 0.005}$ \\
  & 25-NN & $6.8 \pm 1.8$ & $\mathbf{0.419 \pm 0.045}$ & $\mathbf{+0.889 \pm 0.024}$ \\
  & MLP & $2.9 \pm 0.4$ & $\mathbf{0.184 \pm 0.019}$ & $\mathbf{+0.987 \pm 0.004}$ \\
\midrule
\multirow{3}{*}{LeWM} %
  & ridge & $\mathbf{1.4 \pm 0.1}$ & $2.422 \pm 0.001$ & $-0.014 \pm 0.001$ \\
  & 25-NN & $\mathbf{1.1 \pm 0.0}$ & $2.508 \pm 0.002$ & $-0.103 \pm 0.001$ \\
  & MLP & $\mathbf{1.9 \pm 0.1}$ & $2.419 \pm 0.001$ & $-0.011 \pm 0.001$ \\
\bottomrule
\end{tabular}
\end{table}

\textbf{Visualization of the latent representation space \algo{} constructs.} In this dataset, work is performed by the action. While in the conservative pendulum the energy determines the orbit as we observe a closed rings in Fig.~\ref{fig:physics_vs_\algo_cons_pendulum}, in the actuated case, the trajectory is free to wander over a large region in the phase-space, so the surface of $E - W$ is a horizontal plane spanning the whole space, not a ring bounding anything. Different episode's planes overlap in the phase-space because they transverse similar regions, making the latent representation space to look as in Fig. \ref{fig:landscape-actuated}, projected over the main PC of $\bq$ and $\bp$, respectively. The planes are now curves, and each one correspond to a single value of torque (action), and the trajectories can move between these planes. However, given a torque, position in the space is determined.

\begin{figure}
    \centering
    \includegraphics[width=\linewidth]{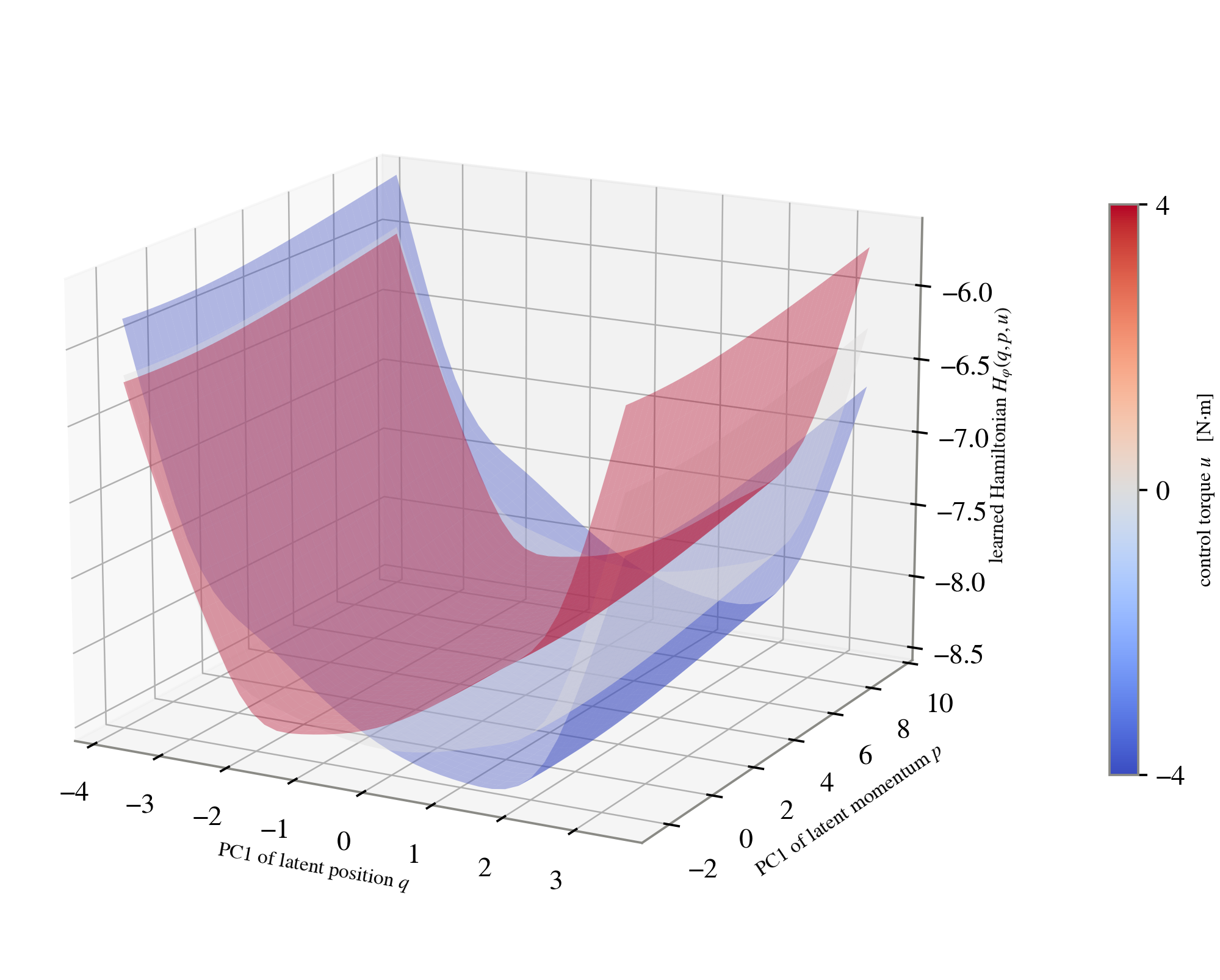}
    \caption{\textbf{The learned energy landscape under actuation.} The same
    projection as Fig.~\ref{fig:physics_vs_\algo_cons_pendulum}, for the actuated pendulum. The
    torque does work, so $H_\phi$ is not constant along a rollout and the trajectories cross the
    level sets instead of nesting into disjoint shells. Symplecticity is unaffected: it
    constrains the map, not the value of $H_\phi$.}
    \label{fig:landscape-actuated}
\end{figure}

\textbf{Where to apply regularizer in \algo.} Given the ground-truth embeddings (produced by the encoder) $(\bq, \bp)$, with  $\bp$ the Savitzky--Golay derivative of $\bq$ on
the centred-valid window, the three placements to apply SIGReg are:
$\lambda_S\,\SIGReg(\bq \Vert \bp])$ (\texttt{joint}),
$\lambda_S\,\SIGReg(\bq)$ (\texttt{q}), and
$\lambda_S\,\SIGReg(\bq) + \lambda_S\,\SIGReg(\bp)$ (\texttt{split}). We study the variance of the embeddings within-episode and between-episode in Table~\ref{tab:episode-variance}. A random unit
direction in the $64$-wide concatenation puts on average half its squared length in the $q$
block, so with $p \approx 0$ the projected variance is $\approx \tfrac12\operatorname{Var}
(q)$ and satisfying a standard-normal test (as with SIGReg) requires $\operatorname{Var}_{\text{per-dim}}
(q) \approx 2$ as we measure  $\operatorname{std}(q) = 1.36$
against a predicted $\sqrt{2} = 1.41$ Under \texttt{q}-mode, where the statistic sees only the 32 real dimensions, the
variance lands at $0.941 \pm 0.010$, the target. The joint statistic can be
blind to the momentum block and indeed it is satisfied by killing it and inflating the
position block to compensate as we see with the largest ratios of $\mathrm{std}(p)/\mathrm{std}(q)$. When splitting, the within-episode variance increases significantly. However, the standard deviation of $p$ though we apply the regularizer over, it remains far from unit. This is because $\bp$ is not free, it is conditioned on $\bq$ as it is estimated as its differentiation.

\begin{table}[t]
\centering
\caption{\textbf{Within- versus between-episode variance of the encoder's latent coordinates, by SIGReg placement.} For each latent dimension, $V_\text{within}$ is the mean squared deviation of an episode's trajectory about its own mean and $V_\text{between}$ is the variance of those per-episode means; the column reports $V_\text{within}/(V_\text{within}+V_\text{between})$, averaged over dimensions. A value near zero means the latent is an episode identifier rather than a state. Measured on the encoder's own latents (\texttt{gt\_q}, \texttt{gt\_p}) over the test episodes at a common horizon $K=0\ldots74$; mean $\pm$ standard deviation over three training seeds. Bold marks the largest within-episode share per window. The unstructured baseline attains $42.8\%$ on this measure but cannot be tabulated here, as its evaluation path stores latents only at $K=0$.}
\label{tab:episode-variance}
\begin{tabular}{ll rrrr}
\toprule
window & SIGReg & within-episode \% & $\mathrm{std}(q)$ & $\mathrm{std}(p)$ & $\mathrm{std}(p)/\mathrm{std}(q)$ \\
\midrule
\multirow{3}{*}{$W=1$} & \texttt{joint} & $0.01 \pm 0.00$ & 1.356 & 0.006 & 0.0043 \\
 & \texttt{q} & $15.32 \pm 0.29$ & 0.980 & 0.101 & 0.1031 \\
 & \texttt{split} & $\mathbf{48.89 \pm 16.39}$ & 1.174 & 0.236 & 0.2014 \\
\midrule
\multirow{3}{*}{$W=3$} & \texttt{joint} & $4.90 \pm 0.30$ & 1.405 & 0.051 & 0.0362 \\
 & \texttt{q} & $8.87 \pm 0.54$ & 0.979 & 0.051 & 0.0519 \\
 & \texttt{split} & $\mathbf{47.06 \pm 2.56}$ & 1.124 & 0.201 & 0.1784 \\
\bottomrule
\end{tabular}
\end{table}

\textbf{What the \texttt{q} row of Table~\ref{tab:episode-variance} means on the actuated pendulum.} Each episode has one of nine (mass, length) settings. Write the variance of $\bq$, summed over dimensions, as $V_w(\bq)+V_b(\bq)$: $V_w$ is the motion about each episode's average, $V_b$ the spread of the episode averages. The within share $w_q=V_w(\bq)/(V_w(\bq)+V_b(\bq))$ is $0.154$ over three seeds, so $85\%$ of the variance of $\bq$ is the episode average, which encodes the setting, and only $15\%$ is the motion within the episode. $\bp$ is a time derivative, so it has no episode average ($w_p=0.995$), and the pooled ratio factorises as $\mathrm{std}(\bp)/\mathrm{std}(\bq)=\rho_w\sqrt{w_q}=0.261\times\sqrt{0.154}=0.102$, the $0.103$ of the table, with $\rho_w=\sqrt{V_w(\bp)/V_w(\bq)}$ the ratio of the spreads of $\bp$ and $\bq$ within episodes. So $\mathrm{std}(\bp)$ looks small next to the episode averages that $\bq$ carries, not because $\bp$ is missing. The cost is elsewhere: the motion within an episode lies close to one direction of the latent (that direction holds $99.2\%$ of $V_w(\bq)$), while $45$ of the $135$ test episodes ($33\%$) go over the top at least once, $40$ of them through a full turn: their angle wraps around, which one direction cannot represent. Split by episode type, the two cases separate (Table~\ref{tab:episode-split}). On the $90$ episodes that only oscillate, $\rho_w=0.235$ matches the data's own ratio of the spreads of $\dot\theta$ and $\theta$ ($0.249$), and \algo{} matches or beats both free predictors at every step. On the $45$ that go over the top, $\rho_w=0.289$ against the data's $0.077$, and \algo{} is the least accurate of the three at short steps; its short-horizon gap on this set (Section~\ref{sec:physics}) comes from these episodes.

\textbf{Weighting the loss within episodes does not fix it.} We trained \algo{} with each clip's $\bq$ and $\bp$ prediction errors weighted by the inverse of that clip's own spread, which gives $\bp$ the same weight as $\bq$ (three seeds, conservative pendulum; Table~\ref{tab:controls}). The encoder then stores the episode and loses the motion: $w_q$ falls from $0.42$ to $0.12$ and the angle error at the anchor rises from $0.02$ to $0.59$\,rad. Weighting the momentum term up rewards a latent that barely moves, which SIGReg allows because the episode constants keep the overall spread.

\textbf{Physics prediction in rollout.} We studied the performance of \algo~applying SIGReg on three placements:
$\lambda_S\,\SIGReg(\bq \Vert \bp])$ (\texttt{joint}),
$\lambda_S\,\SIGReg(\bq)$ (\texttt{q}), and
$\lambda_S\,\SIGReg(\bq) + \lambda_S\,\SIGReg(\bp)$ (\texttt{split}). We study the variance of the embeddings within-episode and between-episode in Table~\ref{tab:episode-variance}. Based on this analysis, and on the performance in predicting the physical quantities of interest, applying SIGReg on the joint $[\bq \Vert \bp]$ has very bad performance. For this reason, we compare the performance in physics prediction of the other two variants. Table~\ref{tab:physics_prediction_anchor} reports the \texttt{q} placement, the model used
for planning; the split placement is analysed in Table~\ref{tab:episode-variance}. \algo~significantly improves the reconstruction of the velocity, over a LeWM that performs worse than predicting the average. Regarding angle, the MAE achieved by LeWM is smaller than \algo's at every probe tier, though
the gap narrows as the probe gains capacity. This is attributed to the fact that \algo~splits the capacity of the latent representation space into  $\bq$ and $\bp$, each (32) having the half the dimension of LeWM (64). Thus, an improvement in expressivity in one seems to deteriorate the other one.

The results corresponding to the rollouts are shown in Fig. \ref{fig:actuated_pendulum_physics}. The probes are trained at the anchor on training data, and the evaluated frozen. While LeWM has significantly lower MAE in angle for the smaller rollout steps, after $K \sim 40$, the MAE overlaps with the one of \algo. Regarding the velocity, for small $K$ \algos~show significantly lower MAE. As rollout progresses, after $K \sim 10$, the performance is on par. This has to do, as commented before, that the predictor takes three context frames, and after rollouts, the embeddings start to have a better representation of velocity though it suffers the expected decay due to error accumulation through rollout.

\textbf{Capturing velocity through the encoder and unstructured predictor.} While expanding the time window improves the prediction of the angle $\theta$ of LeWM, it deteriorates the performance in predicting $\omega$ unless the temporal weights are explicitly decoupled. The reason is that with no time window ($W=1$), the velocity is captured only by the predictor, which takes three context frames to perform the next prediction. In the ``tied'' $W=3$ variant, the kernel of the temporal ConvNet is symmetric by construction, with weights $[a,b,a]$. Thus, the tied encoder is invariant to reversing the order of the three frames in its window. Reversing time leaves the centre angle unchanged but flips $\omega \to -\omega$; the states $(\theta, +\omega)$ and $(\theta, -\omega)$ are therefore mapped to the same latent, and no readout of that latent can recover the sign. For a pendulum at $\theta$ with rate $\omega$, the window holds $\theta-\omega dt,\ \theta,\ \theta+\omega dt$, and the symmetric average is
\begin{equation}
  \frac{b+2a\cos(\omega\,dt)}{b+2a}\,(\sin\theta,\ \cos\theta),
\end{equation}
which corresponds to the middle frame (at time step $t$), shrunk by that scalar factor between 0 and 1. Note that this contraction is even because the cosine function is even (i.e., $\cos(\omega dt)=\cos(-\omega dt)$), so $\omega$ and $-\omega$ give the identical output. Conversely, the ``untied'' $W=3$ variant removes this parameter sharing, learning independent weights $[w_1, w_2, w_3]$. This unconstrained asymmetry allows the untied encoder to natively distinguish forward from backward motion.

This is measured empirically in Table~\ref{tab:sign-blind} where it is shown that while the $R^2(\vert{}\omega\vert{})$ is significantly improved with expanding the time window in the tied encoder, the variance explained by the vector $\omega$ is lower with a lower $R^2(\omega)$, and the gap between the two $R^2$ is amplified, in comparison to $W=1$.

\begin{table}[t]\centering
\caption{\textbf{The comparison of the $R^2$ due to the time window in the encoder.} They correspond to the $k$-NN ($k{=}25$) probe applied on encoded anchor latents over held-out episodes, across 3 seeds. The $n$ refers to the history size of the predictor, which is the same in both. Note that $W=3$ refers here to the tied variant, which exacerbates the sign-blindness gap.}
\label{tab:sign-blind}
\begin{tabular}{lrrrr}
\toprule
Model & $n$ & $R^2(\omega)$ & $R^2(\vert{}\omega\vert{})$ & gap \\
\midrule
LeWM $W{=}1$ & 3 & $-0.048$ & $+0.032$ & $+0.080$ \\
LeWM $W{=}3$ & 3 & $-0.115$ & $+0.921$ & $+1.036$ \\
\bottomrule
\end{tabular}
\end{table}

Both LeWM models use three frames for history so the predictor sees three consecutive latents to predict the next one. Ordering is preserved at the predictor even though the tied encoder discarded it, and the sign of $\omega$ is recoverable there. For $W=1$, LeWM encodes no $\omega$ at all at $K=0$ yet reads it at $R^2 \approx 0.90$ after a single predictor step, so that information can only have come from the ordering of its context. At tied $W{=}3$ that route survives but is degraded. With stride $1$, latents $j$ and $j{+}1$ are built from frames $\{j, j{+}1, j{+}2\}$ and $\{j{+}1, j{+}2, j{+}3\}$, respectively. Because these adjacent latents share two out of three frames, consecutive-frame motion is attenuated by $34\%$ compared to $W=1$. Since it is exactly this frame-to-frame change that carries the directional sign at the predictor level, the tied time window hurts twice over. It hands the model an easy, sign-blind estimate of speed from the encoder, and it attenuates the one signal coming from the change between consecutive latents, from which the direction could still have been read.

On the contrary, \algo~is on the other side of this symmetry. The Hamiltonian flow does not have any history; it is capturing velocity solely via $p$ computed with a Savitzky--Golay finite difference of $q$. This finite difference is an odd operator, hence sign-aware by construction. Any consistent first-derivative estimator satisfies $\sum_m w_m = 0$ and $\sum_m w_m m = 1$, where $m$ indexes frames within the stencil window; reversing time sends $m \to -m$ and hence the estimate to $-1$ times itself. The Hamiltonian model gets velocity not by seeing more frames or learning unrestricted asymmetric weights, but by applying an antisymmetric structural operator to them.

\begin{figure}[t]
    \centering
    \includegraphics[width=\linewidth]{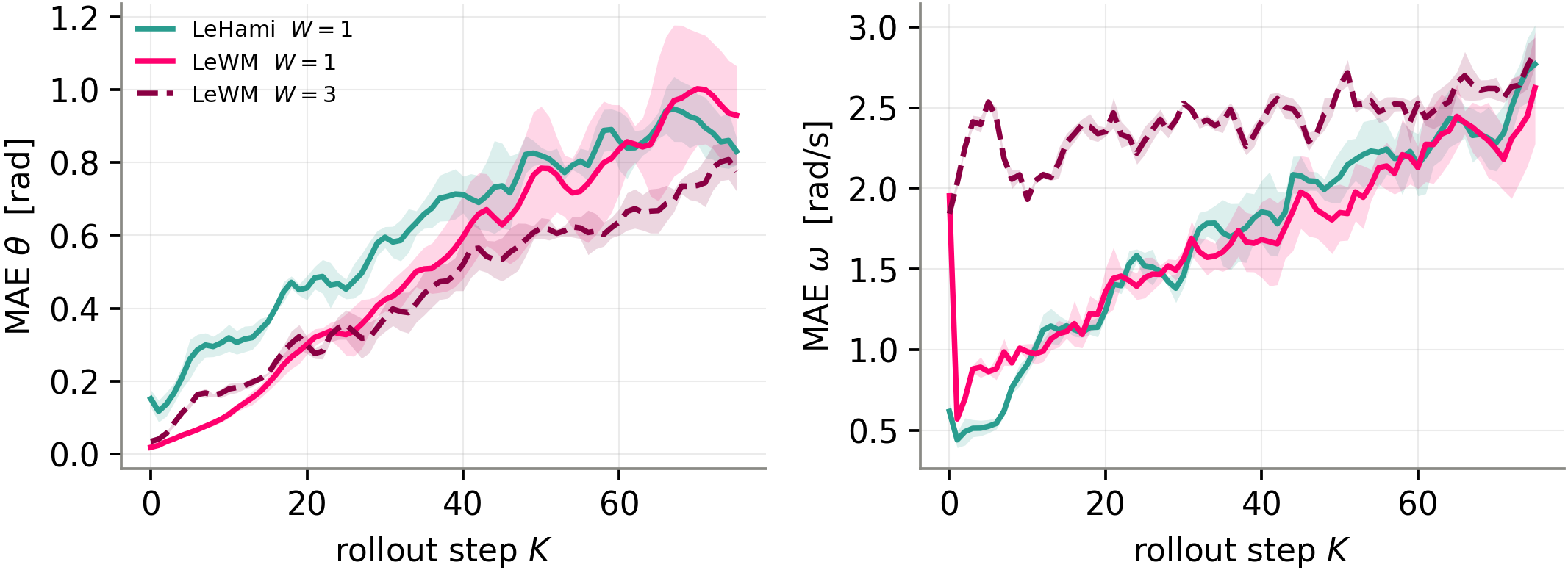}
    \caption{\textbf{Performance in physics predictions in rollout of the actuated pendulum.} Wrapped MAE of $\theta$ \textbf{(left)} and $\omega$ \textbf{(right)} along the test rollout, 25-NN probe fitted at $K=1$ and applied frozen. Same encoder, same schedule, same regularizer weight. Shaded area corresponds to the standard deviation across the three architecture seeds.}
    \label{fig:actuated_pendulum_physics}
\end{figure}

\textbf{The distribution of the angles and velocities across the goals.} As we perform a discussion about the models being able to capture the sign of the velocity, we show the distribution of the angles and the velocities across the goals for which we assess planning performance in Fig.\ref{fig:goal_distribution}. Although the distribution of velocities is not centered around zero, and thus they are not balanced in sign, the difference is not sufficient to back up a claim that the goals favored the planning performance in terms of absolute value.

\begin{figure}
    \centering
    \includegraphics[width=\linewidth]{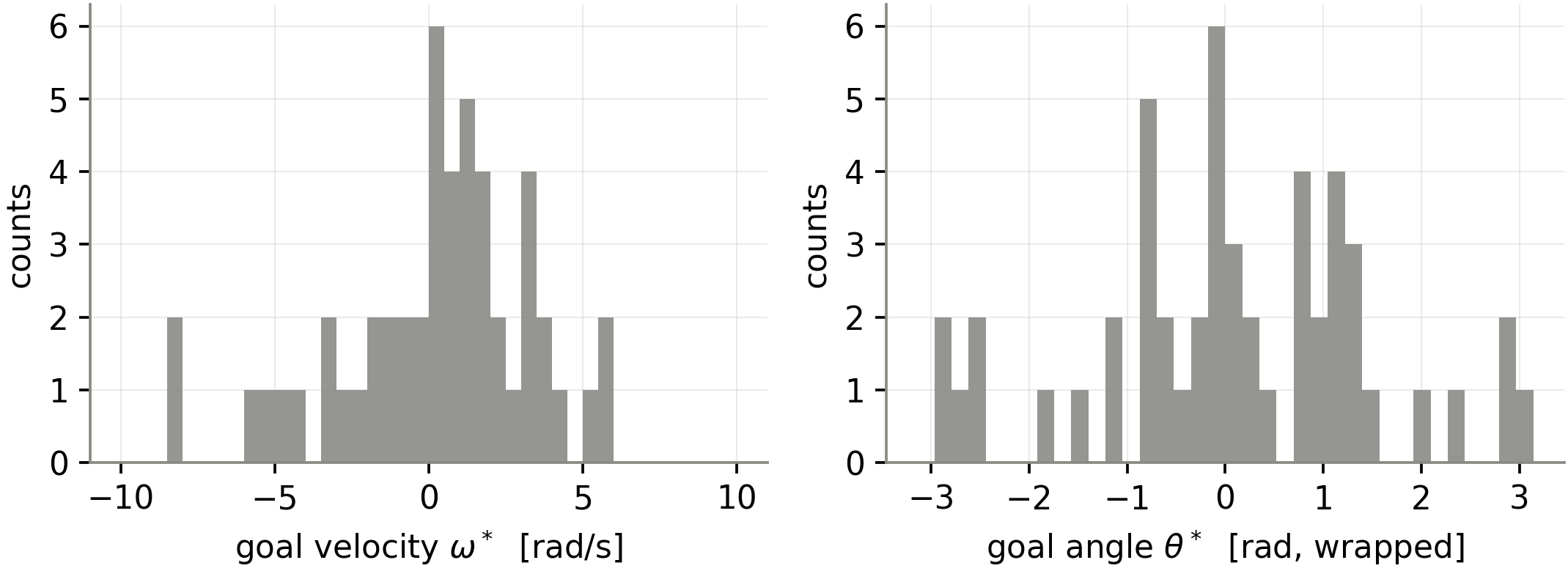}
    \caption{The $50$ planning goals. Left: the goal velocity $\omega^*$; dashed lines mark $\pm\tau_\omega/2$, beyond which a sign error alone fails the velocity clause. $92\%$ of the goals lie outside them. Right: the goal angle $\theta^*$, wrapped; the green lines mark upright, which no goal is within $\tau_\theta$ of. Goals are recorded states of held-out test episodes at $K=30$ frames after the start, shared identically by every model, rung and seed.}
    \label{fig:goal_distribution}
\end{figure}

\paragraph{Planning compute.}
At any fixed rung a planning episode costs \algo~$0.63\times$ the baseline's FLOPs. The saving is
in the predictor rather than the encoder: over $99\%$ of an episode's arithmetic is the CEM
rollout, where \algo's MLP flow costs $13.6$~MFLOP per step per candidate against $21.7$ for LeWM.
LeWM's predictor is a causal, action-conditioned transformer over the three most recent latents, in
which self-attention is the only operation that mixes them; \algo~instead recovers the
momentum-like coordinate from a causal Savitzky--Golay stencil, so its predictor never has to form
that mixture.

\textbf{Decomposing the velocity clause in planning: speed vs. direction.} To understand the representational failure modes during planning, the task's velocity clause $|\omega - \omega^*| < \tau_\omega$ can be decomposed into two distinct sub-components: a magnitude (speed) clause $\big||\omega| - |\omega^*|\big| < \tau_\omega$ and a directional sign clause $\mathrm{sign}(\omega) = \mathrm{sign}(\omega^*)$.

On this domain, directional sign is the decisive factor: $92\%$ of goals satisfy $2|\omega^*| \geq \tau_\omega$, meaning an agent arriving at the goal state spinning in the wrong direction cannot satisfy the velocity clause regardless of speed precision. Table~\ref{tab:swingup-clause-decomp} reports performance on both components.

The magnitude clause is where temporal windowing assists the baseline: LeWM $W{=}3$ reaches $0.638$ tied and $0.813$ untied, compared to \algo's $0.547$. On the sign clause, LeWM $W{=}1$ ($0.656$) and tied LeWM $W{=}3$ ($0.649$) both sit slightly below the passive $u = 0$ coasting floor ($0.680$), so neither planner gets the sign right more often than passive dynamics do. For tied LeWM $W{=}3$ this is expected: its symmetric encoder maps opposite velocity directions to the same latent. Untying the kernel breaks this temporal parity, and untied LeWM $W{=}3$ gets the sign right most often ($0.962$). \algo~and the untied LeWM $W{=}3$ are the only models that improve over coasting on the sign clause; \algo~($0.800$) does so without a learned temporal kernel and at $0.59\times$ the compute ($8.05$ against $13.65$~TFLOP).

\begin{table}[h]
  \centering
  \caption{\textbf{Velocity clause decomposition into magnitude and directional sign components.} Evaluated under the same computational budget ($\leq 15$~TFLOP/episode) and protocol as Table~\ref{tab:swingup-best-clause}.}
  \label{tab:swingup-clause-decomp}
  \resizebox{\linewidth}{!}{%
\begin{tabular}{llcrl}
    \toprule
    Clause & Model & Success & TFLOP & Replanning Interval ($R$) \\
    \midrule
    \multirow{5}{*}{Velocity, magnitude}
      & LeWM $W{=}1$          & $0.251 \pm 0.051$ &  8.78 & $R=15$, 2 solves \\
      & LeWM $W{=}3$ (tied)   & $0.638 \pm 0.008$ & 13.65 & $R=9$, 4 solves \\
      & LeWM $W{=}3$ (untied) & $\mathbf{0.813 \pm 0.052}$ & 13.65 & $R=9$, 4 solves \\
      \cmidrule{2-5}
      & \algo~                & $0.547 \pm 0.022$ & 10.98 & $R=6$, 5 solves \\
      & \textit{$u = 0$ coast} & \textit{0.300} & \textemdash & \textemdash \\
    \midrule
    \multirow{5}{*}{Velocity, sign}
      & LeWM $W{=}1$          & $0.656 \pm 0.017$ &  5.85 & $R=30$, 1 solve \\
      & LeWM $W{=}3$ (tied)   & $0.649 \pm 0.022$ & 13.65 & $R=9$, 4 solves \\
      & LeWM $W{=}3$ (untied) & $\mathbf{0.962 \pm 0.017}$ & 13.65 & $R=9$, 4 solves \\
      \cmidrule{2-5}
      & \algo~                & $0.800 \pm 0.034$ &  8.05 & $R=9$, 4 solves \\
      & \textit{$u = 0$ coast} & \textit{0.680} & \textemdash & \textemdash \\
    \bottomrule
  \end{tabular}}
\end{table}

Conditioning on the goal angle, $P(\theta \wedge \omega) = P(\theta)\,P(\omega \mid \theta)$, separates the two bottlenecks. \algo's advantage is on the sign, and it holds at every replanning interval: among plans that reach the goal angle, \algo's final velocity has the goal's sign in $0.91$--$0.95$ of cases, against $0.65$--$0.69$ for tied LeWM $W{=}3$ and $0.55$--$0.67$ for LeWM $W{=}1$. Only untied LeWM $W{=}3$ does better ($0.95$--$0.99$). The full velocity clause given the angle does not consistently favour \algo.

What limits \algo's joint success is reaching the goal angle. At $R = 9$ it reaches it on $P(\theta) = 0.48$ of plans, against $0.63$ for tied and $0.70$ for untied LeWM $W{=}3$, so $0.48 \times 0.55 = 0.26 < 0.63 \times 0.61 = 0.39$ for tied LeWM $W{=}3$. On this task the bottleneck is positional goal-reaching, not the direction of the velocity.

\textbf{Prediction of the rod's length and mass.} A symplectic flow conserves the
Hamiltonian it has learned and phase-space volume, and nothing else by construction. The
actuated pendulum has two further quantities that are exactly constant within an episode and
visible in every frame: the rod length $L \in \{0.7, 1.0, 1.4\}$\,m, drawn at its true length,
and the bob mass $m \in \{0.5, 1.0, 2.0\}$\,kg, drawn as a constant-density sphere of radius
$\propto m^{1/3}$. Neither the architecture nor the loss asks the latent to keep them. They
matter for control because the torque enters the dynamics only through $u/(mL^2)$
(Eq.~\ref{eq:pendulum_actuated}), so a model that loses $m$ cannot know what a given torque
will do. We regress rather than classify: $m$ and $L$ are continuous parameters that happen to
be sampled at three values, and $R^2$ registers the gradual drift of a frozen readout that a
three-way accuracy would saturate over. Ridge and $25$-NN probes are fitted on the anchor
($K=0$) latents of the $630$ training episodes to predict $(m, L)$ jointly, then applied frozen
to each model's own rollout and scored on the $135$ test episodes.

Table~\ref{tab:param-conservation} shows that both models carry both parameters. At the anchor
every cell reads $R^2 \geq 0.999$. After $76$ rollout steps the $25$-NN probe is still
essentially exact ($R^2 \geq 0.997$) for both models, while the frozen ridge probe drifts, more
for \algo~($0.900$ on $L$ and $0.977$ on $m$) than for LeWM ($0.983$ and $0.993$). The drift is
an affine deformation rather than a loss of information: by $K=76$ the linear readout of $L$
has over-dispersed by $3$--$6\%$ and picked up a small offset, which a nearest-neighbour
readout, depending only on which training points are closest, does not see. A local readout
that keeps working while a fixed affine one loses calibration is the signature of a parameter
that has been transported within the latent rather than destroyed. For control the drift is
inconsequential: assigning the ridge prediction of $m$ to the nearest of
$\{0.5, 1.0, 2.0\}$\,kg is correct on $99.3\%$ of episodes at $K=76$ for both models, against a
majority-class floor of $34.8\%$. The planning results of Section~\ref{sec:planning} therefore
cannot be attributed to either model being blind to the mass.

\begin{table}[t]
\centering
\caption{\textbf{Performance in prediction of rod's length and mass in the actuated pendulum.} Both episode-constant parameters are
recovered from the latent by probes fitted at anchor ($K=0$) on training
episodes and applied unchanged along each model's own rollout over $135$ held-out episodes produced by \algo~and LeWM. $R^2$ is reported at the anchor and at the end of the horizon ($K=76$). Three seeds for \algo~(min--max in brackets) for both ridge and $25$-NN probes. Both models carry both parameters. While the  $k$-NN (k=25) probe is exact at every $K$ rollout, it is observed some drift for the ridge, affecting more \algo. Note that The $R^2$ columns therefore measure how far the frozen
readout drifts, not whether the parameter is still identifiable.}
\label{tab:param-conservation}
\begin{tabular}{ll cc cc}
\toprule
 & & \multicolumn{2}{c}{rod length $L$} & \multicolumn{2}{c}{bob mass $m$} \\
\cmidrule(lr){3-4}\cmidrule(lr){5-6}
probe & model & $K=0$ & $K=76$ & $K=0$ & $K=76$ \\
\midrule
ridge & \algo & 0.999 & 0.900 {\footnotesize [0.871, 0.951]} & 0.999 & 0.977 {\footnotesize [0.958, 0.993]} \\
      & LeWM  & 0.999 & 0.983 {\footnotesize [0.976, 0.990]} & 0.999 & 0.993 \\
\midrule
$25$-NN & \algo & 1.000 & 0.997 {\footnotesize [0.992, 1.000]} & 1.000 & 1.000 \\
        & LeWM & 1.000 & 1.000 & 1.000 & 1.000 \\
\bottomrule
\end{tabular}
\end{table}

\subsection{Physical errors at a short and a long step} \label{app:add-physics}
All results in the rest of this appendix are over three training seeds, on the test split.
Fig.~\ref{fig:app-physics-steps} shows the errors at every rollout step, against chance, and
Tables~\ref{tab:pos-full} and~\ref{tab:vel-full} extend Table~\ref{tab:physics-all} with a short step.

\textbf{Collapse is not the only way to fail.} LeWM on its own one-frame state keeps the distance
between its rollouts (Fig.~\ref{fig:app-collapse-steps}), yet its error grows with the step. At a
fixed step, we split the squared error of the decoded velocity over the test episodes into three
parts: a constant offset, a wrong spread (predictions too spread out or too bunched together), and
a loss of correlation with the truth. On the conservative pendulum at $K=30$, loss of correlation
is $98\%$ of LeWM's error and $87\%$ of \algo{}'s, and LeWM's predictions keep about the true spread.
So LeWM fails by losing track of the state, for example by drifting out of phase, not by merging
states. Keeping volume prevents merging; it does not prevent losing track.

\textbf{How much of the error is collapse.} A frozen probe $P$, fitted on the encoder's latents of true frames, reads a physical quantity from the rollout: $\hat y_k = P(\tilde z_k)$, against the true value $y_k$. Over the test starts, the mean squared error of any such reading can be written as \citep{murphy1988skill}
\[
\mathbb{E}\big[(\hat y_k-y_k)^2\big] = \underbrace{(\hat\mu-\mu)^2}_{\text{offset}} + \underbrace{(\hat\sigma-r\sigma)^2}_{\text{amplitude}} + \underbrace{(1-r^2)\,\sigma^2}_{\text{loss of correlation}},
\]
with $\mu,\sigma$ and $\hat\mu,\hat\sigma$ the mean and standard deviation of $y_k$ and $\hat y_k$ over the test starts, and $r$ the Pearson correlation between $\hat y_k$ and $y_k$ over the same starts, at a fixed step $K$. This is an identity, not a model of the error. Collapse of the rollout shows up in the amplitude term, as readings that vary less than they should ($\hat\sigma<r\sigma$). Table~\ref{tab:error-split} gives its share. On the mass-spring it is most of the free predictors' error: their readings stay correlated with the truth ($0.95$ and $0.98$) but vary too little, and the amplitude part is $79\%$ and $83\%$ of their position error, against $5\%$ for \algo{}. Their rollouts there also lose rank (Table~\ref{tab:collapse}), so this is more than a change of scale: directions of the motion are lost. On the pendulums, two-body and the cart pole their correlation with the truth is low, the amplitude part is small, and the error is mostly loss of correlation; merging can cause that, but so can drifting out of step, as LeWM on its own state shows.

\begin{figure}[h]
\centering
\includegraphics[width=\linewidth]{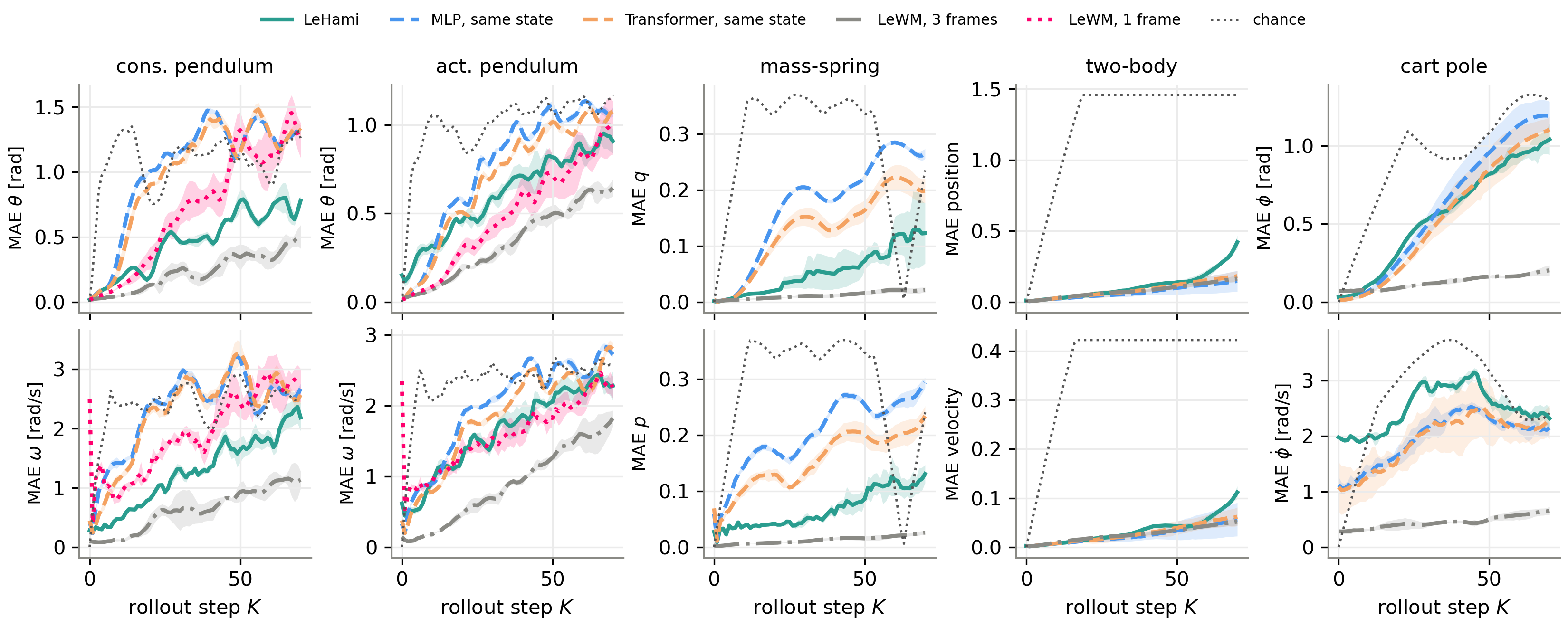}
\caption{\textbf{Physical errors against the rollout step, five environments.} MAE of the position (top) and the velocity (bottom) read by the $25$-NN probe, mean $\pm$ s.d.\ over three seeds; models as in Fig.~\ref{fig:app-collapse-steps}. Dotted: the better of the two baselines without a model. The dotted curve is the lower of the two baselines at each step, which explains its shape on the mass-spring: repeating the start is exact at step $0$, and its error grows as the bob moves away; from about step $10$ predicting the mean is lower, so the curve stays flat; near one period ($\pi$\,s, $63$ steps) the bob returns to where it started, so repeating the start becomes exact again and the curve drops.}
\label{fig:app-physics-steps}
\end{figure}
The probe is the $25$-NN probe fitted on one-step predictions of the training episodes. Two-body
positions are in canvas units (one pixel is $0.087$).

\begin{table}[h]
\centering
\caption{\textbf{MAE of the position.} Mean $\pm$ s.d.\ over three seeds. Bold: lowest of the three
predictors given the same state. The baselines use no model.}
\label{tab:pos-full}
\resizebox{\linewidth}{!}{%
\begin{tabular}{lccccc}
\toprule
environment, step & \algo{} & MLP, matched & Transformer (LeWM's) & repeat the start & predict the mean \\
\midrule
cons.\ pendulum $\theta$, $K=10$ & $\mathbf{0.177 \pm 0.002}$ & $0.416 \pm 0.021$ & $0.279 \pm 0.052$ & 1.981 & 1.299 \\
cons.\ pendulum $\theta$, $K=30$ & $\mathbf{0.460 \pm 0.028}$ & $1.154 \pm 0.002$ & $1.073 \pm 0.009$ & 1.213 & 1.195 \\
cons.\ pendulum $\theta$, $K=70$ & $\mathbf{0.779 \pm 0.034}$ & $1.310 \pm 0.015$ & $1.313 \pm 0.046$ & 1.259 & 1.324 \\
act.\ pendulum $\theta$, $K=10$ & $0.319 \pm 0.037$ & $0.259 \pm 0.012$ & $\mathbf{0.200 \pm 0.018}$ & 1.688 & 1.053 \\
act.\ pendulum $\theta$, $K=30$ & $\mathbf{0.596 \pm 0.027}$ & $0.811 \pm 0.020$ & $0.707 \pm 0.004$ & 1.108 & 1.019 \\
act.\ pendulum $\theta$, $K=70$ & $\mathbf{0.908 \pm 0.067}$ & $1.073 \pm 0.027$ & $1.076 \pm 0.015$ & 1.186 & 1.168 \\
mass-spring $q$, $K=10$ & $\mathbf{0.013 \pm 0.001}$ & $0.030 \pm 0.001$ & $0.025 \pm 0.001$ & 0.327 & 0.355 \\
mass-spring $q$, $K=50$ & $\mathbf{0.074 \pm 0.016}$ & $0.223 \pm 0.005$ & $0.164 \pm 0.018$ & 0.422 & 0.336 \\
two-body position, $K=10$ & $0.027 \pm 0.006$ & $\mathbf{0.024 \pm 0.009}$ & $0.027 \pm 0.001$ & 0.828 & 1.456 \\
two-body position, $K=70$ & $0.421 \pm 0.047$ & $\mathbf{0.148 \pm 0.073}$ & $0.183 \pm 0.037$ & 2.502 & 1.456 \\
cart pole $\phi$, $K=10$ & $0.109 \pm 0.005$ & $0.072 \pm 0.014$ & $\mathbf{0.058 \pm 0.012}$ & 0.511 & 1.224 \\
cart pole $\phi$, $K=70$ & $\mathbf{1.041 \pm 0.098}$ & $1.190 \pm 0.093$ & $1.104 \pm 0.081$ & 2.276 & 1.285 \\
\bottomrule
\end{tabular}}
\end{table}

\begin{table}[h]
\centering
\caption{\textbf{MAE of the velocity,} same layout. On the cart pole the probe reads $\dot\phi$
poorly from $[\bq,\mathrm{SG}(\bq)]$ even before any rollout, so those rows measure the probe more
than the rollout.}
\label{tab:vel-full}
\resizebox{\linewidth}{!}{%
\begin{tabular}{lccccc}
\toprule
environment, step & \algo{} & MLP, matched & Transformer (LeWM's) & repeat the start & predict the mean \\
\midrule
cons.\ pendulum $\omega$, $K=10$ & $\mathbf{0.473 \pm 0.003}$ & $1.418 \pm 0.038$ & $1.160 \pm 0.055$ & 3.602 & 2.386 \\
cons.\ pendulum $\omega$, $K=30$ & $\mathbf{1.318 \pm 0.070}$ & $2.905 \pm 0.106$ & $2.658 \pm 0.110$ & 2.813 & 2.723 \\
cons.\ pendulum $\omega$, $K=70$ & $\mathbf{2.185 \pm 0.227}$ & $2.672 \pm 0.085$ & $2.566 \pm 0.144$ & 2.741 & 2.499 \\
act.\ pendulum $\omega$, $K=10$ & $0.913 \pm 0.049$ & $0.895 \pm 0.045$ & $\mathbf{0.773 \pm 0.024}$ & 3.120 & 2.065 \\
act.\ pendulum $\omega$, $K=30$ & $\mathbf{1.458 \pm 0.059}$ & $2.089 \pm 0.052$ & $1.858 \pm 0.032$ & 2.733 & 2.607 \\
act.\ pendulum $\omega$, $K=70$ & $\mathbf{2.280 \pm 0.206}$ & $2.721 \pm 0.033$ & $2.801 \pm 0.086$ & 3.222 & 2.541 \\
mass-spring $p$, $K=10$ & $\mathbf{0.038 \pm 0.003}$ & $0.144 \pm 0.006$ & $0.103 \pm 0.011$ & 0.324 & 0.367 \\
mass-spring $p$, $K=50$ & $\mathbf{0.086 \pm 0.009}$ & $0.254 \pm 0.004$ & $0.201 \pm 0.019$ & 0.436 & 0.347 \\
two-body velocity, $K=10$ & $0.007 \pm 0.001$ & $\mathbf{0.007 \pm 0.003}$ & $0.008 \pm 0.000$ & 0.273 & 0.423 \\
two-body velocity, $K=70$ & $0.112 \pm 0.013$ & $\mathbf{0.056 \pm 0.034}$ & $0.062 \pm 0.019$ & 0.688 & 0.423 \\
cart pole $\dot\phi$, $K=10$ & $2.015 \pm 0.054$ & $1.370 \pm 0.104$ & $\mathbf{1.272 \pm 0.395}$ & 1.987 & 2.544 \\
cart pole $\dot\phi$, $K=70$ & $2.314 \pm 0.192$ & $\mathbf{2.140 \pm 0.101}$ & $2.248 \pm 0.299$ & 3.778 & 2.409 \\
\bottomrule
\end{tabular}}
\end{table}

\subsection{\algo{} against LeWM with its own state} \label{app:add-lewm}
Tables~\ref{tab:collapse} and~\ref{tab:physics-all} compare predictors on the same state.
Table~\ref{tab:lewm-own} compares \algo{} with LeWM on its own learned state: from one frame
($W{=}1$, the baseline of the main text, run on the pendulums) or from three ($W{=}3$, untied kernel,
run on all five). Both have the same latent size as \algo{} ($64$), but learn all of it, while
\algo{} learns $32$ dimensions and computes the other $32$ ($\bp=\mathrm{SG}(\bq)$); LeWM $W{=}3$
also has more parameters ($4.45$M against $3.03$M).

LeWM $W{=}3$ is more accurate on four of five environments and mostly does not contract; on the
cart pole, the hardest environment, it starts to (tangent area $0.76$), and nothing in its
predictor prevents this. LeWM $W{=}1$ is less accurate than \algo{} on the
conservative pendulum at $K=70$ ($1.24$ against $0.78$\,rad) and close on the actuated one ($0.98$
against $0.91$).

\begin{table}[h]
\centering
\caption{\textbf{\algo{} against LeWM with its own state.} Collapse cells give distance between
rollouts / tangent area / spread / effective rank, as in Table~\ref{tab:collapse}, at the same steps.
MAE as in Table~\ref{tab:physics-all}, at $K=70$ ($K=50$ for the mass-spring), mean $\pm$ s.d.\ over
three seeds. --: not run.}
\label{tab:lewm-own}
\resizebox{\linewidth}{!}{%
\begin{tabular}{lccccc}
\toprule
 & cons.\ pendulum & act.\ pendulum & mass-spring & two-body & cart pole \\
\midrule
\multicolumn{6}{l}{\emph{Collapse}} \\
\algo{} & 1.03/1.35/1.04/1.25 & 1.00/1.04/1.01/1.21 & 1.10/1.22/1.15/1.57 & 1.50/0.94/1.59/0.83 & 1.33/0.58/1.42/0.94 \\
LeWM $W{=}1$ & 1.01/0.78/1.01/0.98 & 1.00/1.06/1.00/1.00 & -- & -- & -- \\
LeWM $W{=}3$ & 0.99/1.00/0.99/0.99 & 0.99/1.00/0.99/0.99 & 0.99/1.00/0.98/1.01 & 0.98/0.95/0.98/1.00 & 0.93/0.76/0.92/0.97 \\
\midrule
\multicolumn{6}{l}{\emph{Position MAE}} \\
\algo{} & $0.779 \pm 0.034$ & $0.908 \pm 0.067$ & $0.074 \pm 0.016$ & $0.421 \pm 0.047$ & $1.041 \pm 0.098$ \\
LeWM $W{=}1$ & $1.242 \pm 0.134$ & $0.984 \pm 0.151$ & -- & -- & -- \\
LeWM $W{=}3$ & $0.502 \pm 0.090$ & $0.644 \pm 0.049$ & $0.016 \pm 0.003$ & $0.165 \pm 0.022$ & $0.203 \pm 0.033$ \\
\midrule
\multicolumn{6}{l}{\emph{Velocity MAE}} \\
\algo{} & $2.185 \pm 0.227$ & $2.280 \pm 0.206$ & $0.086 \pm 0.009$ & $0.112 \pm 0.013$ & $2.314 \pm 0.192$ \\
LeWM $W{=}1$ & $2.732 \pm 0.276$ & $2.259 \pm 0.151$ & -- & -- & -- \\
LeWM $W{=}3$ & $1.134 \pm 0.145$ & $1.820 \pm 0.112$ & $0.015 \pm 0.004$ & $0.052 \pm 0.007$ & $0.652 \pm 0.077$ \\
\bottomrule
\end{tabular}}
\end{table}

\subsection{Backward rollout} \label{app:add-backward}
From a mid-episode frame of a held-out episode we roll \algo{} back $K$ steps, with the torques
that drove the earlier frames on the actuated pendulum, and score it against the true frame $K$
steps earlier; the forward rollout from the same frame is scored against the frame $K$ steps later.
The probe is an MLP fitted on the encoder's latents of the training episodes. The baselines are
computed on the same earlier frames. Table~\ref{tab:backward} gives the numbers behind
Fig.~\ref{fig:dt-reversal} (right). The angle is recovered as well backward as forward on both
pendulums. The velocity backward beats both baselines up to $K=20$ on the conservative pendulum
and up to $K=10$ on the actuated one.

\begin{table}[h]
\centering
\caption{\textbf{Forward and backward rollouts against the true frames.} MAE, mean $\pm$ s.d.\ over
three seeds; baselines on the earlier frames.}
\label{tab:backward}
\small
\begin{tabular}{lccccc}
\toprule
 & $K$ & forward & backward & repeat the start & predict the mean \\
\midrule
\multirow{4}{*}{cons., $\theta$ [rad]}
 & 1 & $0.052 \pm 0.004$ & $0.057 \pm 0.003$ & 0.251 & 1.272 \\
 & 5 & $0.209 \pm 0.009$ & $0.206 \pm 0.002$ & 1.196 & 1.286 \\
 & 10 & $0.456 \pm 0.000$ & $0.446 \pm 0.008$ & 1.676 & 1.299 \\
 & 20 & $0.527 \pm 0.007$ & $0.507 \pm 0.008$ & 1.054 & 1.274 \\
\midrule
\multirow{4}{*}{cons., $\omega$ [rad/s]}
 & 1 & $0.238 \pm 0.005$ & $0.230 \pm 0.021$ & 0.640 & 2.509 \\
 & 5 & $0.873 \pm 0.026$ & $0.855 \pm 0.036$ & 2.870 & 2.477 \\
 & 10 & $1.141 \pm 0.013$ & $1.216 \pm 0.010$ & 4.306 & 2.433 \\
 & 20 & $1.801 \pm 0.019$ & $1.858 \pm 0.034$ & 3.047 & 2.500 \\
\midrule
\multirow{4}{*}{act., $\theta$ [rad]}
 & 1 & $0.077 \pm 0.009$ & $0.074 \pm 0.006$ & 0.242 & 1.047 \\
 & 5 & $0.343 \pm 0.030$ & $0.332 \pm 0.031$ & 1.120 & 1.041 \\
 & 10 & $0.531 \pm 0.036$ & $0.506 \pm 0.021$ & 1.487 & 1.025 \\
 & 20 & $0.716 \pm 0.050$ & $0.684 \pm 0.096$ & 1.038 & 1.004 \\
\midrule
\multirow{4}{*}{act., $\omega$ [rad/s]}
 & 1 & $0.462 \pm 0.053$ & $0.452 \pm 0.049$ & 0.669 & 2.423 \\
 & 5 & $1.708 \pm 0.316$ & $1.575 \pm 0.156$ & 2.800 & 2.382 \\
 & 10 & $2.178 \pm 0.414$ & $1.895 \pm 0.238$ & 3.909 & 2.361 \\
 & 20 & $3.193 \pm 1.124$ & $2.343 \pm 0.410$ & 2.863 & 2.282 \\
\bottomrule
\end{tabular}
\end{table}

\subsection{The mass stays in the rollout} \label{app:add-mass}
On the actuated pendulum the bob is drawn with a constant density, so its mass is visible in one
frame. Table~\ref{tab:mass-r2} reads it from each model's rollout with a linear probe fitted on the
encoder's latents of the training episodes, on held-out test episodes with the same three masses.
Every model keeps it; unit mass does not stop \algo{} from carrying it. Table~\ref{tab:param-conservation}
gives the rod length and the $25$-NN probe for \algo{} and LeWM.

\begin{table}[h]
\centering
\caption{\textbf{$R^2$ of the bob mass read from the rollout,} actuated pendulum, linear probe,
mean $\pm$ s.d.\ over three seeds.}
\label{tab:mass-r2}
\small
\begin{tabular}{lcccc}
\toprule
model & $K=0$ & $K=10$ & $K=30$ & $K=70$ \\
\midrule
\algo{} & $0.999 \pm 0.000$ & $0.999 \pm 0.000$ & $0.999 \pm 0.001$ & $0.984 \pm 0.011$ \\
LeWM $W{=}1$ & $0.999 \pm 0.000$ & $0.993 \pm 0.009$ & $0.988 \pm 0.019$ & $0.979 \pm 0.023$ \\
LeWM $W{=}3$ & $0.998 \pm 0.000$ & $0.995 \pm 0.003$ & $0.991 \pm 0.006$ & $0.982 \pm 0.016$ \\
MLP, matched & $0.999 \pm 0.000$ & $0.994 \pm 0.003$ & $0.993 \pm 0.003$ & $0.993 \pm 0.004$ \\
Transformer (LeWM's) & $0.999 \pm 0.000$ & $0.996 \pm 0.002$ & $0.995 \pm 0.002$ & $0.995 \pm 0.003$ \\
\bottomrule
\end{tabular}
\end{table}

\subsection{Further controls} \label{app:controls}
Tables referred to in Sections~\ref{sec:collapse} and~\ref{sec:physics}, and in Appendices~\ref{appendix:results_act_pendulum}, \ref{app:add-envs} and~\ref{app:add-physics}.

\begin{table}[h]
\centering
\caption{\textbf{Two-body at the suite's $0.05$\,s per frame.} MAE of the position of one body (canvas units), $25$-NN probe, mean $\pm$ s.d.\ over three seeds, and distance between rollouts over the true distance at $K=70$ ($1$ is right). Velocity and the baselines were not computed at this spacing. \algo{}'s rollouts do not merge but spread apart, and its position error is several times the others'. The state is read well at the anchor ($R^2=1.00$ for position and velocity); the $8$ training steps cover only $0.4$\,s, $1$ to $6\%$ of an orbit, too little curvature to learn the force.}
\label{tab:twobody-005}
\resizebox{\linewidth}{!}{%
\begin{tabular}{lcccc}
\toprule
 & \algo{} & MLP, matched & Transformer (LeWM's) & LeWM $W{=}3$ \\
\midrule
position, $K=10$ & $0.016 \pm 0.001$ & $0.017 \pm 0.001$ & $0.019 \pm 0.002$ & $0.015 \pm 0.001$ \\
position, $K=70$ & $0.474 \pm 0.075$ & $0.052 \pm 0.006$ & $0.091 \pm 0.014$ & $0.066 \pm 0.005$ \\
distance between rollouts, $K=70$ & 2.92 & 0.90 & 0.83 & 0.94 \\
\bottomrule
\end{tabular}}
\end{table}
\begin{table}[h]
\centering
\caption{\textbf{Share of the squared error due to the amplitude part.} Collapse shows up in this part, as predictions that vary less than they should. Each cell gives position / velocity, read by the $25$-NN probe at the steps of Table~\ref{tab:collapse}, mean over three seeds. Position and velocity: $\sin\theta$ and $\omega$ for the pendulums, $q$ and $p$ for the mass-spring, the orbit's canonical coordinates $q$ and $p$ for two-body, $\phi$ and its momentum for the cart pole. The rest of the error is almost all loss of correlation.}
\label{tab:error-split}
\resizebox{\linewidth}{!}{%
\begin{tabular}{lccccc}
\toprule
predictor & cons.\ pendulum, $K=30$ & act.\ pendulum, $K=30$ & mass-spring, $K=50$ & two-body, $K=70$ & cart pole, $K=70$ \\
\midrule
\algo{} & $14\%$ / $13\%$ & $24\%$ / $23\%$ & $5\%$ / $2\%$ & $22\%$ / $24\%$ & $17\%$ / $5\%$ \\
MLP, matched & $16\%$ / $16\%$ & $18\%$ / $9\%$ & $\mathbf{79\%}$ / $\mathbf{69\%}$ & $1\%$ / $1\%$ & $12\%$ / $1\%$ \\
Transformer (LeWM's) & $22\%$ / $15\%$ & $20\%$ / $9\%$ & $\mathbf{83\%}$ / $\mathbf{71\%}$ & $1\%$ / $1\%$ & $3\%$ / $3\%$ \\
LeWM $W{=}3$ & $6\%$ / $7\%$ & $9\%$ / $11\%$ & $1\%$ / $3\%$ & $0\%$ / $0\%$ & $0\%$ / $0\%$ \\
\bottomrule
\end{tabular}}
\end{table}

\begin{table}[h]
\centering
\caption{\textbf{Two controls on the conservative pendulum.} \emph{Euler, same $H$}: \algo{}'s $T(\bp)+V(\bq,\bu)$, networks and parameter count, with an explicit Euler step instead of the leapfrog, which does not keep volume. \emph{Within-episode loss}: \algo{} with each clip's $\bq$ and $\bp$ errors weighted by the inverse of that clip's own spread. Collapse as in Table~\ref{tab:collapse} at $K=30$ (distance / tangent area / spread / rank; above $1$ the rollout expands). MAE of $\theta$ [rad] by the $25$-NN probe; at the anchor, with the probe fitted on the encoder's latents. $w_q$: share of the variance of $\bq$ that is motion within episodes. Mean $\pm$ s.d.\ over three seeds.}
\label{tab:controls}
\resizebox{\linewidth}{!}{%
\begin{tabular}{lcccccc}
\toprule
model & collapse, $K=30$ & $\theta$, $K=10$ & $\theta$, $K=30$ & $\theta$, $K=70$ & $\theta$ at the anchor & $w_q$ \\
\midrule
\algo{} (leapfrog) & 1.03/1.35/1.04/1.25 & $\mathbf{0.177 \pm 0.002}$ & $\mathbf{0.460 \pm 0.028}$ & $\mathbf{0.779 \pm 0.034}$ & $0.020 \pm 0.000$ & 0.42 \\
Euler, same $H$ & 1.71/4.67/1.71/1.66 & $0.508 \pm 0.012$ & $1.291 \pm 0.103$ & $1.524 \pm 0.262$ & $0.018 \pm 0.001$ & 0.49 \\
\algo{}, within-episode loss & 1.00/0.83/1.00/1.10 & $1.684 \pm 0.063$ & $1.515 \pm 0.105$ & $1.364 \pm 0.085$ & $0.588 \pm 0.156$ & 0.12 \\
MLP, same state & 0.81/0.49/0.78/0.56 & $0.416 \pm 0.021$ & $1.154 \pm 0.002$ & $1.310 \pm 0.015$ & $0.036 \pm 0.000$ & 0.64 \\
Transformer, same state & 0.86/0.62/0.83/0.70 & $0.279 \pm 0.052$ & $1.073 \pm 0.009$ & $1.313 \pm 0.046$ & $0.034 \pm 0.001$ & 0.62 \\
chance & -- & 1.299 & 1.195 & 1.259 & -- & -- \\
\bottomrule
\end{tabular}}
\end{table}

\begin{table}[h]
\centering
\caption{\textbf{Actuated pendulum, by episode type.} The test episodes that go over the top at least once ($45$, \emph{rotating}) and those that never do ($90$, \emph{oscillating}). MAE of $\theta$ [rad] by the $25$-NN probe fitted at $K=1$, mean $\pm$ s.d.\ over three seeds; existing checkpoints, evaluated again.}
\label{tab:episode-split}
\resizebox{\linewidth}{!}{%
\begin{tabular}{lcccc}
\toprule
model & oscillating, $K=10$ & rotating, $K=10$ & oscillating, $K=70$ & rotating, $K=70$ \\
\midrule
\algo{} & $\mathbf{0.156 \pm 0.040}$ & $0.646 \pm 0.030$ & $\mathbf{0.603 \pm 0.106}$ & $\mathbf{1.512 \pm 0.124}$ \\
MLP, same state & $0.229 \pm 0.006$ & $0.318 \pm 0.030$ & $0.826 \pm 0.020$ & $1.564 \pm 0.044$ \\
Transformer, same state & $0.170 \pm 0.018$ & $\mathbf{0.259 \pm 0.018}$ & $0.781 \pm 0.029$ & $1.678 \pm 0.026$ \\
LeWM $W{=}3$, own state & $0.047 \pm 0.011$ & $0.120 \pm 0.020$ & $0.343 \pm 0.030$ & $1.200 \pm 0.090$ \\
\bottomrule
\end{tabular}}
\end{table}

%% file: diagram_convnet.tex
\begin{figure}[htbp]
\begin{centering}
\begin{tikzpicture}[
  font=\small,
  box/.style={draw, rounded corners=2pt, minimum height=8mm, minimum width=17mm,
              align=center, inner sep=3pt},
  new/.style={box, fill=blue!8, draw=blue!55, thick},
  ham/.style={box, fill=orange!10, draw=orange!65, thick},
  ar/.style={-{Latex[length=2mm]}, thick},
  lbl/.style={font=\scriptsize, text=black!65},
]

\node (x1) [box, minimum width=9mm] {$x_{t-2}$};
\node (x2) [box, minimum width=9mm, right=2mm of x1] {$x_{t-1}$};
\node (x3) [box, minimum width=9mm, right=2mm of x2] {$x_{t}$};
\node (xlbl) [lbl, above=1mm of x2] {$W=3$ frame window \quad $3{\times}64{\times}64$};

\node (stem) [box, below=9mm of x2, minimum width=42mm] {%
  \textbf{stem} \; Conv$3{\times}3$ $3\!\to\!32$ + BN + GELU + MaxPool};
\node (stemlbl) [lbl, right=2mm of stem.east, text width=27mm]
  {run \emph{once per frame}, shared across overlapping windows};
\draw [ar] (x1.south) -- ++(0,-3mm) -| (stem.north);
\draw [ar] (x2.south) -- (stem.north);
\draw [ar] (x3.south) -- ++(0,-3mm) -| (stem.north);

\node (tmix) [box, below=8mm of stem, minimum width=42mm] {%
  \textbf{temporal mix} \; Conv1d over $t$, $32{\to}32$\\[1pt]
  kernel tied $(a,\,b,\,a)$ \; at $32{\times}32$ spatial res.};
\draw [ar] (stem) -- node[lbl, right] {$3\times(32{\times}32{\times}32)$} (tmix);

\node (rest) [box, below=8mm of tmix, minimum width=42mm] {%
  \textbf{trunk} (unchanged) \; $32\!\to\!64\!\to\!128\!\to\!192$\\[1pt]
  3 conv blocks + AdaptiveAvgPool + Flatten};
\draw [ar] (tmix) -- node[lbl, right] {$32{\times}32{\times}32$} (rest);

\node (proj) [box, below=8mm of rest, minimum width=42mm]
  {\textbf{projector} MLP $192 \to 2048 \to d$};
\draw [ar] (rest) -- node[lbl, right] {$192$} (proj);

\node (q) [box, below=8mm of proj, minimum width=16mm] {$q_t \in \mathbb{R}^{d}$};
\draw [ar] (proj) -- (q);

\node (p) [box, right=16mm of q, minimum width=30mm]
  {$p_t = \mathrm{SG}(q_{t-3..t})$};
\draw [ar] (q.east) -- node[lbl, above] {causal} (p.west);

\node (flow) [ham, below=9mm of $(q)!0.5!(p)$, minimum width=62mm] {%
  \textbf{leapfrog (KDK)} \; $H = T(p/s) + V(q;u)$\\[2pt]
  $\begin{aligned}[t]
     p_{t+\frac12} &= p_t - \tfrac{\Delta t}{2}\,\nabla_q V(q_t;u) \\[-1pt]
     q_{t+1}       &= q_t + \Delta t\,\nabla_p T(p_{t+\frac12}/s) \\[-1pt]
     p_{t+1}       &= p_{t+\frac12} - \tfrac{\Delta t}{2}\,\nabla_q V(q_{t+1};u)
   \end{aligned}$};
\draw [ar] (q.south) |- ([yshift=3mm]flow.north west) -- ++(0,-3mm);
\draw [ar] (p.south) |- ([yshift=3mm]flow.north east) -- ++(0,-3mm);

\node (mass) [new, left=14mm of flow, minimum width=26mm]
  {$s = e^{\sigma}$ learnable};
\node [lbl, below=1mm of mass, text width=26mm, align=center]
  {per-dimension momentum scale};
\draw [ar] (mass) -- (flow);

\node [lbl, below=3mm of flow, text width=62mm, align=center]
  {$(q_{t+1}, p_{t+1})$ --- one dataset frame ahead};
\end{tikzpicture}
\caption{\textbf{The temporal ConvNet encoder inside the \algo~framework}, drawn at $W = 3$.
  A window of frames goes through a stem applied once per frame, then a single Conv1d over
  time with a symmetric tied kernel $(a, b, a)$, then the baseline's spatial trunk and
  projector, unchanged, to give $\bq_t$. The same encoder is used by the $W{=}3$ LeWM
  baselines; the main \algo~model uses $W{=}1$, i.e.\ no temporal mix. The rollout is
  initialised with the causal Savitzky--Golay momentum $\bp_t$, computed from the current
  and three previous positions only, and advanced one frame at a time by the leapfrog step
  (orange) of $H_\phi = T_\phi(\bp/s) + V_\phi(\bq,\bu)$, where the learnable per-dimension
  scale $s$ (blue) rescales the momentum and therefore appears only in the position update.}
\label{diagram:convet}
\end{centering}
\end{figure}

%% file: iclr2027_conference.bib
@article{balestriero2025lejepa,
  title     = {{LeJEPA}: Provable and Scalable Self-Supervised Learning
               Without the Heuristics},
  author    = {Balestriero, Randall and LeCun, Yann},
  journal   = {arXiv preprint arXiv:2511.08544},
  year      = {2025},
  url       = {https://arxiv.org/abs/2511.08544},
}

@article{maes2026leworldmodel,
  title     = {{LeWorldModel}: Stable End-to-End Joint-Embedding Predictive
               Architecture from Pixels},
  author    = {Maes, Lucas and Le Lidec, Quentin and Scieur, Damien and
               LeCun, Yann and Balestriero, Randall},
  journal   = {arXiv preprint arXiv:2603.19312},
  year      = {2026},
  url       = {https://arxiv.org/abs/2603.19312},
}

@article{assran2025v,
  title={V-jepa 2: Self-supervised video models enable understanding, prediction and planning},
  author={Assran, Mido and Bardes, Adrien and Fan, David and Garrido, Quentin and Howes, Russell and Muckley, Matthew and Rizvi, Ammar and Roberts, Claire and Sinha, Koustuv and Zholus, Artem and others},
  journal={arXiv preprint arXiv:2506.09985},
  year={2025}
}

@article{greydanus2019hamiltonian,
  title={Hamiltonian neural networks},
  author={Greydanus, Samuel and Dzamba, Misko and Yosinski, Jason},
  journal={Advances in neural information processing systems},
  volume={32},
  year={2019}
}

@article{zhang2026hierarchical,
  title={Hierarchical Planning with Latent World Models},
  author={Zhang, Wancong and Terver, Basile and Zholus, Artem and Chitnis, Soham and Sutaria, Harsh and Assran, Mido and Balestriero, Randall and Bar, Amir and Bardes, Adrien and LeCun, Yann and others},
  journal={arXiv preprint arXiv:2604.03208},
  year={2026}
}

@article{bardes2021vicreg,
  title={Vicreg: Variance-invariance-covariance regularization for self-supervised learning},
  author={Bardes, Adrien and Ponce, Jean and LeCun, Yann},
  journal={arXiv preprint arXiv:2105.04906},
  year={2021}
}

@article{sobal2025learning,
  title={Learning from reward-free offline data: A case for planning with latent dynamics models},
  author={Sobal, Vlad and Zhang, Wancong and Cho, Kyunghyun and Balestriero, Randall and Rudner, Tim GJ and LeCun, Yann},
  journal={arXiv preprint arXiv:2502.14819},
  year={2025}
}

@article{zhou2024dino,
  title={Dino-wm: World models on pre-trained visual features enable zero-shot planning},
  author={Zhou, Gaoyue and Pan, Hengkai and LeCun, Yann and Pinto, Lerrel},
  journal={arXiv preprint arXiv:2411.04983},
  year={2024}
}

@article{lecun2022path,
  title={A path towards autonomous machine intelligence version 0.9. 2, 2022-06-27},
  author={LeCun, Yann and others},
  journal={Open Review},
  volume={62},
  number={1},
  pages={1--62},
  year={2022}
}

@book{arnold1989mathematical,
  title     = {Mathematical Methods of Classical Mechanics},
  author    = {Arnold, Vladimir I.},
  edition   = {2nd},
  publisher = {Springer},
  year      = {1989},
}

@article{chen2019symplectic,
  title={Symplectic recurrent neural networks},
  author={Chen, Zhengdao and Zhang, Jianyu and Arjovsky, Martin and Bottou, L{\'e}on},
  journal={arXiv preprint arXiv:1909.13334},
  year={2019}
}

@inproceedings{toth2019hamiltonian,
  title={Hamiltonian generative networks},
  author={Toth, Peter and Rezende, Danilo J and Jaegle, Andrew and Racani{\`e}re, S{\'e}bastien and Botev, Aleksandar and Higgins, Irina},
  booktitle={International conference on learning representations},
  year={2019}
}

@inproceedings{zhong2020dissipative,
  title     = {Dissipative {SymODEN}: Encoding {Hamiltonian} Dynamics with Dissipation and Control into Deep Learning},
  author    = {Zhong, Yaofeng Desmond and Dey, Biswadip and Chakraborty, Amit},
  booktitle = {ICLR 2020 Workshop on Integration of Deep Neural Models and Differential Equations},
  year      = {2020},
  eprint    = {2002.08860},
  archivePrefix = {arXiv},
}

@book{goldstein2011classical,
  title={Classical mechanics},
  author={Goldstein, Herbert},
  year={2011},
  publisher={Pearson Education India}
}

@article{doshi2024scaling,
  title={Scaling cross-embodied learning: One policy for manipulation, navigation, locomotion and aviation},
  author={Doshi, Ria and Walke, Homer and Mees, Oier and Dasari, Sudeep and Levine, Sergey},
  journal={arXiv preprint arXiv:2408.11812},
  year={2024}
}

@article{mitchell2005time,
  title={A time-dependent Hamilton-Jacobi formulation of reachable sets for continuous dynamic games},
  author={Mitchell, Ian M and Bayen, Alexandre M and Tomlin, Claire J},
  journal={IEEE Transactions on automatic control},
  volume={50},
  number={7},
  pages={947--957},
  year={2005},
  publisher={IEEE}
}

@inproceedings{herbert2017fastrack,
  title={FaSTrack: A modular framework for fast and guaranteed safe motion planning},
  author={Herbert, Sylvia L and Chen, Mo and Han, SooJean and Bansal, Somil and Fisac, Jaime F and Tomlin, Claire J},
  booktitle={2017 IEEE 56th Annual Conference on Decision and Control (CDC)},
  pages={1517--1522},
  year={2017},
  organization={IEEE}
}

@article{cui2026physically,
  title={Physically Native World Models: A Hamiltonian Perspective on Generative World Modeling},
  author={Cui, Sen and Ma, Jingheng},
  journal={arXiv preprint arXiv:2605.00412},
  year={2026}
}

@article{luan2026ph,
  title={PH-Dreamer: A Physics-Driven World Model via Port-Hamiltonian Generative Dynamics},
  author={Luan, Xueyu and Shi, Chenwei},
  journal={arXiv preprint arXiv:2605.18303},
  year={2026}
}

@article{tang2026ham,
  title={HaM-World: Soft-Hamiltonian World Models with Selective Memory for Planning},
  author={Tang, Haoyun and Cui, Haodong and Xu, Keyao and Wang, Kun and Mei, Zhandong},
  journal={arXiv preprint arXiv:2605.05951},
  year={2026}
}

@article{tan2016model,
  title={Model predictive control with non-uniformly spaced optimization horizon for multi-timescale processes},
  author={Tan, Chee Keong and Tippett, Michael James and Bao, Jie},
  journal={Computers \& Chemical Engineering},
  volume={84},
  pages={162--170},
  year={2016},
  publisher={Elsevier}
}

@article{hairer2003geometric,
  title={Geometric numerical integration illustrated by the St{\"o}rmer--Verlet method},
  author={Hairer, Ernst and Lubich, Christian and Wanner, Gerhard},
  journal={Acta numerica},
  volume={12},
  pages={399--450},
  year={2003},
  publisher={Cambridge University Press}
}

@article{gorry1990general,
  title={General least-squares smoothing and differentiation by the convolution (Savitzky-Golay) method},
  author={Gorry, Peter A},
  journal={Analytical Chemistry},
  volume={62},
  number={6},
  pages={570--573},
  year={1990},
  publisher={ACS Publications}
}

@article{kim2014catching,
  title={Catching objects in flight},
  author={Kim, Seungsu and Shukla, Ashwini and Billard, Aude},
  journal={IEEE Transactions on Robotics},
  volume={30},
  number={5},
  pages={1049--1065},
  year={2014},
  publisher={IEEE}
}

@article{zeng2020tossingbot,
  title={Tossingbot: Learning to throw arbitrary objects with residual physics},
  author={Zeng, Andy and Song, Shuran and Lee, Johnny and Rodriguez, Alberto and Funkhouser, Thomas},
  journal={IEEE Transactions on Robotics},
  volume={36},
  number={4},
  pages={1307--1319},
  year={2020},
  publisher={IEEE}
}

@book{marsden1999introduction,
  title={Introduction to mechanics and symmetry: a basic exposition of classical mechanical systems},
  author={Marsden, Jerrold E and Ratiu, Tudor S and Golubitsky, M},
  volume={17},
  year={1999},
  publisher={Springer}
}

@article{duong2021hamiltonian,
  title={Hamiltonian-based neural ODE networks on the SE (3) manifold for dynamics learning and control},
  author={Duong, Thai and Atanasov, Nikolay},
  journal={arXiv preprint arXiv:2106.12782},
  year={2021}
}

@misc{maes_lld2026swm,
  title  = {stable-worldmodel: A Platform for Reproducible World Modeling Research and Evaluation},
  author = {Lucas Maes and Quentin Le Lidec and Luiz Facury and Nassim Massaudi and
            Ayush Chaurasia and Francesco Capuano and Richard Gao and Taj Gillin and
            Dan Haramati and Damien Scieur and Yann LeCun and Randall Balestriero},
  year   = {2026},
  eprint = {2605.21800},
  archivePrefix = {arXiv},
  primaryClass = {cs.LG},
  url    = {https://arxiv.org/abs/2605.21800},
}

@article{balestriero2025stable,
  title={stable-pretraining-v1: Foundation Model Research Made Simple},
  author={Balestriero, Randall and Van Assel, Hugues and BuGhanem, Sami and Maes, Lucas},
  journal={arXiv preprint arXiv:2511.19484},
  year={2025}
}

@inproceedings{wu2022tinyvit,
  title={Tinyvit: Fast pretraining distillation for small vision transformers},
  author={Wu, Kan and Zhang, Jinnian and Peng, Houwen and Liu, Mengchen and Xiao, Bin and Fu, Jianlong and Yuan, Lu},
  booktitle={European conference on computer vision},
  pages={68--85},
  year={2022},
  organization={Springer}
}

@article{oquab2023dinov2,
  title={Dinov2: Learning robust visual features without supervision},
  author={Oquab, Maxime and Darcet, Timoth{\'e}e and Moutakanni, Th{\'e}o and Vo, Huy and Szafraniec, Marc and Khalidov, Vasil and Fernandez, Pierre and Haziza, Daniel and Massa, Francisco and El-Nouby, Alaaeldin and others},
  journal={arXiv preprint arXiv:2304.07193},
  year={2023}
}

@article{hafner2019dream,
  title={Dream to control: Learning behaviors by latent imagination},
  author={Hafner, Danijar and Lillicrap, Timothy and Ba, Jimmy and Norouzi, Mohammad},
  journal={arXiv preprint arXiv:1912.01603},
  year={2019}
}

@article{hafner2020mastering,
  title={Mastering atari with discrete world models},
  author={Hafner, Danijar and Lillicrap, Timothy and Norouzi, Mohammad and Ba, Jimmy},
  journal={arXiv preprint arXiv:2010.02193},
  year={2020}
}

@article{tao2016explicit,
  title={Explicit symplectic approximation of nonseparable Hamiltonians: Algorithm and long time performance},
  author={Tao, Molei},
  journal={Physical Review E},
  volume={94},
  number={4},
  pages={043303},
  year={2016},
  publisher={APS}
}

@inproceedings{tran2018closer,
  title={A closer look at spatiotemporal convolutions for action recognition},
  author={Tran, Du and Wang, Heng and Torresani, Lorenzo and Ray, Jamie and LeCun, Yann and Paluri, Manohar},
  booktitle={2018 IEEE/CVF conference on computer vision and pattern recognition},
  pages={6450--6459},
  year={2018},
  organization={IEEE}
}

@article{hansen2022temporal,
  title={Temporal difference learning for model predictive control},
  author={Hansen, Nicklas and Wang, Xiaolong and Su, Hao},
  journal={arXiv preprint arXiv:2203.04955},
  year={2022}
}

@article{cranmer2020lagrangian,
  title   = {Lagrangian Neural Networks},
  author  = {Cranmer, Miles and Greydanus, Sam and Hoyer, Stephan and Battaglia, Peter and Spergel, David and Ho, Shirley},
  journal = {arXiv preprint arXiv:2003.04630},
  year    = {2020}
}

@article{jin2020sympnets,
  title   = {{SympNets}: Intrinsic structure-preserving symplectic networks for identifying {H}amiltonian systems},
  author  = {Jin, Pengzhan and Zhang, Zhen and Zhu, Aiqing and Tang, Yifa and Karniadakis, George Em},
  journal = {Neural Networks},
  volume  = {132},
  pages   = {166--179},
  year    = {2020}
}

@inproceedings{botev2021priors,
  title     = {Which priors matter? {B}enchmarking models for learning latent dynamics},
  author    = {Botev, Aleksandar and Jaegle, Andrew and Wirnsberger, Peter and Hennes, Daniel and Higgins, Irina},
  booktitle = {NeurIPS Datasets and Benchmarks Track},
  year      = {2021}
}

@inproceedings{higgins2021symetric,
  title     = {{SyMetric}: Measuring the quality of learnt {H}amiltonian dynamics inferred from vision},
  author    = {Higgins, Irina and Wirnsberger, Peter and Jaegle, Andrew and Botev, Aleksandar},
  booktitle = {Advances in Neural Information Processing Systems},
  year      = {2021}
}

@article{alvarez2026beyond,
  title   = {Beyond isotropy in {JEPA}s: {H}amiltonian geometry and symplectic prediction},
  author  = {Alvarez, Robert Jenkinson},
  journal = {arXiv preprint arXiv:2605.20107},
  year    = {2026}
}

@article{belloulata2026contractive,
  title   = {Contractive {MASO}-generalized predictors for stable latent-space learning in {JEPA}},
  author  = {Belloulata, Kamel and Belalia, Amina},
  journal = {Transactions on Machine Learning Research},
  year    = {2026}
}

@article{towers2024gymnasium,
  title   = {Gymnasium: A standard interface for reinforcement learning environments},
  author  = {Towers, Mark and Kwiatkowski, Ariel and Terry, Jordan and Balis, John U. and De Cola, Gianluca and Deleu, Tristan and Goul{\~a}o, Manuel and Kallinteris, Andreas and Krimmel, Markus and KG, Arjun and others},
  journal = {arXiv preprint arXiv:2407.17032},
  year    = {2024}
}

@article{murphy1988skill,
  title   = {Skill scores based on the mean square error and their relationships to the correlation coefficient},
  author  = {Murphy, Allan H.},
  journal = {Monthly Weather Review},
  volume  = {116},
  number  = {12},
  pages   = {2417--2424},
  year    = {1988}
}
